\documentclass[msnc]{informs3} 

\DoubleSpacedXI 

\usepackage[colorinlistoftodos,bordercolor=orange,backgroundcolor=orange!20,line
color=orange,textsize=scriptsize]{todonotes}
\catcode`\@=11
\catcode`\@=12
\usepackage{footnote}

\newcommand{\tb}[1]{{\color{black} #1}}
\newcommand{\opt}{^{\star}}

\usepackage{natbib}
 \bibpunct[, ]{(}{)}{,}{a}{}{,}%
 \def\bibfont{\small}%

 \usepackage{amsmath}
\usepackage{algorithm, algorithmicx, algpseudocode}
\usepackage{float}
\usepackage{subcaption}
\usepackage{bm}

\newcommand{\R}{\mathbb{R}}

\newcommand{\A}{\mathbb{A}}

\newcommand{\TT}{\mathbb{T}}

\newcommand{\N}{\mathbb{N}}
\newcommand{\X}{\mathbb{S}}
\usepackage{bbm}

\newcommand{\M}{\mathcal{M}}

\newcommand{\T}{\mathcal{T}}

\newcommand{\E}{\mathbb{E}}

\newcommand{\LL}{\mathcal{L}}
\newcommand{\mcH}{\mathcal{H}}
\newcommand{\mcC}{\mathcal{C}}

\newcommand{\mcV}{\mathcal{V}}
\newcommand{\mcE}{\mathcal{E}}

\newcommand{\XX}{\mathcal{X}}

\newcommand{\optimize}{{\sf OPTIMIZE}}
\newcommand{\fit}{{\sf FIT}}
\newcommand{\updval}{{\sf UPDATEVALUE}}

\newcommand{\comp}{{\sf comp}}
\newcommand{\cart}{{\sf CART}}
\newcommand{\gosdt}{{\sf GOSDT}}

\theoremstyle{plain}

\TheoremsNumberedThrough     
\ECRepeatTheorems

\EquationsNumberedBySection 

\begin{document}


\RUNAUTHOR{Grand-Cl{\'e}ment, Chan, Goyal and Chuang}

\RUNTITLE{Interpretable Triage Guidelines}

\TITLE{Interpretable Machine Learning for Resource Allocation with Application to Ventilator Triage}

\ARTICLEAUTHORS{%
\AUTHOR{Julien Grand-Cl\'ement}
\AFF{ISOM Department, HEC Paris \EMAIL{grand-clement@hec.fr}}
\AUTHOR{You Hui Goh}
\AFF{ISOM Department, HEC Paris \EMAIL{gohy@hec.fr}}
\AUTHOR{Carri W. Chan}
\AFF{Columbia Business School, Columbia University, \EMAIL{cwchan@columbia.edu}} 
\AUTHOR{Vineet Goyal}
\AFF{IEOR Department, Columbia University \EMAIL{vgoyal@ieor.columbia.edu}}
\AUTHOR{Elizabeth Chuang}
\AFF{Department of Family and Social Medicine,  Albert Einstein College of Medicine, Montefiore \EMAIL{echuang@montefiore.org}}
}

\ABSTRACT{%
Rationing of healthcare resources  is a challenging decision that policy makers and providers may be forced to make  during a pandemic, natural disaster, or mass casualty event. Well-defined guidelines to triage scarce life-saving resources must be designed to promote transparency, trust and consistency. To facilitate buy-in and use during high stress situations, these guidelines need to be interpretable and operational.  We propose a novel data-driven model to compute interpretable triage guidelines  based on policies for Markov Decision Process that can be represented as simple sequences of decision trees (\textit{tree policies}).  In particular, we characterize the properties of optimal tree policies and present an algorithm based on dynamic programming recursions to compute tree policies.  We utilize this methodology to obtain simple, novel triage guidelines for ventilator allocations for COVID-19 patients, based on real patient data from Montefiore hospitals. We also compare the performance of our guidelines to the official New York State guidelines that were developed in 2015 (well before the COVID-19 pandemic). Our empirical study shows that the number of excess deaths associated with ventilator shortages could be reduced significantly using our policy.  Our work highlights the limitations of the existing official triage guidelines, which need to be adapted specifically to COVID-19 before being successfully deployed.
}%

\KEYWORDS{Triage guidelines, SARS-CoV-2, New York State Guidelines, Interpretable Machine Learning,  Markov Decision Process}
 
\maketitle

\section{Introduction}
Healthcare delivery operates in a resource limited environment where demand can sometimes exceed supply, resulting in situations where providers have to make difficult decisions of who to prioritize to receive scarce resources. Having a framework to guide such decisions is critical, particularly with growing threats of pandemics, natural disasters, and mass casualty events that make the healthcare system vulnerable to situations where demand vastly exceeds supply of critical healthcare resources. Triage of health resources in such situations has garnered quite a bit of attention from the Operations Management (OM) community  (e.g. \cite{jacobson2012priority,mills2013resource,sun2018patient}). The primary focus of prior work has been to determine how care should be rationed. In this work, we develop a machine learning methodology that considers how data can be used to guide these decisions in an {\it interpretable} manner and what criteria should be considered in these life-or-death decisions.

Triage guidelines are often implemented during high stress, complex situations. As such, triage algorithms need to be systematic, simple and intuitive in order to facilitate adoption and to ease the decision burden on the provider. In 2012, the National Academy of Medicine identified an ethical framework for triage~\citep{gostin2012crisis}.  In this framework, the primary component which the operations community has focused on is `the duty to steward resources', which calls for withholding or withdrawing resources from patients who will not benefit from them. In this work, we take a data-driven approach to develop triage guidelines for ventilator allocation in order to maximize lives saved. Such decisions are ethically challenging and such an objective can sometimes be in tension with other critical elements such as fairness. We will also examine how the ethics and fairness criteria can be incorporated into such decisions.

On the one hand,  government officials have issued pre-specified and transparent utilitarian \textit{triage} guidelines for preventing loss of life, promote fairness, and support front-line clinicians~\citep{christian2014triage,zuckerventilator,piscitello2020variation}. In the United States, 26 states have scarce resource allocation guidelines ~\citep{babar2006direct}.  These guidelines emphasize simplicity,  and can often be represented as decision trees of small depth~\citep{breiman1984classification,bertsimas2017optimal}. However, these guidelines have never been used in practice,  and are not constructed in a \textit{data-driven} way, but rather via expert opinion of clinicians, policy makers, and ethicists. Therefore, it is not known how well they perform for the intended purpose of directing scarce resources to those most likely to benefit.  In addition, we cannot ethically perform a prospective study to determine the efficacy (or performance) of such policies.

On the other hand,  there is a large body of literature in the OM community on the management of healthcare resources, including breast magnetic resonance imaging~\citep{hwang2014patterns}, patient scheduling~\citep{bakker2017dynamic}, ICU beds~\citep{chan2012optimizing,kim2015icu},  mechanical ventilators and high-flow nasal cannula~\citep{gershengorn2021impact,anderson2023rationing}.  Because the health of each patient evolves dynamically over time,  the problem of scarce resource allocation is inherently \textit{sequential} in nature.  In theory,  one can leverage the tools from OM and Machine Learning (ML) to compute new allocation guidelines.   Markov Decision Processes (MDPs) and queueing theory are tools that are commonly used to find optimal sequences of decisions in a stochastic environment~\citep{Puterman,whitt2002stochastic}. With such methodologies,  optimal policies can often be numerically computed efficiently using iterative algorithms; however, the optimal policies may not have an interpretable structure. This is often referred to as \textit{black-box} policies in the ML community~\citep{rudin2019stop}.  To obtain practical operational guidelines,  it is necessary to obtain \textit{interpretable} decision rules that can easily be explained and implemented by the medical staff and discussed with the practitioners. This becomes even more important for the ethically challenging task of triaging scarce life-sustaining resources. In this case, what constitutes an appropriate input into the model may be contested. For instance, age is a strong predictor of outcomes for patients infected with Sars-CoV-2, but ethicists disagree about the appropriateness of using age in triage decisions~\citep{may2020age}. For triage algorithms to be used in practice, triage rules must be exposed for public debate, and therefore, interpretability is a key property.

In certain settings, simple index-based triage rules can be shown to perform well (e.g. , the $c\mu$-rule for multiclass queues with Poisson arrivals~\citep{van1995dynamic}. However, these works typically assume static patient health status and/or a limited number of patient classes (e.g. two classes in \cite{van1995dynamic}). The complexity of these problems arises from the capacity constraint which creates externalities across different patients. In contrast, we consider a setting where the health status of each patient evolves dynamically (and stochastically) overtime. This richer state-space for the patient state substantially increases the computational complexity of our model and renders existing approaches for triage intractable. As such, we develop a model that explicitly incorporates the dynamic patient health state and approximates the capacity constraint through appropriately calibrated reward parameters.   

Our goal in this paper is to develop allocation guidelines that are both \textit{data-driven} and \textit{interpretable}.  To do so, we propose a new model for interpretable sequential decision-making, based on decision trees.  We then apply this methodology to develop guidelines for ventilator allocation to COVID-19 patients. The COVID-19 pandemic has highlighted the challenge of managing life-saving health resources as demand for intensive care unit beds,  critical appliances such as mechanical ventilators, and therapies such as dialysis all were in short supply in many countries. Shortages of ventilators and oxygen  occurred in Italy and India, amongst other countries \citep{rosenbaum2020facing,kotnala2021clinical}. Given the ongoing risks of emerging infectious diseases~\citep{zumla2016infectious} and the projected increase in frequency of extreme weather events~\citep{woodward2018climate}, hospitals are increasingly  vulnerable to conditions that may result in periods of scarcity of life-sustaining resources even after the COVID-19 pandemic subsides.  We utilize the question of ventilator allocation for COVID-19 patients as a canonical example of how our methodology can be used to develop data-driven, interpretable triage guidelines for such settings.

Our main contributions can be summarized as follows: 
\begin{itemize}
\item \textit{Interpretable policies for MDPs.} We propose a framework to derive interpretable decision rules with good performances for sequential decision-making problems, based on decision trees. In particular, we model the evolution of the health of the patient as an MDP and we focus on \textit{tree policies}, which have a tree-like structure, and provide intuitive and explainable decision rules for MDPs.  
\item \textit{Properties of optimal tree policies.} By construction, the set of policies we consider is constrained to have a tree-like structure. Therefore,  the properties of optimal tree policies are in stark contrast with the classical properties for unconstrained MDPs. We show that optimal tree policies may depend on the initial condition, and may even be history-dependent. However, we show that an optimal tree policy can be chosen to be deterministic (though it may be history-dependent). \tb{For the sake of simplicity, we focus on Markovian tree policies, and we show that an optimal Markovian tree policy can still be chosen deterministic, which is appealing from an implementation standpoint. We highlight the challenges of computing optimal Markovian tree policies by connecting this problem to the minimum vertex cover, which is NP-hard. }
\item \textit{Algorithm for tree policies.} \tb{We present two different algorithms for computing tree policies. Since computing history-dependent policies is intractable, we focus on finding Markovian tree policies. Our first algorithm fits a tree to an optimal unconstrained policy. Our second algorithm performs dynamic programming recursions akin to Bellman recursions but forces the visited policies to be tree policies. The main bottleneck of each algorithm is the subroutine for computing decision trees.}
\item \textit{Application to mechanical ventilators allocations.} We apply our novel model and algorithm to compute interpretable triage guidelines for ventilator allocations for COVID-19 patients.
We leverage a data set of 807 COVID-19 patients intubated at an urban academic medical center in New York City and build an MDP model to obtain interpretable triage guidelines. We compare them to the official New York State (NYS) guidelines and First-Come-First-Served (FCFS) guidelines. 

We find that the NYS guidelines perform on par with the FCFS rules, highlighting the need to adjust official guidelines using data-driven methods. We also find that our tree policies can significantly decrease the number of deaths in the overall patient population. Compared to NYS guidelines, our novel triage policies are less aggressive and exclude fewer patients at triage, but more patients at reassessments.
We also provide a detailed discussion about the \textit{criteria} (SOFA score, comorbidities, demographics) included in allocation rules.  Surprisingly, when we include information like age,  BMI or diabetes, the resulting tree policies may not significantly outperform our simpler tree policy only based on SOFA. This reflects that categorical exclusions of patients may be too rigid in a dynamic crisis \citep{white2020framework}. This also shows that the SOFA score itself is a powerful predictor for patient's survival, and highlights the practical limitations of learning increasingly sophisticated triage rules with a limited amount of data. Additionally, the use of comorbidities and demographics in triage guidelines causes important moral concerns~\cite {auriemma2020eliminating},  and we show that this may not be justified on the grounds of a much higher number of lives saved. Triaging based on the NYS guidelines alone is inadequate for COVID-19 patients. Different disaster scenarios (e.g., COVID-19 pandemic vs. hurricane Sandy) potentially warrant different triage criteria, which have to be learned dynamically as more information on the disaster becomes available, and current state-wise SOFA-based policies may need to be revisited before being implemented and efficiently saving lives.
\end{itemize}

\paragraph{Outline.} The rest of this paper is organized as follows.  The end of this section is devoted to a brief literature review. In Section \ref{sec:model}, we introduce our Markov Decision Process (MDP) model.  In Section \ref{sec:interpretable}, we  introduce our tree policies for MDPs,  we present the properties of optimal tree policies, we study the complexity of finding an optimal tree policy and we present our algorithms to compute Markovian tree policies.
In Section \ref{sec:Monte-learning-simu-setup}, we present our numerical study (data set, MDP model, simulation model), applying our framework to mechanical ventilator allocations for COVID-19 patients.  We discuss our empirical results in Section \ref{sec:simu},  where we detail our comparison of various triage guidelines under different levels of ventilator shortages, as well as the limitations of our approaches.

\subsection{Related literature}
Our work primarily builds upon (i) triage guidelines for scarce healthcare resources, (ii) decision trees and MDPs for decision making in healthcare,  and (iii) recent advances to compute interpretable decisions in sequential decision-making.

\paragraph{Triage guidelines in the Operations Management literature.}
There is a large body of literature in the operations management (OM) and medical literature on designing efficient triage guidelines for allocations of scarce healthcare resources.  In a broad sense, the usual objective is of \textit{doing the greatest good for the greatest number}~\citep{frykberg2005triage}. For instance,  \cite{sacco2007new} propose to use linear programming to determine priorities among patients in the hospitals, and \cite{jacobson2012priority} rely on sample-path methods and stochastic programming for computing policies for assignments to key resources (such as ambulances and operating rooms) in the aftermath of mass-casualty events.  \cite{kim2015icu} estimate the impact of ICU admission denials on the outcomes of the patients.  Triage conditions can be tailored to various situations, including austere conditions and  imperfect information~\citep{argon2009priority,childers2009prioritizing,li2010approximate,sun2018patient} or specific health conditions, such as burn-injured patients~\citep{chan2013prioritizing}. Note that triage guidelines are also of interest in other areas of OM, e.g.  allocating customers to servers~\citep{dobson2011impact,alizamir2013diagnostic}.
Other works analyze directly the regimes where triage may or may not be beneficial~\citep{sun2022triage},  and the impact of available capacity on triage decisions~\citep{chen2020effects}. \tb{Perhaps most similar to our paper is the work in \cite{anderson2023rationing}. The authors also look at the NYS ventilator allocation guidelines and compare to two guidelines they developed which incorporate machine learning predictors (regularized logistic regressions) beyond just the Sequential Organ Failure Assessment (SOFA) score. We take a different methodological approach. While \cite{anderson2023rationing} focus on the {\em prediction} of survivability and length-of-ventilation, we focus on {\em allocation} and {\em interpretability}, developing a methodology to identify a tree-based policy to triage ventilators. We will provide an in-depth comparison with the methods in \cite{anderson2023rationing} in Appendix \ref{app:comparison ISP-LU}.}
More generally, the efficacy of the  triage and allocation guidelines are often demonstrated numerically and, in some instances,  some structural properties of optimal policies can be derived theoretically. \tb{Many of these works focus on index policies in order to facilitate implementation, see~\cite{mills2016simple} for application to mass-casualty triage and~\cite{bastani2021efficient} for applications to COVID-19 testing at the Greek border.} The indices are often derived from aggregate metrics which capture factors like survival risk, length-of-stay, and/or risk of mis-triage. However, such approaches can obscure clinical features resulting in lack of \textit{interpretability} and are not easily explainable to the medical staff. 

\paragraph{Decision trees and MDPs in healthcare.}
Tree-based models are popular in general classification tasks, and specifically in healthcare applications.  \tb{Indeed,  trees provide a decision model that is easily interpretable, which may be preferred in high-stake settings over more accurate methods that are hard to interpret and justify.}
Classical heuristics such as CART \citep{breiman1984classification} and C4.5 \citep{quinlan2014c4} scale well but often return suboptimal trees, while optimal solutions can be computed with mixed-integer linear programming~\citep{bertsimas2017optimal}.  In healthcare, decision trees have been used in a number of settings including diagnosing heart diseases \citep{shouman2011using},  developing novel non-linear stroke risk scores \citep{orfanoudaki2020machine}, and predicting survival probability of the patients~\citep{bertsimas2022optimal}.  As the scope of applications of decision trees to healthcare decision-making is very large, we refer the reader to the reviews in  \cite{tomar2013survey} and \cite{dhillon2019machine}.  Most of these applications are based on static \textit{predictions} (e.g., diagnosis) of the diseases and do not take into account the dynamic evolution of the health of the patient over time, \tb{with the notable exception of \cite{zhang2018interpretable}}. In contrast to most of the literature, we build upon the Markov Decision Process (MDP) literature to better model the impact of our decisions over time.

Markov Decision Processes (MDPs) provide a simple framework for the analysis of \textit{sequential} decision rules.  This is the reason why MDPs are widely used in many healthcare applications, where the evolution of the patient's health impacts the sequence of decisions chosen over time. The use of Markov models in healthcare can be traced back to \cite{beck1983markov}. In particular, MDPs have been used, among others, for kidney transplantation \citep{mdp-kidney}, HIV treatment recommendations \citep{mdp-HIV},  cardiovascular controls for patients with Type 2 diabetes \citep{mdp-steimle} and determining proactive transfer policies to the ICU\citep{grand2023robustness}.  We refer the reader to \cite{mdp-med-2,mdp-med-1} for reviews of applications of MDP to medical decision making.
\paragraph{Interpretable sequential decision-making.}
Decision trees are efficient where there is  a single decision epoch. This is the setting where there has been the most traction of interpretable machine learning in healthcare settings~\citep{khan2008predicting,amin2018performance,bertsimas2022optimal}. In contrast, MDPs are suitable for modeling sequential decisions but do not, a priori, incorporate interpretability constraints.  Recently, there has been some works toward developing methods for interpretable policies in sequential decision-making(not necessarily specific to the healthcare setting).  \cite{bravo2020mining} propose to  explain the optimal unconstrained policies with decision trees, applying their framework to classical operations problems such as queuing control and multi-armed bandit (MAB).  However, this may be misleading, as there is no guarantee that the novel explainable, suboptimal policies have the same performance as the unconstrained, optimal policies \citep{rudin2019stop}.  \tb{The recent framework of {\em interpretable stopping time}, as introduced in \cite{ciocan2022interpretable}, is probably the closest to our work. In particular, \cite{ciocan2022interpretable} introduce the notion of tree policies for stopping time problems and design an algorithm returning an interpretable stopping policy, given a dataset of past observations. While also motivated by the shortcoming that optimal policies may be hard to interpret, there are several crucial differences between the results in \cite{ciocan2022interpretable} and our approach. First, \cite{ciocan2022interpretable} consider a {\em model-free} approach, and optimize an objective function directly parametrized by the dataset of historical observations (the sample average approximation). In contrast, we consider a {\em model-based} approach, in the case where the decision maker already knows the transition kernel and the reward of the underlying stochastic process. This enables us to provide several important structural properties of optimal tree policies, which may be history-dependent but can always be chosen to be deterministic. From an algorithmic standpoint, our solution methods also differ from the ones in \cite{ciocan2022interpretable}. In particular, \cite{ciocan2022interpretable} modifies an algorithm for computing classification trees (CART, from~\cite{breiman1984classification}) to compute stopping time policies. In contrast, we modify an algorithm for computing MDP policies (value iteration) by exploiting an algorithm for classification trees as a subroutine at every iteration (see Algorithm \ref{alg:optimize then fit} and Algorithm \ref{alg:optimize and fit}), to ensure that value iteration only visits tree policies. Therefore, our algorithms find a new decision tree for the set of states at each period, whereas the algorithm in \cite{ciocan2022interpretable} finds a single decision tree across all periods, treating the period as a variable in the decision process.  Our overall algorithm is independent of the subroutine used to compute decision trees, and heuristics or optimal methods can be chosen, depending of the preference for fast computations or accuracy of the decision trees returned.
}

\section{Model of sequential decisions}\label{sec:model}
Our goal is to develop guidelines for determining which patients to allocate (or not) a healthcare  resource. We consider the dynamic evolution of patient health and how to utilize this information in making triage decisions. We consider a finite-horizon Markov Decision Process (MDP) to model the patient health evolution and we develop a methodology to compute policies, which admit simple \textit{tree} representations at every decision epoch.  In Section \ref{sec:Monte-learning-simu-setup},  we will examine the application of our methodology to ventilator allocations for COVID-19 patients.

We consider a fairly general setting of sequential resource allocation in healthcare. Patient health evolves stochastically over time.  We consider the allocation of a single type of healthcare resource (e.g., ICU beds, mechanical ventilator, specialized nurses,  medications) to different patients.  The nature of the allocation guideline is \textit{sequential}; the decision can be revisited multiple times and can depend on: 1) the current resource capacity,  2) the current health condition of the focal patient and its potential subsequent improvement and deterioration,  and 3) the current health conditions of all other patients and their potential subsequent improvements and deteriorations. Because of the capacity constraint, the decision of allocating a resource (as well as how much) to a single patient impacts the ability to treat other current and, possibly, future patients. While most prior works reduce the state-space by restricting to two classes of patients and/or ignoring the evolution of patient health, the richer model of patient health is an important component of our data-driven approach to triage.

\subsection{Evolution of Patient Health and Action Space}
We model the health condition of each patient according to a finite-state Markov Chain. We let $s_{t} \in \X_{t}$ denote the health state of a patient in  period $t$.  Each state represents the health condition which may include information such as vital signs and lab values, as well as comorbidities and demographics.  In our specific application to mechanical ventilator allocation we will consider multi-dimensional states, but our model allows for more general, arbitrary (finite) set of states.
 The vector $\bm{p}_{1}$ represents the likelihood to start in each state in $\X_{1}$. In contrast to the existing OM triage literature, we allow for an arbitrary (finite) number of patient health states that evolve stochastically depending on the action taken.

At each period $t$, let $\A_{t}$ denote the set of discrete possible actions. For example, one could consider a binary decision of whether to intubate or extubate a patient when considering  ventilator allocations. Alternatively, one could consider different levels of drug dosages -- $\{\sf 100mg, 200mg, 300mg\}$ -- for medication allocation. 
Once an action $a_{t}$ is chosen based on a policy, the patient transitions to the next state $s_{t+1} \in \X_{t+1}$ with probability $P_{s_{t}a_{t}s_{t+1}}$. 
 The transitions represent the health evolution of the patient toward future states. We assume that the transitions are \textit{Markovian} to keep the model tractable. Terminal states in $\X_{H}$ represent the status at discharge, typically  in the set $\{ {\sf deceased, alive} \}$. For conciseness, we assume that the set of states $\X_{1}, ..., \X_{H}$ attainable at period $t=1,...,H$, are disjoint. This allows us to consider transition kernels and rewards that do not depend of the current period.

\subsection{Capacity Constraint}
\tb{Explicitly modeling the complex health evolution over time of all the patients simultaneously, with a common capacity constraint on the shared resource, leads to a high-dimensional MDP formulation. Indeed, if we want to model the evolution of $N$ patients, each with $S$ possible health states, the number of possible states for the $N$ patients is $S^N$, which becomes intractable for a large number of patients.}
\tb{To solve this issue, a common assumption when considering the multi-patient allocation problem is that the dynamics of the patients are independent: they are only linked by the utilization of the common resource, e.g. there is a bound on the total number of patients using the resource at any time.} This can be seen as a {\em weakly coupled processes}~\citep{adelman2008relaxations}, and a Lagrangian relaxation of this multi-patients allocation problem can be obtained as a classical Markov Decision Process (MDP), describing the health evolution and resource allocation of a single patient, with adequate penalization of the instantaneous rewards \tb{(directly proportional to the Lagrange multipliers introduced in the relaxation). We refer to Section 2.3 in \cite{adelman2008relaxations} for a rigorous derivation of these results.} As such, to facilitate some tractability to derive interpretable policies, we consider the allocation for an individual patient and we {\em implicitly} incorporate the capacity constraint via appropriately defined reward parameters. \tb{In particular, in each period $t$, a time-dependent reward $r_{s_{t}a_{t}} \in \R$ reflects the reward obtained given that the patient health state is $s_t \in \X_{t}$ and the action taken is $a_t \in \A_{t}$. This single-patient allocation problem is a relaxation of the multi-patient allocation problem, with rewards that are decreasing in the amount of resources allocated to this single patient, so that the decision-maker favours allocation policies that also minimize the amount of resources used, on the top of   clinical metrics considered, like survival.} 

The reward parameters reflect the (potentially multiple) objective(s) of the decision maker and can  incorporate 1) patient risk of bad outcomes, 2) the benefits of using healthcare resources (both explicit as well as opportunity costs), and 3) potential risks associated with the action (e.g., complications).
The rewards can be decreased to deter the decision maker from using too many resources compared to a model without any capacity constraint.  

While our model ultimately focuses on a single patient with dynamic health state and accounts for the capacity constraint via the reward formulation, we note that the focus on a single patient is the standard approach in designing triage algorithms in the medical community. In particular, the majority of triage guidelines are myopic, and they are only executed when the capacity constraint is met. When there are any available resources, resources are allocated myopically (i.e., to patients that require them) without concern about the future potential of running out of capacity~\citep{zuckerventilator,white2020framework}.
\subsection{Objective function and decision rule}\label{sec:nominal-MDP}
The goal of the decision maker is to maximize the expected return $R(\pi)$ associated with a \textit{policy} $\pi=\left(\pi_{1},...,\pi_{H}\right)$, which is a sequence of \textit{decision rules} $\pi_{t}$ over a finite horizon $H$.
A (deterministic) decision rule $\pi_{t}$ maps a  \textit{history up to period t}, $h_{t}=(s_{1},a_{1},...,s_{t-1},a_{t-1},s_{t})$, to an action $\pi_t(h_t)$ in the set $\A_{t}$. The {\em expected cumulative rewards} $R(\pi)$, later referred to as the {\em return}, associated with a policy $\pi$ is calibrated to capture the balance between clinical objectives (e.g., optimizing the survival of patients) and the costs of using resources. It is defined as
\begin{equation}\label{eq:expected-return}
R(\pi) = \E^{\pi, \boldsymbol{P}} \left[ \sum_{t=1}^{H} r_{s_{t}a_{t}} \; \bigg| \; s_{1}  \sim \bm{p}_{1} \right].
\end{equation}
For a fixed Markovian deterministic policy $\pi$,  we can associate a sequence of \textit{value functions} $\left( \bm{v}_{t}^{\pi} \right)_{t \in [H]} \in \R^{\X_{1} \times .... \times \X_{H}}$, defined recursively as
\begin{equation}\label{eq:Bellman recursion - policy evaluation}
    \begin{aligned}
v^{\pi}_{H,s} & =  r_{s\pi_H(s)},  \forall \; s \in \X_{H},\\
v^{\pi}_{t,s} & = r_{s\pi_t(s)} +  \sum_{s' \in \X_{t+1}} P_{s\pi_t(s)s'} v^{\pi}_{t+1,s'},  \forall \; s \in \X_{t}, \forall \; t \in [H-1].
\end{aligned}
\end{equation}
For each period $t \in [H]$ and state $s \in \X_{t}$, $v^{\pi}_{t,s}$ represents the cumulative expected return for following the decision rules $(\pi_{t},...,\pi_{H})$ from period $t$ to period $H$, starting from state $s$:
\[v^{\pi}_{t,s} = \E^{\pi, \boldsymbol{P}} \left[ \sum_{t'=t}^{H}  r_{s_{t'}a_{t'}} \; \bigg| \; s_{t} = s \right].\]
From the definition of the return $R(\pi)$ as in \eqref{eq:expected-return}, we have $R(\pi) = \bm{p}_{1}^{\top}\bm{v}^{\pi}_{1}$.
Crucially, an optimal policy $\pi\opt$ which maximizes the expected return \eqref{eq:expected-return} can be chosen to be \textit{Markovian} ($\pi\opt_{t}$ only depends of the current state $s_{t}$ and not of the whole history $h_{t}$), and \textit{deterministic}. Additionally, $\pi\opt$ can be computed using the following \textit{value iteration} algorithm~\citep{Puterman}: the value functions $\left( \bm{v}_{t}\opt \right)_{t \in [T]} \in \R^{\X_{1} \times .... \times \X_{H}}$ of an optimal policy $\pi\opt$ follow the \textit{Bellman optimality equation} \eqref{eq:Bellman-recursion}:
\begin{equation}\label{eq:Bellman-recursion}
\begin{aligned}
v\opt_{H,s} & = \max_{a \in \A_{H}} r_{sa},  \forall \; s \in \X_{H},\\
v\opt_{t,s} & = \max_{a \in \A_{t}} r_{sa} + \sum_{s' \in \X_{t+1}} P_{sas'} v\opt_{t+1,s'},  \forall \; s \in \X_{t}, \forall t \in [H-1],
\end{aligned}
\end{equation}
and an optimal policy $\pi\opt$ can be chosen as the policy that maps each state $s \in \X_{t}$ to the action attaining the $\arg \max$ in the \textit{Bellman equations} \eqref{eq:Bellman-recursion}.
\section{Interpretable Policies}\label{sec:interpretable}
In a healthcare setting,  interpretability of the decisions is crucial to operationalize the guidelines and facilitate practical implementation of the policies,  generate buy-in from policy makers and providers, and mitigate obscuring of any potential ethical issues.
A priori,  the optimal policy for a classical, unconstrained MDP instance may not have any interpretable structure.
In this work, we use a model of interpretable decisions based on \textit{decision trees} \citep{breiman1984classification}.  Intuitively, the goal is to compute an efficient policy for the finite-horizon MDP problem, which can be succinctly represented as an interpretable decision tree at each decision period.

\subsection{Classification trees}\label{subsec:classification-trees}
We start with a brief introduction to decision trees. Decision trees are widely used in classification problems \citep{breiman1984classification,bertsimas2017optimal}.  We use the following definition of a decision tree and present two examples of decision trees in $\R^{3}$ in Figure \ref{fig:example-trees}.
\tb{
\begin{definition}[Decision tree and labeling rule]\label{def:decision-tree} 
A decision tree $T$ is a map $\R^{p} \rightarrow \{1, ..., K \}$, with $K \in \N$, which \textit{recursively} partitions $\R^{p}$ into $K$ disjoint sub-regions (called \textit{classes)}, using \textit{branch} nodes and \textit{leaf} nodes. Branch nodes rely on \textit{univariate} splits such as $``x_{1} \leq 2"$ or $``x_{3} \leq 8"$.  A data point follows the left branch if it satisfies the splitting condition, otherwise it follows the right branch. A node is a leaf if there is no split at this node. Each leaf defines a sub-region of $\R^{p}$, resulting from the sequence of splits leading to this leaf. Each sub-region is uniquely identified with a \textit{class} $c \in [K]$.
Given a finite set of labels $\LL$, a labeling rule $\mu$ is a map $\mu: [K] \rightarrow \LL$.
\end{definition}}

\begin{figure}
\center
 \begin{subfigure}{0.40\textwidth}
\centering
         \includegraphics[width=0.9\linewidth]{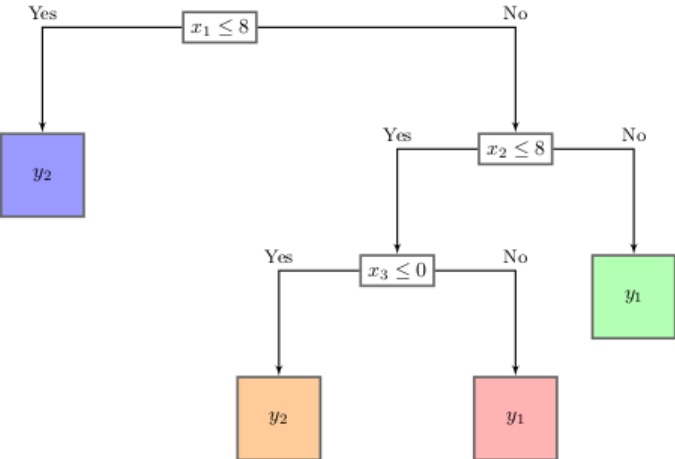}
         \caption{Decision tree and labeling rule.}
         \label{fig:tree-label-1}
  \end{subfigure}
   \begin{subfigure}{0.40\textwidth}
\centering
         \includegraphics[width=0.9\linewidth]{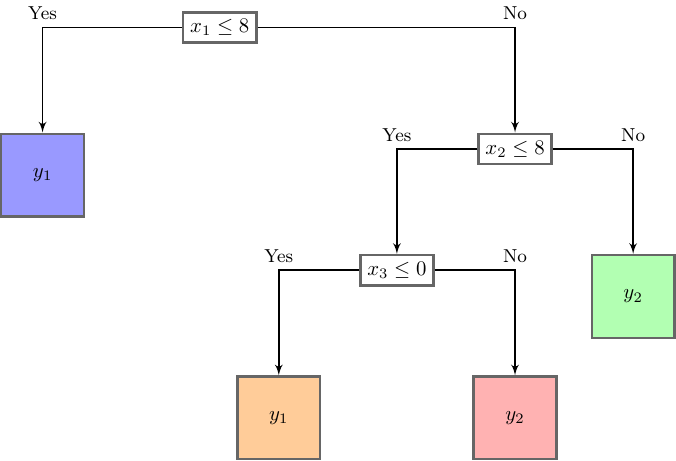}
         \caption{Same decision tree and another labeling rule.}
         \label{fig:tree-label-2}
  \end{subfigure}
  \caption{Example of two decision trees with three branching nodes, four classes (indicated by the four colors at the leaves), and two labels $y_{1}, y_{2}$. The labeling rules are different in Figure \ref{fig:tree-label-1} and in Figure \ref{fig:tree-label-2}.}
  \label{fig:example-trees}
\end{figure}
\tb{Note the difference between {\em classes} (subregions of $\R^{p}$ defined by a decision tree) and {\em labels}, which are assigned by a labeling rule. This distinction is important because different subregions of $\R^{p}$ may share the same label. 
 We write $\T(\XX,[K])$ for the set of decision trees defining $K$ sub-regions of the set $\XX \subset \R^{p}$, identified as functions $\XX \rightarrow [K]$, so that $T(\bm{x}) \in [K]$ is the class of an observation $\bm{x} \in \XX$, $\mu(c) \in \LL$ is the label assigned to a class $c \in [K]$, and $\mu \circ T: \XX \rightarrow \LL$ maps observations to labels. We also define the {\em depth} of a tree as the length of the longest sequence of edges between the root and a leaf node. We now describe the optimization problem associated with finding an optimal decision tree. 
  Let $\{\bm{x}_{1},...,\bm{x}_{m}\} \subset \R^{p}$ be a finite set of observations, $\LL$ be a finite set of labels, and $\left(\omega_{i,\ell}\right)_{i,\ell} \in \R^{m \times |\LL|}$ be some misclassification costs. Typically, each observation $\bm{x}_{i}$ is associated with a label $y_{i} \in \LL$ and $\omega_{i,\ell} = 0$ if and only if $\ell = y_{i}$, otherwise $\omega_{i,\ell}=1$.   We write $C_{\sf tree}(T,\mu)$ the expected \textit{classification error} of the decision tree $T$ and the labeling rule $\mu$:}
 \tb{
\begin{equation}\label{eq:classification-error-randomized}
C_{\sf tree}(T,\mu) = \frac{1}{m} \sum_{c=1}^{K} \sum_{i=1}^{m} \omega_{i,\mu(T(\bm{x}_i))}.
\end{equation}}
\tb{The Classification Tree \eqref{eq:cart} optimization problem is to compute a decision tree and a labeling rule which minimize the expected classification error: 
\begin{equation}\label{eq:cart}\tag{CT}
    \min_{T \in \T(\XX,[K])} \; \min_{\mu:[K] \rightarrow \LL} \; C_{\sf tree}(T,\mu).
\end{equation}}
Classical heuristics for solving \eqref{eq:cart} such as CART \citep{breiman1984classification} scale well but often return suboptimal trees. Optimal classification trees can be computed using reformulations as mixed-integer linear programming \citep{bertsimas2017optimal,mctavish2022fast}.
A certain number of \textit{interpretability constraints} can be incorporated in the set $\T(\XX,[K])$,  including upper bounding the \textit{depth} of the tree,  defined as the maximum number of splits leading to a sub-region. Smaller trees are easier to understand as they can be drawn entirely and prevent over-fitting~\citep{bertsimas2017optimal}.
\tb{
\begin{remark} We note that in Definition \ref{def:decision-tree}, we restrict the decision trees to use {\em univariate splits} at each branch node (such as $``x_{1} \leq 2"$ or $``x_{3} \leq 8"$) instead of {\em multivariate} splits (such as $``x_{1}+x_{3} \leq 10"$), since univariate splits may be easier to interpret for healthcare applications. All the results in this paper can be extended to the case of multivariate splits, and our algorithms for computing tree policies can also incorporate multivariate splits as long as the subroutine for fitting trees also supports computing decision trees with multivariate splits.
\end{remark}
}
\subsection{Tree policies for Markov Decision Processes}
Decision trees are a popular framework for finding interpretable classification rules, but they are not readily defined for {\em sequential} decision-making.  We develop an analogous notion of decision trees for MDPs, which we refer to as \textit{tree policies}.  Intuitively, a policy $\pi$ is called a \textit{tree policy} if at every period $t$,  the decision rule $\pi_{t}$ can be represented as a decision tree which assigns labels (actions, or treatments) from $\A_{t}$ to observations (states, or health conditions) from $\X_{t}$. In particular, we have the following definition. Recall that from our definition of a decision tree in Section \ref{subsec:classification-trees},  we see a tree $T$ as a map from the set of observations in $\R^{p}$ to a class in $\{1,...,K\}$, i.e., we have $T: \R^{p} \rightarrow \{1,...,K\}$.
\tb{
\begin{definition}\label{def:tree-policies}
Let
$\M=(H,\X,\A,\bm{P},\bm{r},\bm{p}_{1})$ be an MDP instance.
\begin{enumerate}
\item {\em Admissible trees.}
A sequence of decision trees $T = \left(T_{t} \right)_{t \in [H]}$ is {\em admissible} for $\M$ if $T_{1} \in \T(\X_{1},[K_{1}]), ..., T_{H} \in \T(\X_{H},[K_{H}])$.
We define $\TT$ the set of all sequence of decision trees admissible for the MDP instance $\M$:
\[\TT = \left\{ T = \left(T_{t} \right)_{t \in [H]} \; | \; T_{t} \in \T(\X_{t},[K_{t}]),  \forall \; t \in [H]\right\}.\]
\item {\em Policies compatible with a sequence of decision tree.} Let $T \in \TT$ be an admissible tree. A tuple $\pi= \left(\pi_{1},...,\pi_{H} \right)$ is a {\em policy compatible with the sequence of decision tree $T$} if for each $t \in [H]$, $\pi_{t}$ is a map $\mcH_{t} \times [K_{t}] \rightarrow \A_{t}$ with $\mcH_{t}$ the set of feasible histories up to time $t$ $\left(s_{0},a_{0},...,s_{t-1},a_{t-1}\right)$.
\item {\em Tree policies.} A \textit{tree policy} $\pi$ is a tuple $\pi= \left(\pi_{1},...,\pi_{H} \right)$ that is compatible with an admissible sequence of decision tree $T \in \TT$. We define $\Pi_{T}$ the class of policies that are compatible with a particular sequence of classification trees $T$.
\end{enumerate}
\end{definition}}
\tb{ From the above definition, we see that a tree policy actually consists of two parts: (i) an admissible sequence of decision tree $T \in \TT$, which provides structural constraints for the decisions at each period $t=1,...,H$ and that maps each state $s \in \X_{t}$ to a class $T_{t}(s) \in [K_{t}]$, where we recall  that $T_{t}(s)$ is the class returned by the tree $T_{t}$ for state $s$, and (ii) a policy $\pi$ that is compatible with $T$ and chooses some actions in $\A_{t}$ for each class in $[K_{t}]$ defined by the decision trees in $T$ at each period $t$. The policy $\pi$ plays the same role as the labeling rule from Definition \ref{def:decision-tree}, except that they can be history-dependent and randomized since we now consider solving a sequential problem, in contrast to \eqref{eq:cart}. In the case of a {\em Markovian} and {\em deterministic} policy, the tree policy $\pi=\left(\pi_{1},...,\pi_{H}\right)$ is simply represented as $\pi_{t}:[K_t] \rightarrow \A_{t}$ for $t \in [H]$, with $\pi_t(c) \in \A_{t}$ the action chosen for all states in class $c \in [K_t]$. We note that we slightly overload the notations here, as an {\em unconstrained} Markovian deterministic decision rule $\pi_t$ is a map $\X_{t} \rightarrow \A_{t}$ from the set of states to the set of actions (see the definition in Section \ref{sec:nominal-MDP}), while a Markovian deterministic tree policy $\pi_t$ at time $t$  is a map $[K_t] \rightarrow \A_{t}$ from the set of classes (defined by a tree over $\X_t$) to the set of actions. In the rest of the manuscript, we always make it clear when $\pi_t$ is an unconstrained decision rule or a tree policy. We also note that when the set of admissible trees $\TT$ contains trees that are sufficiently deep, the decision trees can simply enumerate the finite sets of states at each period, and any policy can be represented as a tree policy. However, in practice, one may impose additional constraints on the set of admissible trees $\TT$. Indeed, deep decision trees hinder interpretability, as the number of leaf nodes grows quickly with the depth of the trees. For this reason, it is typical to impose constraints on the set $\TT$, for instance an upper bound on the total number of split nodes or a lower bound on the number of data points that reach a split node. All these criteria can be incorporated in existing heuristics~\citep{breiman1984classification} and optimal algorithms~\citep{bertsimas2017optimal} for computing decision trees.} 

\tb{As an illustration, we present an example of two tree policies in Figure \ref{fig:tree-policy-deterministic} (deterministic tree policy) and Figure \ref{fig:tree-policy-randomized} (randomized tree policy). These two tree policies are compatible with the decision tree in Figure \ref{fig:tree-zero}. We also note that Definition \ref{def:tree-policies} introduce tree policies that are {\em deterministic}. This simple definition is sufficient because optimal tree policies can always be chosen deterministic, as we highlight in Proposition \ref{prop:properties-optimal-policy}.}

Our goal is to compute an optimal tree policy, i.e., our goal is to solve the following \textit{Optimal Tree Policy} (OTP) problem:
\begin{equation}\label{eq:otp}\tag{OTP}
\max_{T \in \TT} \max_{\pi \in \Pi_{T}} \; R(\pi)
\end{equation}
\begin{figure}
\center
 \begin{subfigure}{0.30\textwidth}
\centering
         \includegraphics[width=0.9\linewidth]{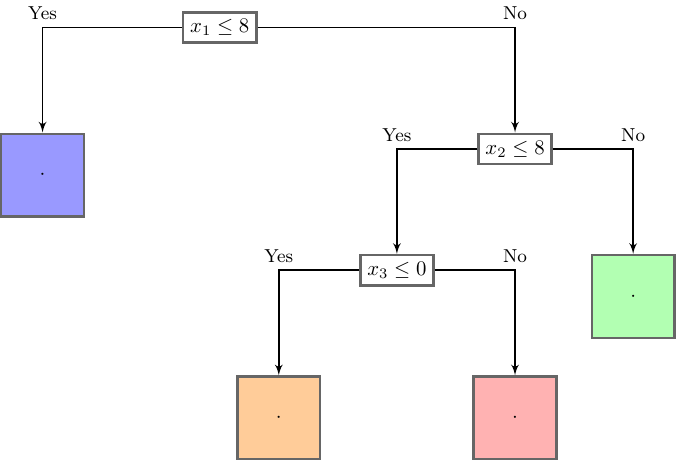}
         \caption{Decision tree}
         \label{fig:tree-zero}
  \end{subfigure}
 \begin{subfigure}{0.30\textwidth}
\centering
         \includegraphics[width=0.9\linewidth]{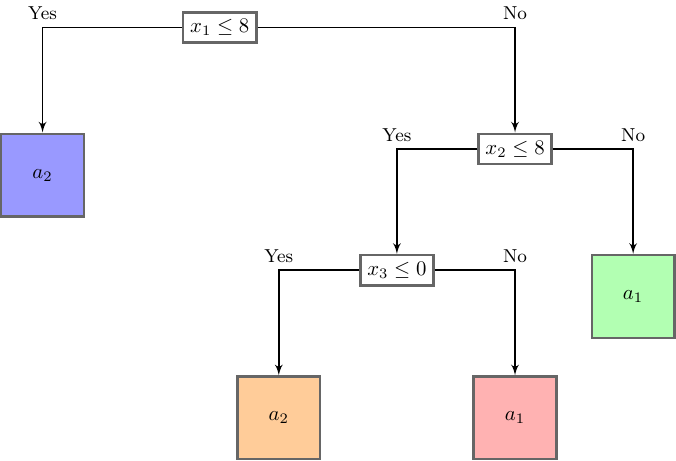}
         \caption{Deterministic decisions}
         \label{fig:tree-policy-deterministic}
  \end{subfigure}
   \begin{subfigure}{0.30\textwidth}
\centering
         \includegraphics[width=0.9\linewidth]{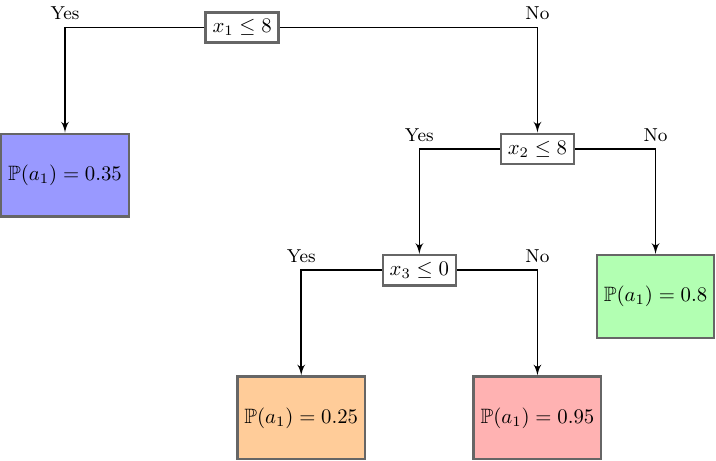}
         \caption{Randomized decisions}
         \label{fig:tree-policy-randomized}
  \end{subfigure}
 \caption{Example of a decision tree $T$ with four classes (Figure \ref{fig:tree-zero}), a deterministic tree policy in $\Pi_{T}$ (Figure \ref{fig:tree-policy-deterministic}) and a randomized tree policy in $\Pi_{T}$ (Figure \ref{fig:tree-policy-randomized}). }\label{fig:example-tree-policies}
\end{figure}
\subsection{Structural Properties of optimal tree policies}\label{sec:tree-policy-structural-properties}
\tb{Note that we have defined tree policies that are a priori history-dependent and randomized. An important subclass of tree policies is the class of {\em Markovian} tree policies, where the action chosen at time $t \in [H]$ only depends on the current class instead of the entire history up to time $t$. In the unconstrained MDP setting presented in Section \ref{sec:model}, an optimal policy may be chosen to be Markovian, deterministic, and independent of the initial distribution $\bm{p}_{1}$~\citep{Puterman}.  In the following proposition, we contrast these properties with the properties of optimal tree policies.}
\begin{proposition}\label{prop:properties-optimal-policy}
Consider an MDP instance $\M$, a set of sequence of trees $\TT$ admissible for $\M$ and an admissible sequence of trees $T \in \TT$.
\begin{enumerate}
\item The optimal tree policies \tb{compatible with $T$} may depend on the initial distribution $\bm{p}_{1}$.
\item The optimal tree policies \tb{compatible with $T$} may be history-dependent.
\item There always exists an optimal tree policy \tb{compatible with $T$} that is deterministic (even though it may be history-dependent).
\end{enumerate}
\end{proposition}
We present a detailed proof in Appendix \ref{app:properties-opt-policy}. 
\tb{Proposition \ref{prop:properties-optimal-policy} shows that optimizing over tree policies radically alters the structure of optimal policies, compared to classical unconstrained MDPs. Let us give some intuition on the optimality of history-dependent policies. An optimal decision rule $\pi_{t}\opt$ at time $t \geq 2$ must choose the same action for all states in a given subregion of $\X_{t}$ defined by the decision tree $T_{t}$. But the best actions to choose may differ across states in the same class.  Therefore, $\pi_{t}\opt$ may depend on the \textit{distribution} over the set of states $\X_{t}$ visited at time $t$. This distribution depends on the previous decisions $\pi_{1}\opt,...,\pi_{t-1}\opt$, which implies that $\pi_{t}\opt$ is history-dependent. Because implementing a history-dependent policy may be hard in practice, for the sake of simplicity, we will focus on the subclass of {\em Markovian} tree policies in the next section.}
\tb{
\subsection{The case of Markovian tree policies}\label{sec:markovian-policies}}
\tb{
In this section, we focus on Markovian tree policies. Despite their potential suboptimality (see Proposition \ref{prop:properties-optimal-policy}), they are easier to operationalize than history-dependent policies, and in the unconstrained setting they correspond to optimal policies. We now study their properties and the complexity of computing an optimal Markovian tree policy. We first provide the following definition.
\begin{definition}[Markovian tree policies]\label{def:tree-policies-markovian}
Consider an MDP instance $\M$ and a set of sequence of trees $\TT$ admissible for $\M$. Let $T \in \TT $ be an admissible sequence of decision trees: $T=T_{1} \in \T(\X_{1},[K_{1}]), ..., T_{H} \in \T(\X_{H},[K_{H}])$. A tree policy $\pi = \left(\pi_{1},...,\pi_{H}\right)$ compatible with $T$ is called {\em Markovian} if for all $t \in [H]$, the decision rule $\pi_{t}$ only depends of the current class in $[K_{t}]$ and not on the history of state-action $\mcH_{t}$ up to time $t$. 
\end{definition}
We write $\Pi_{T,{\sf M}}$ for the class of Markovian tree policies.
In the next proposition, we show that an optimal Markovian tree policy can be chosen deterministic. This is important for deploying policies in practice, especially in a healthcare setting. We present a detailed proof in Appendix \ref{app:properties-opt-policy}.
\begin{proposition}\label{prop:markovian-tree-policies-deterministic}
Consider an MDP instance $\M$, a set of sequence of trees $\TT$ admissible for $\M$ and a feasible sequence of trees $T \in \TT$. An optimal Markovian tree policy for $T$ can  be chosen deterministic.
\end{proposition}}
We now consider the complexity of computing an optimal Markovian policies. We start with the following simple setting.
\paragraph{The case $H=1$: relation with classification trees.}
With a horizon of $H=1$,  a tree policy is simply a decision rule that maps each state $s \in \X_{1}$ to a class $c \in [K]$ using a decision tree, then maps each class to an action $a \in \A_{1}$. Note that all tree policies are Markovian when $H=1$. We can identify the finite set of states $\X_{1}$ with the finite set of observations $\{\bm{x}_{1},...,\bm{x}_{m}\}$ and the finite set of actions $\A_{1}$ with the finite set of labels $\LL$. This shows that instances of \eqref{eq:otp} (with $H=1$) and instances of \eqref{eq:cart} are equivalent. The instantaneous rewards for the MDP instance play the role of the misclassification costs in the definition of the classification error \eqref{eq:classification-error-randomized}.  We provide a formal proof of the next proposition in Appendix \ref{app:proof-markovian}.
\begin{proposition}\label{th:equivalence-cart-otp}
 Any instance of \eqref{eq:otp} with $H=1$ can be reduced to an instance of \eqref{eq:cart} \tb{with the same set of optimal solutions}. Additionally, any instance of \eqref{eq:cart} can be reduced to an instance of  \eqref{eq:otp} with $H=1$ \tb{with the same set of optimal solutions}.
\end{proposition}
\tb{To the best of our knowledge, the complexity of \eqref{eq:cart} remains unsettled. However, the only algorithms for solving \eqref{eq:cart} are either heuristics~\citep{breiman1984classification} or based on mixed-integer linear programming~\citep{bertsimas2017optimal}}\footnote{\tb{Some NP-hardness results exist for the related problem of minimizing the depth of a classification tree that achieves perfect classification~\citep{laurent1976constructing}.}}\tb{, which suggests that \eqref{eq:cart} may be a difficult optimization problem.}
\paragraph{The case $H \geq 2$.}
\tb{When $H \geq 2$, we show in the next proposition that it is NP-hard to compute an optimal Markovian tree policy compatible with a given sequence of decision trees $T$. 
\begin{proposition}\label{prop:tree-policy-complexity}
Let $\M=(H,\X,\A,\bm{P},\bm{r},\bm{p}_{1})$ be an MDP instance with $H \geq 2$ and $\TT$ be a set of admissible trees. Let $T \in \TT$. Then it is NP-hard to compute an optimal Markovian tree policy compatible with $T$, i.e., the following optimization problem is NP-hard:
\begin{equation}\label{eq:otp-markov}
\max_{\pi \in \Pi_{T,{\sf M}}} R(\pi).
\end{equation}
\end{proposition}
\tb{
The proof of Proposition \ref{prop:tree-policy-complexity} is presented in Appendix \ref{app:proof-markovian}. Let us provide some intuition on the proof, where we reduce an instance of the minimum vertex cover (e.g., section 7.5 in \cite{sipser1996introduction}) to a particular instance of \eqref{eq:otp-markov}. Given a finite set of vertices $\mcV$ and a set of edges $\mcE \subseteq \mcV \times \mcV$, the Minimum Vertex Cover ({\sf MVC}) problem is to find a subset $\mcC \subseteq \mcV$ with the smallest cardinality, such that every edge in $\mcE$ has an endpoint in $\mcC$. Such an edge is said to be {\em covered} by a {\em selected} vertex in $\mcC$. We define a specific instance of \eqref{eq:otp-markov} that solves {\sf MVC}. We choose $H=3$, the set of states describes the set of vertices and possible endpoints, and the actions are simply {\sf select} and {\sf not select}. The constraints imposed by $T$ ensure that selected vertices exactly correspond to covered edges. We need $H \geq 2$ because edges have two endpoints, and therefore they may be selected via their first endpoint at $t=1$ or via their second endpoint at $t=2$; the MDP terminates at $t=3$. 
\begin{remark}
Our proof for Proposition \ref{prop:tree-policy-complexity} is close to the proof of Proposition 1 in \cite{ciocan2022interpretable} for the NP-hardness of computing an optimal stopping time compatible with a given sequence of decision trees. The main difference lies in the objective function which in \cite{ciocan2022interpretable} is different from the optimization problem \eqref{eq:otp-markov}, since \cite{ciocan2022interpretable} directly optimize the sample average approximation across all observations whereas we build the return $R(\pi)$ from the MDP parameters (rewards $\bm{r}$ and transitions $\bm{P}$) as in Equation \eqref{eq:expected-return}.
\end{remark}
}
}

\paragraph{Intuition on our NP-hardness result: the role of rectangularity.}
We finish this section by discussing the fundamental role of \textit{rectangularity} in solving classical unconstrained MDPs and contrasting it with the decision tree structure introduced in this paper. The rectangularity assumption is a common assumption in the robust optimization and robust MDP literature \citep{Iyengar,BBC08,Kuhn,goyal2023robust}. In the unconstrained MDP setting of Section \ref{sec:nominal-MDP}, the rectangularity assumption states that at each time $t$, the decisions $\pi_{t,s}$ for $s \in \X_{t}$ can be chosen independently across $s \in \X_{t}$. In particular, the choice of the decision maker in one state does not influence the choice of the decision maker at another state (at the same period).  In this case, the maximization problems defining the inductive updates \eqref{eq:Bellman-recursion} on the value function $\bm{v}\opt$ of the optimal policy $\pi\opt$ can be solved independently across $s \in \X_{t}$ for every period $t \in [H]$.  Crucially, this also implies the component-wise inequality $\bm{v}\opt_{t} \geq \bm{v}^{\pi}_{t}$ at any period $t$ for any Markovian policy $\pi$, since $\pi\opt$ is solving the maximization programs \eqref{eq:Bellman-recursion} at each state independently. From the definition of tree policies (Definition \ref{def:tree-policies}), it is clear that {\it the rectangularity assumption is  not satisfied for tree policies}, because if $\pi$ is a tree policy compatible with a given tree $T$, then $\pi$ must choose the same action for all the states in a given subregion of the state space (defined by a leaf of $T$). In this case, the equation \eqref{eq:Bellman-recursion} may not hold for an optimal tree policy. In particular, if there are some constraints across the decisions chosen at different states,  a policy attaining the $\arg \max$ in \eqref{eq:Bellman-recursion} may not even be feasible, i.e., it may not be compatible with any admissible sequence of decision trees.

\tb{
\begin{remark}[The case of stationary policies.]
Our main results in this section pertain to the complexity of computing optimal {\em Markovian} tree policies. For the sake of interpretability, in the case where $\X_{1}= \X_{2} = ... = \X_{H}$ and $\A_{1}=...=\A_{H}$, i.e., when the sets of states and actions are the same over time, it may be interesting to also consider stationary tree policies, i.e., tree policies that remain constant over the entire horizon. Unfortunately, even without any constraints on the sets of policies, computing optimal {\em stationary} policies for finite-horizon MDPs may be NP-complete~\citep{littman1994memoryless}. For this reason, in this paper, we leave aside the question of computing stationary tree policies and we focus on Markovian tree policies.
\end{remark}

}

\subsection{Our algorithms to compute Markovian tree policies}
\tb{
In this section, we introduce two iterative algorithms to compute Markovian tree policies. Recall that computing an {\em optimal} Markovian tree policy is NP-hard (see  Proposition \ref{prop:tree-policy-complexity}). In this section, we describe two heuristics. Both of our algorithms use the following subroutines:
\begin{itemize}
    \item $\optimize$ returns the best unconstrained decision rule that can be chosen at a given period, {\em conditioned on a continuation value}. In particular, for $t \in [H-1]$ and $\bm{v} \in \R^{\X_{t+1}}$, we define $\optimize(t,\bm{v})$ as $ \pi_{t}:\X_{t} \rightarrow \A_{t}$ such that
    \[\pi_{t}(s) \in \arg \max_{a \in \A_{t}} r_{s_{t}a_{t}} + \sum_{s' \in \X_{t+1}} P_{sas'}v_{s'}, \forall \; s \in \X_{t}.\]
    Note if $\bm{v} \in \R^{\X_{t+1}}$ is the value function of the optimal unconstrained policy, then $\optimize(t,\bm{v}_{t})$ returns an optimal unconstrained decision rule at period $t$. Otherwise, $\optimize(t,\bm{v}_{t})$ only returns the decision rule that maximizes the total return between time $t$ to time $H$, given that the value derived from time $t+1$ to time $H$ is given by $\bm{v} \in \R^{\X_{t+1}}$.
    Note that $H$ is the last period of the decision process, and therefore, continuation values are irrelevant at $t=H$. Therefore we simply define $\optimize(H,\bm{v})$ as $ \pi_{H}:\X_{H} \rightarrow \A_{H}$ such that 
    \[\pi_{H}(s) \in \arg \max_{a \in \A_{H}} r_{s_{H}a_{H}}, \forall \; s \in \X_{H}.\]
    \item $\updval$ returns the value function at period $t$, given a period $t$, a (deterministic) decision rule at period $t$ and a continuation value at period $t+1$. In particular, we define $\updval(t, \pi,\bm{v}) \in \R^{\X_{t}}$ for $t \in [H-1]$, $\pi_t:\X_{t} \rightarrow \A_{t}$ and $\bm{v} \in \R^{\X_{t+1}}$, as
    \[\updval(t, \pi_t,\bm{v})_{s} = r_{s\pi_t(s)} +  \sum_{s' \in \X_{t+1}} P_{s\pi_t(s)s'} v_{s'} ,  \forall \; s \in \X_{t}.\]
    For $t = H$ we simply define
    \[\updval(H, \pi_H,\bm{v})_{s} = r_{s\pi_H(s)},  \forall \; s \in \X_{H}.\]
    \item $\fit(\pi,\T)$ returns a decision tree and a labeling rule given a deterministic decision rule $\pi$ and a set of admissible trees $\T$, based on \eqref{eq:cart}. In particular, given a period $t \in [H]$, an unconstrained decision rule $\pi_t:\X_{t} \rightarrow \A_{t}$ can be interpreted as a set of points $\{(s,\pi_t(s)) \; | \; s \in \X_{t}\}$, where the state $s \in \X_{t}$ represents the observation (covariates) while $\pi_t(s)$ represents the label of state $s$. The function $\fit(\pi,\T(\X_t,[K_t])$ builds a decision tree $T_t$ classifying states in $\X_{t}$ into classes in $[K_t]$ and a labeling rule $\mu_t$ mapping classes in $[K_t]$ to actions in $\A_t$, with the labels $\pi_t(s)$ as the true label for state $s \in \X_{t}$. We then interpret $(\mu_1,...,\mu_H)$ as a tree policy compatible with the sequence of trees $(T_{1},...,T_{H})$. Since $T_t$ is a map $\X_t \rightarrow [K_t]$ and $\mu_t$ a map $[K_t] \rightarrow \A_t$, we interpret $\mu_t \circ T_t$ as a decision rule $\X_t \rightarrow \A_t$. For the sake of conciseness, the choice of hyperparameters (e.g., maximum depth of the admissible trees) is considered to be described in the set of admissible trees $\T$.
\end{itemize}
}
\tb{
With these notations, we are now ready to describe our algorithms.
Intuitively, our first algorithm (Algorithm \ref{alg:optimize then fit}) solves the (unconstrained) MDP by computing value functions with backward induction as in \eqref{eq:Bellman-recursion}, and then fits trees to an optimal unconstrained policy. Our second algorithm (Algorithm \ref{alg:optimize and fit}) also proceeds inductively, but value functions are computed using {\em interpretable} decision rules (i.e., decision trees) instead of optimal decision rules as in Algorithm \ref{alg:optimize then fit}. Both algorithms return a sequence of decision trees $(T_{1},...,T_{H})$ and labeling rules $(\mu_1,...,\mu_H)$ used to choose actions at periods $1,...,H$. We emphasize that the only difference between Algorithm \ref{alg:optimize then fit} and Algorithm \ref{alg:optimize and fit} is in Step 5 of both algorithms.
We start by analyzing our algorithms in the following proposition.
\begin{proposition}\label{prop:analysis of our algorithms}
    \begin{enumerate}
        \item Let $\left(\pi_t\right)_{t \in [H]}$ and $\left(\bm{v}_t\right)_{t \in [H]}$ be computed by Algorithm \ref{alg:optimize then fit}. For every period $t \in [H]$, the following holds:
        \begin{itemize}
            \item $\pi_t$ is the best unconstrained decision rule that can be chosen at period $t$, given that the decision-maker subsequently chooses $\pi_{t+1},...,\pi_{H}$.
            \item $\bm{v}_{t} \in \R^{\X_{t}}$ is the value function for choosing $\pi_t,...,\pi_{H}$ at periods $t,...,H$.
        \end{itemize}
     Additionally, $(\pi_{1},...,\pi_{H})$ is an optimal {\em unconstrained} policy.
        \item Let $\left(\pi_t\right)_{t \in [H]}$ and $\left(\bm{v}_t\right)_{t \in [H]}$ be computed by Algorithm \ref{alg:optimize and fit}. For every period $t \in [H]$ the following holds:
        \begin{itemize}
            \item $\pi_t$ is the best unconstrained decision rule that can be chosen at period $t$, given that the decision-maker subsequently chooses $\mu_{t+1} \circ T_{t+1},...,\mu_{H} \circ T_{H}$.
            \item $\bm{v}_{t} \in \R^{\X_{t}}$ is the value function for choosing $\mu_{t} \circ T_{t},...,\mu_{H} \circ T_{H}$ at periods $t,...,H$.
        \end{itemize}
    \end{enumerate}
\end{proposition}
Proposition \ref{prop:analysis of our algorithms} is a consequence of Bellman equations as in \eqref{eq:Bellman recursion - policy evaluation} and \eqref{eq:Bellman-recursion}, and it highlights a stark contrast between Algorithm \ref{alg:optimize then fit} and Algorithm \ref{alg:optimize and fit}. In Algorithm \ref{alg:optimize then fit}, $(\pi_1,...,\pi_H)$ is an optimal unconstrained policy and Algorithm \ref{alg:optimize then fit} computes the tree $T_{t}$ and labeling rule $\mu_t$ by fitting a tree to $\pi_t$. One potential shortcoming of this is that for any $t \in \{1,...,H-1\}$,  $\pi_{t}$ is the optimal choice of decision rule for period $t$ \textit{only} if the decision maker chooses $\pi_{t+1}, ..., \pi_{H}$ subsequently, i.e., only if the subsequent decision rules are also optimal. This follows from the induction in the Bellman optimality equation \eqref{eq:Bellman-recursion}. However, the decision-maker is interested in deploying the interpretable decisions $\mu_1 \circ T_1,...,\mu_H \circ T_H$, which may be different from $\pi_{t},...,\pi_{H}$. Therefore, in Algorithm \ref{alg:optimize then fit},  $\pi_{t}$ may not be optimal at period $t$, because the subsequent decision rules chosen by the decision-maker are $\mu_{t+1} \circ T_{t+1},...,\mu_{H} \circ T_{H}$ and not $\pi_{t+1},...,\pi_{H}$, and it may be irrelevant to fit a tree to $\pi_{t}$. 

To remedy this potential shortcoming,  Algorithm \ref{alg:optimize and fit} incorporates knowledge about the subsequent interpretable decision rules $\mu_{t+1} \circ T_{t+1},...,\mu_{H} \circ T_{H}$ (already computed) when returning the unconstrained decision rule $\pi_{t}$ for period $t$.
In particular, Algorithm \ref{alg:optimize and fit} updates the value functions $\bm{v}_{t}$ by accounting for the decision tree $T_{t}$ and labeling rule $\mu_t$ that are computed for period $t$ instead of the unconstrained decision rule $\pi_t$. This ensures that at every period $t$, $\bm{v}_{t} \in \R^{\X_{t}}$ is the continuation value associated with choosing $\mu_{t+1} \circ T_{t+1},...,\mu_{H} \circ T_{H}$ (already computed) for the next periods, $\pi_t$ is the best unconstrained decision rule for period $t$ {\em given the value function $\bm{v}_{t+1}$}, and $T_{t}$ and $\mu_t$ are computed by fitting a tree to $\pi_t$. 
}

\tb{
\begin{figure}[htb]
\begin{minipage}[b]{.5\textwidth}
\begin{algorithm}[H]
\caption{}\label{alg:optimize then fit}
\begin{algorithmic}[1]
\State {\bf Initialization:} Set $\bm{v}_{H+1} = \bm{0}.$
\For{$ t=H,...,1$}
\State 
\label{alg:opt-then-fit-step:compute optimal decision rule} $\pi_t = \optimize(t,\bm{v}_{t+1})$
\State \label{alg:opt-then-fit-step:fit a tree} $(T_{t},\mu_t) = \fit(\pi_t,\T(\X_{t},[K_{t}]))$
\State \label{alg:opt-then-fit-step:value-update}  $\bm{v}_{t} = \updval\left(t,\pi_t,\bm{v}_{t+1}\right)$
\EndFor
\State {\bf Output:} $\left(T_{1},...,T_{H}\right),(\mu_1,...,\mu_H)$.
\end{algorithmic}
\end{algorithm}
\end{minipage}
\begin{minipage}[b]{.5\textwidth}
 \begin{algorithm}[H]
\caption{}\label{alg:optimize and fit}
\begin{algorithmic}[1]
\State {\bf Initialization:} Set $\bm{v}_{H+1} = \bm{0}.$
\For{$ t=H,...,1$}
\State 
\label{alg:opt-and-fit-step:compute optimal decision rule} $\pi_t = \optimize(t,\bm{v}_{t+1})$
\State \label{alg:opt-and-fit-step:fit a tree} $(T_{t},\mu_t) = \fit(\pi_t, \T(\X_{t},[K_{t}]))$
\State \label{alg:opt-and-fit-step:value-update}  $\bm{v}_{t} = \updval\left(t,\mu_t \circ T_t,\bm{v}_{t+1}\right)$
\EndFor
\State {\bf Output:} $\left(T_{1},...,T_{H}\right),(\mu_1,...,\mu_H)$.
\end{algorithmic}
\end{algorithm}
\end{minipage}
\end{figure}
}

Several remarks are in order.

\tb{
First, we note that both of our algorithms have the same numerical complexity since they require $H$ calls to the subroutines $\optimize,\fit$, and $\updval$. In particular, we have the following proposition.
\begin{proposition}\label{prop:alg complexity}
    Let $\comp$ be the maximum number of arithmetic operations required for evaluating $\fit$ in Step \ref{alg:opt-then-fit-step:fit a tree} of Algorithm \ref{alg:optimize then fit} and Step \ref{alg:opt-and-fit-step:fit a tree} of Algorithm \ref{alg:optimize and fit}. Let $A = \max_{t \in [H]} |\A_{t}|$ and $S = \max_{t \in [H]} |\X_{t}|$. Then both Algorithm \ref{alg:optimize then fit} and Algorithm \ref{alg:optimize and fit} terminate in a number of arithmetic operations that is $O\left( H \cdot \left(S^{2}A + \comp\right)\right)$.
\end{proposition}
Note that computing an optimal unconstrained policy can be done in $O(HS^{2}A)$ arithmetic operations via the Bellman equations~\eqref{eq:Bellman-recursion}. Compared to this, our algorithms require an additional $O(H \cdot \comp)$ number of arithmetic operations to fit trees at every period $t \in [H]$ with the subroutine $\fit$. We would like to emphasize that calling $\fit$ at every period is the main bottleneck of our algorithms. Indeed, this can prove computationally intensive, because the only {\em exact} methods to compute $\fit$ are based on mixed-integer linear programming~\citep{bertsimas2017optimal}. In our simulations in the next section, we will use the fast Python implementation of \cite{mctavish2022fast} for computing optimal trees. We note that for large instances where computing optimal trees may be too long, some fast heuristics are also available, e.g. the {\sf scikitlearn} Python package includes an improved implementation of CART~\citep{breiman1984classification}.
}

\tb{
Second, we would like to note that Algorithm \ref{alg:optimize then fit} and Algorithm \ref{alg:optimize and fit} are both heuristics.  Because an optimal tree policy may be history-dependent (see Proposition \ref{prop:properties-optimal-policy}), it appears difficult to provide theoretical guarantees on the performances of the policies returned by our algorithms. For this reason, we investigate the question of the potential suboptimality of the tree policies returned by Algorithm \ref{alg:optimize then fit} and Algorithm \ref{alg:optimize and fit} numerically in the next sections. In particular, we will see in our numerical study that our algorithms compute tree policies that can outperform state-of-the-art allocation guidelines, in the case of ventilator triage for COVID-19 patients, and  can perform as well as optimal unconstrained policies.
}

\tb{
Third, recall that we have shown in Section \ref{sec:tree-policy-structural-properties} that optimal tree policies at time $t$ depend on the {\em distribution} over $\X_{t}$ induced by the decisions at time $1,...,t-1$. In particular, in Algorithm \ref{alg:optimize and fit}, the pair $(T_{t},\mu_t)$ is computed by fitting a tree to $\pi_{t}$, which maps each state $s_{t} \in \X_{t}$ to a decision $a_{t} \in \A_{t}$. However, if $s_{t}$ is visited very rarely (given the previous decisions), including it in the $\fit$ at period $t$ will entice the tree $T_{t}$ to mimic $\pi_t$ even at a state that is not visited often (state $s_{t}$). If at period $t \in [H]$ of Algorithm \ref{alg:optimize and fit} we knew the {\em distributions} over $\X_{t}$ induced by the previous decisions $\mu_{1}\circ T_{1},...,\mu_{t-1}\circ T_{t-1}$, at period $t$ we could decide to remove from $\fit$ all the states that are not visited often (e.g. less than $1 \%$ of the time). The main issue with this reasoning is that at period $t$ in Algorithm \ref{alg:optimize and fit}, the trees $T_{1},...,T_{t-1}$ and labeling rules $\mu_1,...,\mu_{t-1}$ have not been computed yet. To improve upon this issue, it is possible to construct a ``looping" variant of Algorithm \ref{alg:optimize and fit}, by running Algorithm \ref{alg:optimize and fit} multiple times, and using the distributions induced by the previous output of Algorithm \ref{alg:optimize and fit} to compute the next tree policies. Our numerical experiments to implement this idea have not shown substantial improvements in the performance of the obtained tree policies. We leave this question as an interesting future direction.
}

\tb{Finally, we note that for the sake of interpretability, in practice we may favor tree policies that vary slowly over time. We note that when implementing our model of tree policies on real decision problems (as we do in the next two sections), it is possible to increase the interval of time between consecutive time steps $t$ and $t+1$ so that the tree policies do not vary too fast. Incorporating this constraint explicitly in the current model appears difficult and we leave this question open for future work. }
\section{Mechanical Ventilator Triage for COVID-19 Patients}\label{sec:Monte-learning-simu-setup}
In this section, we apply the methodology developed in Section \ref{sec:interpretable} to develop interpretable triage guidelines for allocating ventilators to COVID-19 patients and compare them to existing triage guidelines.

\subsection{Current triage guidelines}
New York State (NYS) policy follows the 2015 Ventilator Triage Guidelines which were recommended by the NYS Taskforce on Life and the Law~\citep{zuckerventilator}. These guidelines were designed to ration critical care services and staff following a disaster (e.g., a Nor'easter, a hurricane, or an influenza epidemic). In particular, the guidelines outline clinical criteria for triage  of patients using the Sequential Organ Failure Assessment (SOFA) score,  a severity score that has been shown to correlate highly with mortality in COVID-19~\citep{zhou2020clinical}.

The goal of the NYS guidelines is to maximize the number of lives saved.
There are three decision epochs: at  \textit{triage}, the first time that a patient requires a ventilator; and thereafter, at two \textit{reassessment periods} after 2 and 5 days of intubation. The NYS triage guidelines define \textit{priority classes} -- low, medium, and high -- based on the current SOFA score of the patient. At the reassessment periods, the classification also depends on the improvement/deterioration of the SOFA score since the last decision epoch.  A patient with \textit{low} priority class will either be excluded from ventilator treatment (at triage) or removed from the ventilator (at reassessment). These patients typically either can be safely extubated, or have low survival chance (e.g., SOFA score $>$ 11).  Patients with \textit{medium} priority class are intubated and maintained on a ventilator, unless a patient with \textit{high} priority class also requires intubation, in which case they are extubated and provided with palliative care.
 The guidelines at triage and reassessments admit simple tree representations (even though they were originally presented with tables). Details about the NYS guidelines are provided in Appendix \ref{app:details-NYS}. We note that the NYS guidelines were defined in 2015  and, hence, are \textit{not} specifically calibrated to the disease dynamics  of COVID-19 patients. 
\subsection{Markov Model for COVID-19 Ventilator Triage}\label{subsec:MDP-triage-model}
We now formalize the MDP model for COVID-19 ventilator allocation. \tb{Recall that the allocation guidelines are meant to be deployed {\em only} in the case where there are not enough ventilators, and in this case, the guidelines do not take into account the length of the queue, i.e., the total number of patients currently waiting to be intubated. Therefore, the states of our MDP model only capture the health condition of an individual patient, and we capture the cost of using resources with an appropriate choice of reward parameters.}

\paragraph{Decision epochs, states and epochs.} 
Recall that the NYS guidelines have only three decision periods: at triage and two reassessments.  Therefore, we consider an MDP model where there are $H=4$ periods; the last period corresponds to discharge of the patients after 5 days (or more) of intubation.  Our MDP model is depicted in Figure \ref{fig:MDP-model}.  After 0/2/5 days of intubation, the patient is in a given health condition captured by the SOFA score (value of SOFA, and increasing/decreasing value of SOFA compared to last decision time), which changes dynamically over time,  as well as static information on comorbidities and demographics. For the sake of completeness, we will also consider alternative decision epochs in our numeric experiments.

At triage (0 day of intubation, $t=1$ in the MDP), the decision maker chooses to {\it allocate}  the ventilator to this patient, or \textit{exclude} the patient from ventilator treatment.  After 2 days or 5 days of intubation (second and third period in the MDP), the decision is whether or not to {\it maintain} the patient on the ventilator. After the decision is executed, the patient transitions to another health condition or is nominally extubated before the next period which will correspond to a terminal state $D_{t}$ or $A_{t}$,  indicating the status at discharge from the hospital ({\bf D}eceased or {\bf A}live) (see Figure \ref{fig:MDP-model-maintain}).
If the patient is excluded from ventilator treatment, s/he transitions to state $D_{t}^{\sf ex}$ (if s/he dies at discharge) or $A_{t}^{\sf ex}$ (if s/he recovers at discharge), see Figure \ref{fig:MDP-model-exclude}. \tb{It is important to note that we do not consider the future health evolution of patients excluded from ventilator treatment after triage, we only consider their outcomes at discharge (deceased or alive): per the NYS guidelines, triaged patients that are not allocated a ventilator will {\em not} be further considered for ventilator allocations and will access palliative care, whereas our MDP instance is used to model the trajectories of intubated patients.}

\paragraph{Rewards parameters.}
\tb{
In our MDP instance, the rewards are only associated with the terminal states, and we design the rewards to capture the following aspects of our decision problems:
\begin{enumerate}
    \item {\bf Patient outcomes.} We want to ensure that our policies maximize the survival among the patient population. To do so we need the {\sf Alive} final states to have much larger rewards than the {\sf Deceased} final states. 
    \item {\bf Use of resources.} We want to ensure that our policies also minimize excessive use of ventilators. To do so, we penalize keeping the patients intubated for multiple periods. 
    \item {\bf Opportunity costs.} We want our policy to refrain from intubating patients with a large probability of death (even when intubated), so as to keep ventilators for patients with a larger chance of survival (when intubated). 
\end{enumerate}
To capture these three important aspects, we design our reward functions based on the three following parameters:
\begin{enumerate}
    \item To capture {\bf patient outcomes}, we introduce a parameter $C > 1$ corresponding to a large multiplicative bonus for {\sf Alive} outcomes (compared to {\sf Deceased} outcomes).
    \item To capture the {\bf use of resources}, we introduce a parameter $0<\rho <1$ corresponding to a ``discount factor" incurred for keeping the patients intubated for an additional period.
    \item To capture the {\bf opportunity costs}, we introduce a parameter $0<\gamma <1$ corresponding to a multiplicative penalty (resp. bonus) for patients that are proactively extubated but are discharged alive (resp. discharged diseased). That is, the extubation decision did not appear to adversely affect the patient's outcome.
\end{enumerate}
}
\tb{
It is important to choose the multiplicative bonus $C$ much larger than $1$, to appropriately entice our policies to avoid the {\sf Deceased} states. In our simulation, we choose $C=100$, and we provide a sensitivity analysis for the values of $C$ in Appendix \ref{app:sensitivity-analysis}.
}  
\tb{
  To penalize for using the limited resource, the reward for ventilator use decreases by a multiplicative factor $\rho < 1$ for each period of ventilator use.  The parameter $\rho$ can be understood as a discount factor, e.g., $\rho = 0.9 $ means that future rewards are decreased by $10 \%$ (per period) compared to the same outcome at the current period.  This discount factor reflects the desire to use resources for shorter periods.  We choose $\rho=0.9$ in our simulation. Finally,  among patients who survive,  a multiplicative penalty of $\gamma<1$ is given to patients who have been extubated, while for patients who die,  a multiplicative bonus of $1/\gamma >1$ is given to patients who have been extubated.   We choose $\gamma = 0.5$ in our simulation. When choosing the values for $C, \rho$ and $\gamma$, one needs to be careful to maintain the terminal ``rewards'' for deceased patients smaller than the terminal rewards for patients discharged alive. In particular, we recommend values of $C,\rho$ and $\gamma$ to maintain $C\rho^2 \gamma^2 >1$,  so that the ``best" {\sf Deceased} state ($D_{1}^{\sf ex}$, with reward $r_{D_{1}^{\sf ex}}=1/\gamma$) has a smaller reward than the ``worst" {\sf Alive} state ($A_{3}^{\sf ex}$, with reward $ r_{A_{3}^{\sf ex}} = C \rho^2 \gamma$).
  }
\tb{
Overall, our reward functions can be represented as, for $t \in \{1,2,3\}$:
\[ r_{A_{t}} = C \rho ^{t-1}, r_{A^{\sf ex}_{t}} = C \gamma \rho^{t-1}, r_{D_{t}} =  \rho^{t-1}, r_{D_{t}^{\sf ex}} =  \gamma^{-1} \rho^{t-1}.\]
While this reward parametrization is an artifact of our model and is necessary to compute the triage policy, we evaluate the performance of the resulting policy through simulations.  We also conduct a sensitivity analysis, where we vary the values of $C,\rho$ and $\gamma$, and study the changes in the estimated performances of the corresponding optimal policies and tree policies in the MDP.  We observe stable performances in the simulation model for the optimal policies and tree policies for a wide range of parameters $C, \rho,\gamma$, as long as $C\rho^2\gamma^2$ is much larger than $1$. We present the sensitivity analysis in Appendix \ref{app:sensitivity-analysis}.}

\begin{figure}[htb]
\center
 \begin{subfigure}{0.48\textwidth}
\centering
         \includegraphics[width=0.95\linewidth]{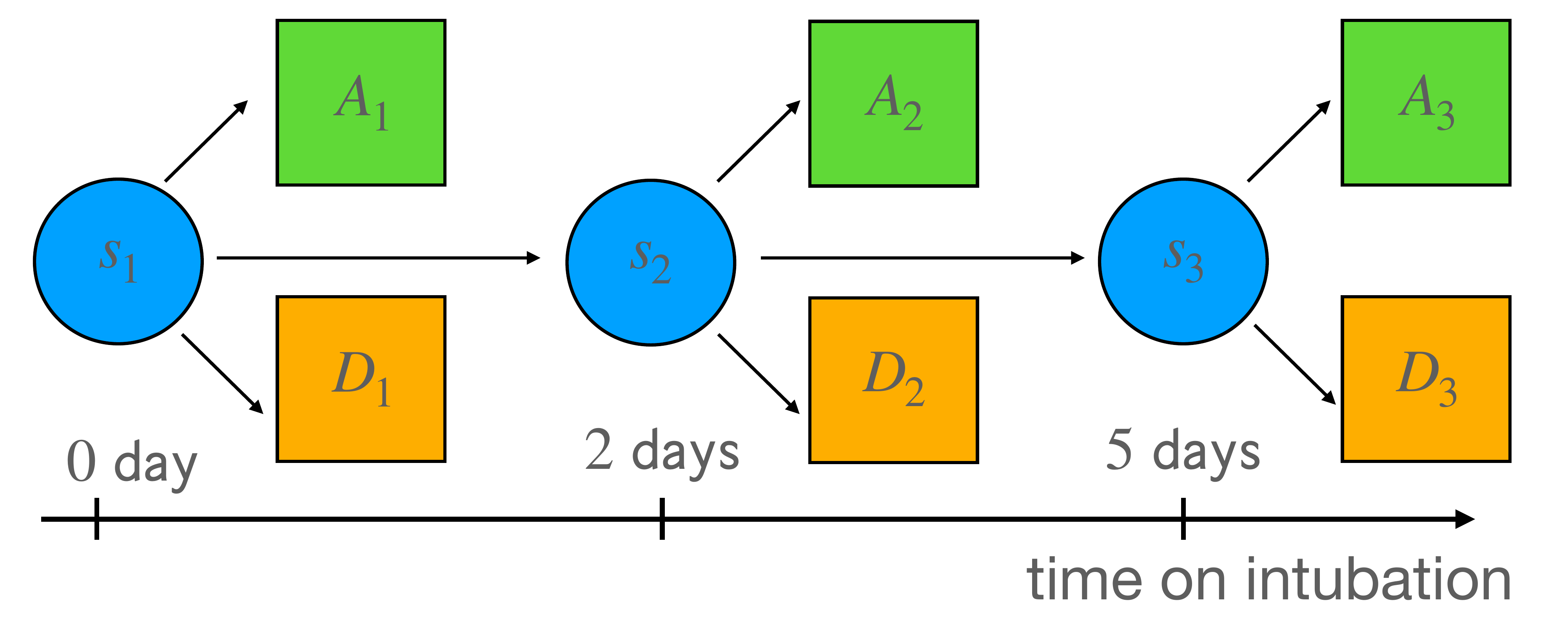}
         \caption{Transition for action = \textit{maintain}.}
         \label{fig:MDP-model-maintain}
  \end{subfigure}
   \begin{subfigure}{0.48\textwidth}
\centering
         \includegraphics[width=0.95\linewidth]{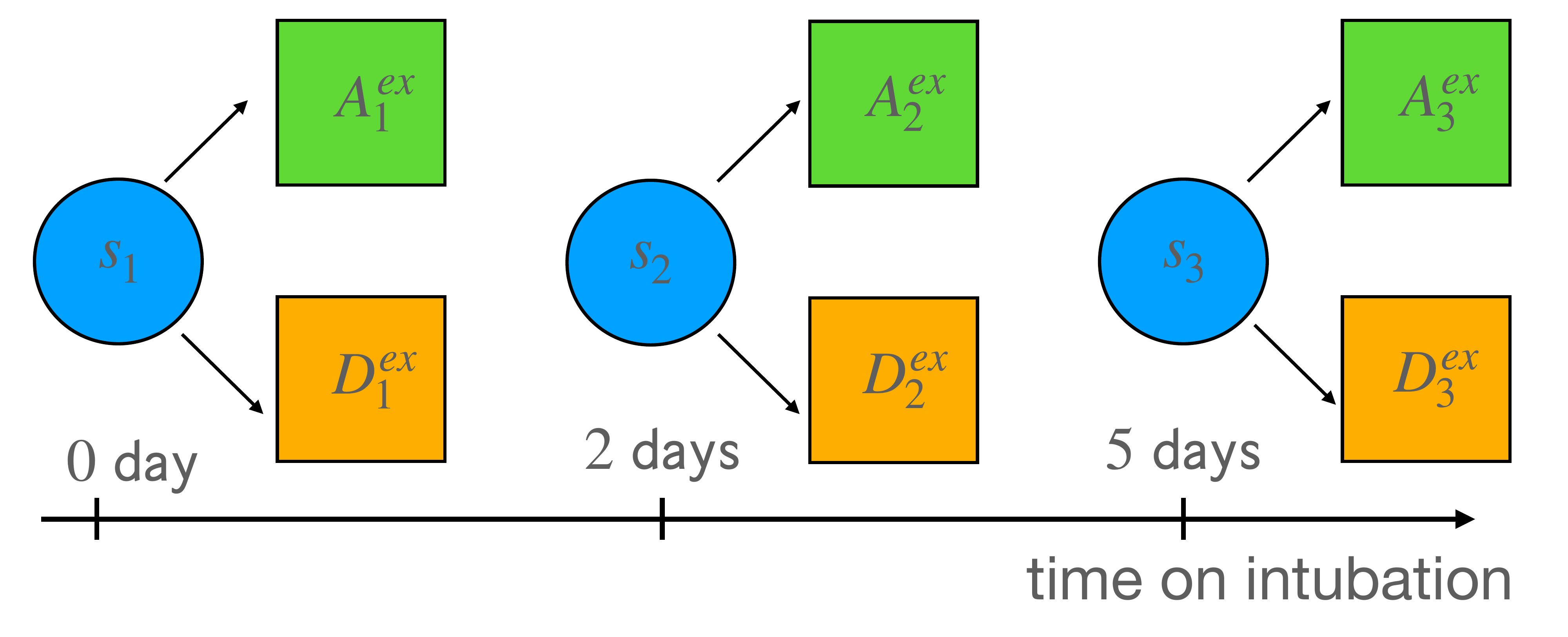}
         \caption{Transition for action = \textit{exclude}.}
         \label{fig:MDP-model-exclude}
  \end{subfigure}
  \caption{States and transitions in our MDP model with actions \textit{`maintain'} and \textit{`exclude'}.}
  \label{fig:MDP-model}
\end{figure}
\subsection{Data set and parameter estimation}\label{subsec:triage-data-set}
To calibrate our model, we utilize a retrospective data of  807 COVID-19 hospitalizations at the Montefiore Medical Center.  In particular,  we include patients with  confirmed laboratory test (real-time reverse polymerase chain reaction) for SARS-CoV-2 infection, admitted to three acute care hospitals within a single urban academic medical center   from 3/01/2020 and 5/27/2020, with respiratory failure requiring mechanical ventilation. This hospital system is located in the Bronx, NY, which was one of the hardest hit neighborhoods during the initial COVID-19 surge seen in the United States. This study was approved by the Albert Einstein College of Medicine/Montefiore Medical Center and Columbia University IRBs.

Each hospitalization corresponds to a unique patient.  For each patient, we have patient level admission data. This includes demographics such as age, gender, weight, BMI,  as well as comorbid health conditions such as the Charlson score, diabetes, malignancy, renal disease, dementia,  and congestive heart failure.  Our data provides admission and discharge time and date, as well as status at discharge from the hospital (deceased or alive). Finally,  every patient in our data set is assigned a SOFA score, which quantifies the number and severity of failure of different organ systems ~\citep{jones2009sequential},  and which is updated every two hours.
 The ventilator status of the patients (intubated or not) is also updated on a two-hours basis.

The hospital that we study was able to increase ventilator capacity through a series of actions, including procurement of new machines, repurposing ventilators not intended for extended use, and borrowing of machines from other states. The maximum number of ventilators that were used for COVID-19 patients over the study period was 253. The hospital never hit their ventilator capacity, so triage was never required to determine which patients to intubate.

The mean SOFA score at admission was 2.0 (IQR: [0.0,3.0]) and the maximum SOFA score over the entire hospital stay was 9.7 (IQR: [8.0,12.0]).  The mean age was 64.0 years (SD 13.5). The patients who survived were significantly younger than those who died in the hospital (p$<$0.001). The average SOFA at time of intubation was 3.7 (IQR: [1.0,6.0]), at 2 days of intubation it was 6.3 (IQR: [4.0,9.0]) and at 5 days of intubation it was 5.9 (CI: [3.0,8.0]).
 Details and summary statistics about the data set and our patient cohort are presented in Table \ref{tab:summary-statistics} in Appendix \ref{app:details-data-set}.

\subsection{Model Calibration}
To calibrate the MDP model, we utilize {\em static} patient data assigned at the time of admission (e.g.,  history of comorbodities), and {\em dynamical} patient data updated on a two-hour basis (e.g., SOFA score and intubation status). \tb{We calibrate the transition rates across states, as well as the likelihood of deaths and survival at each period,  using the empirical frequencies for these events from the data. We consider three MDP models, corresponding to three different sets of states.

The first MDP model only relies on SOFA scores and the direction of change of the SOFA score (decreasing or increasing, compared to last triage/reassessment period) to make the allocation decision. In this MDP model, the set of states at time $1$ is just $\{1,...,18\}$ (the possible SOFA scores), while at time $2$ (reassessment after 2 days) and at time $3$ (reassessment after 5 days) the sets of states are $\{1,...,18\} \times \{+,-\}$, with $``+"$ (resp. $``-"$) indicating that the current SOFA score is strictly larger (resp. smaller or equal) than the previous SOFA score. We estimate the transition probabilities by directly estimating the empirical frequencies of the observed transitions in our dataset.

The second MDP model relies on SOFA scores, the direction of change of the SOFA score, and the age of the patient to make the allocation decision, while the third MDP model additionally relies on the comorbidities and demographics. For these two MDP models, the number of possible states is too large (this is also known as the curse of dimensionality) and most of the possible transitions are not observed in our dataset. For this reason, we need to reduce the total number of states and to increase the number of observations and transitions per states. To do so, we first create $k=10$ clusters (using the {\sf KMEANS} function from the {\sf scikitlearn} Python package) for the patients's age (for the second tree policies) and for the patient's age and comorbidities (for the third tree policies). \tb{We present a detailed description of the resulting clusters in Appendix \ref{app:details cluster}.} A state then consists of a cluster label, and the current SOFA score, and the direction of change of the SOFA score (decreasing or increasing, compared to last triage/reassessment period). The cluster label encodes the additional information, e.g. age or comorbidities. We then estimate the transition probabilities with the empirical frequencies of observed transitions {\em across clusters}.}

Note that in our data we do not observe any early extubation (i.e., we only observe extubation when it is safe or when the patient is deceased). Therefore, we cannot estimate the transition rates to $D^{\sf ex}_{t}$ (the death state if extubated between period $t$ and the next period).
We use a single parameter $p \in [0,1]$ for the transitions to $D^{\sf ex}_{t}$. We choose a uniform $p$ across periods $t \in \{1,2,3\}$ and states. This gives a range of estimates, from optimistic estimates ($p=0$) to more realistic estimates ($p \geq 0.90$), with values $p=0.90$ and $p=0.95$ being closer to the death rates estimated by our clinical collaborators. We use $p=0.99$ for learning policies in our numerical experiments.

We note that some patients may be intubated more than once during their hospital visits. This can happen when the health condition of an intubated patient first improves, the patient is safely extubated, and then the health condition starts to deteriorate again.  We do not account for this in our MDP model, as this is a relatively rare event. In our data, it occurs in only 5.7\% of the patients.  Therefore, we treat second ventilator uses as {\it new trajectories}, starting from period $t=1$.  While the dynamical health evolution of the patient who are reintubated may differ from the dynamics of the patients who are intubated for the first time, we emphasize that only the computational part (i.e., computation of  tree policies for the MDP in Figure \ref{fig:MDP-model}) is based on this approximation. The evaluation of patient survival with our simulation model does not depend on this approximation.

 \subsection{Policy Computation}
\tb{
We use Algorithm \ref{alg:optimize then fit} and Algorithm \ref{alg:optimize and fit} to compute several tree policies, based on the three MDP models described in the previous section using different covariates describing the health conditions of the patients in the state space. We also compare our tree policies with the optimal unconstrained policies obtained in the three MDP models.

In particular, we first compute tree policies only using SOFA (Figure \ref{fig:tree-sofa}), since this is the only covariate used in the NYS guidelines. 
We then compute tree policies based on SOFA and age of the patient (Figure \ref{fig:tree-sofa-age-algo-1} and Figure \ref{fig:tree-sofa-age-algo-2}), as there is some debate about making categorical exclusions based on age: for example,  the NYS guidelines break ties based on age. We also compute tree policies based on all the comorbid conditions and demographic information available in our data (Figure \ref{fig:tree-sofa-cov-algo-1} and Figure \ref{fig:tree-sofa-cov-algo-2}).  
When we include covariates other than SOFA scores, we create 10 clusters to reduce the final number of states and transitions.  For instance, for the policies based on SOFA scores and age,  the final states in the MDP of Figure \ref{fig:MDP-model} consists of pairs $({\sf sofa},\ell,+/-)$ where ${\sf sofa}$ is the current SOFA score, $\ell$ is a cluster label describing the age of the patient, and $+$ or $-$ captures whether patient condition is improving or worsening.  We present the details of the covariates used in the full tree policy along with the corresponding trees for each computed policy in Appendix \ref{app:tree-policies}.
}
 
We note that one needs to be cognizant of potential ethical considerations when including certain covariates. For instance, diabetes has been shown to be correlated with higher risk of severe COVID disease~\citep{orioli2020covid}. However, increased prevalence of diabetes (and other risk factors for severe COVID patients) is observed in underserved communities who have suffered from structural health inequities.  Excluding patients from ventilator access on the basis of such covariate could further exacerbate these long-standing structural inequities.  As such, there are several ethical justifications for the absence of categorical exclusion criteria from the triage decisions~\citep{white2020framework}. During the COVID-19 pandemic, there has been a movement away from including these covariates in triage algorithms, because of the potential to exacerbate inequities~\citep{mello2020respecting}.
We believe  there is value in estimating the potential benefits (or not) of including as much information as possible in the triage guidelines, to provide quantitative data on the consequences of these choices to inform this discourse.

\tb{
\begin{remark}
    In practice, when computing trees, one needs to choose some important hyperparameters, such as the maximum depth. There is a clear tradeoff between deeper trees (which are more powerful classifiers but may be difficult to interpret) and shallower trees (which may be poor classifiers). We recommend using the parameters that match the official guidelines (if such guidelines exist), or involving human experts to check for the interpretability of the resulting trees (if official guidelines do not exist). 
    Since the NYS guidelines, interpreted as tree policies, have a maximum depth of $4$ (see Figure \ref{fig:NYS-guidelines} in Appendix \ref{app:details-NYS}), in our simulations we constrain the set of admissible decision trees to have a depth of at most $4$. 
\end{remark}
}

\subsection{Simulation model}\label{subsec:triage-simulation-model}  We use simulations to estimate the number of deaths associated with the implementation of the various triage guidelines at different levels of ventilator capacity. Specifically, we bootstrap patient arrivals and ventilator demand from our data and examine different allocation decisions, based on our MDP model.
 The time periods considered consist of a discretization of the time interval (03/01/20 - 05/27/20) into intervals of two hours.
At each time period, the following events happen:
\begin{enumerate}
\item We sample (with replacement) the arrivals and ventilator demand from our data set of observed patients trajectories. The number of samples at each period is equal to the number of new intubations observed at this time period in our data.
\item We update the status (intubated, reassessment or discharge) of all patients in the simulation.  Prior to reaching ventilator capacity, ventilators are allocated on a first-come-first served basis. When the ventilator capacity is reached, new patients are triaged using the triage guideline chosen and assigned a priority (low, medium or high). Low priority patients are excluded from ventilator treatments.
A medium priority patient currently intubated may be excluded from ventilator treatment, to intubate a high priority patient.  High priority patients are never extubated. In particular, if all patients currently intubated have high priority,  any remaining patients who need a ventilator will be excluded from ventilation; i.e., no patients currently intubated will be excluded from ventilator treatment.
\item After two days and five days of intubation,  patients on ventilators are reassessed and reassigned priority classes. Patients in low and medium priority classes are excluded from ventilator treatment if a new patient with higher priority requires one.  Patients with low priority class are removed first.
\end{enumerate}

The First-Come-First-Served (FCFS) triage rule is operationalized as follows. When the capacity constraint is reached,  no new patient is assigned a ventilator until one becomes available, regardless of priority classes.  Extubation only occurs when the patient is deceased or can be safely extubated (the timing of which is indicated by extubation in the observed data).

 At discharge, the status of patients who were not impacted by the triage/reassessment guidelines (i.e. their simulated duration of ventilation corresponds to the observed duration in the data) is the same as observed in the data.
For the outcomes of patients excluded from ventilation by the triage guideline, we use the same method as for our MDP model. In particular, we use a single parameter $p \in [0,1]$ to model the chance of mortality of patients excluded from ventilator treatment.
With probability $p$, the discharge status of a patient excluded from ventilator treatment is deceased. Otherwise, with probability $1-p,$ the discharge status of a patient excluded from ventilator treatment is the same as if this patient had obtained a ventilator (i.e., the same as observed in the data).
We acknowledge that the three potential exclusion events (at triage, at reassessments,  or when removed in order to intubate another patient) may require different values of $p$. Additionally, $p$ may vary across patients and is difficult to estimate in practice. However, when $p = 0$, we obtain an optimistic estimate of the survival rate, since being excluded from ventilator treatment has no impact on the outcome of the patient.  When $p = 1$, we obtain a pessimistic estimate of the survival rate, as in this case any patient excluded from ventilator treatment will die. Therefore, using a single parameter $p \in [0,1]$ enables us to interpolate between an optimistic ($p=0$) and a pessimistic ($p=1$) estimation of the survival rates associated with the triage guidelines. In practice, our clinical collaborators estimate that any $p \geq 0.90$ is a reasonable value. 
\tb{We present our results for $p=0.99$ in this section, and we refer to Remark \ref{rmk:impact parameter p} for some discussion on this parameter.}
\section{Empirical results}\label{sec:simu}
We now evaluate the performance of the tree policies computed \tb{using Algorithm \ref{alg:optimize then fit} and Algorithm \ref{alg:optimize and fit}. We compare to three benchmarks: the NYS Triage guidelines, a First-Come-First-Serve (FCFS) guidelines, and the optimal unconstrained MDP policies obtained in our models.} Recall that for all policies -- including ours -- ventilators are allocated according to FCFS until the ventilator capacity is reached. When there is insufficient ventilator supply to meet all of the demand, they will be allocated according to the specified priority scheme, as described in Section \ref{subsec:triage-simulation-model}.
\tb{We compare the average patient survival for different guidelines, the impact of the reassessment times and of the information used in the state space for the guidelines, and the number of excluded patients. We perform some sensitivity analysis in Appendix \ref{app:sensitivity-analysis-parameter-p} and Appendix \ref{app:sensitivity-analysis} and we compare with other machine learning-based methods for learning allocation guidelines in Appendix \ref{app:comparison ISP-LU}. We present a summary of our experiment pipeline in Figure \ref{fig:experiments pipeline}.}
\begin{figure}[htb]
\center
\includegraphics[width=0.9\linewidth]{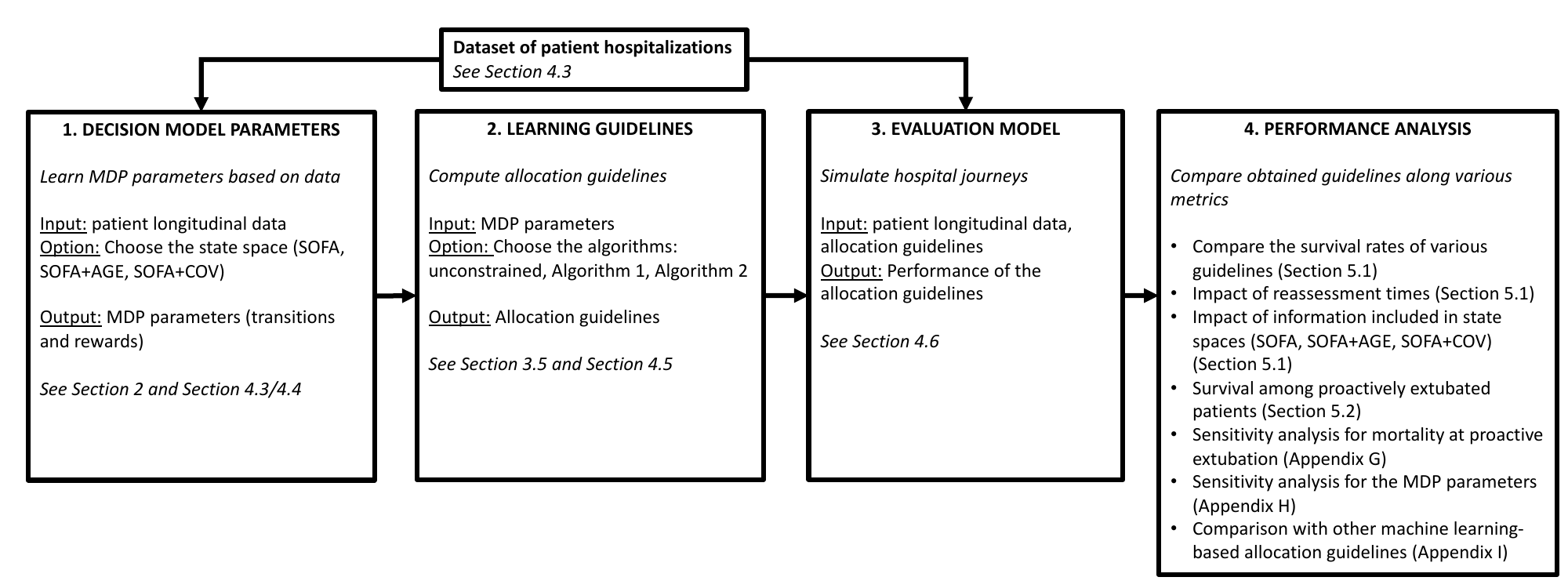}
\caption{The experiments pipeline for our numerical results.}
\label{fig:experiments pipeline}
\end{figure}
\subsection{Number of deaths}
We obtain estimates of the number of deaths, associated with a ventilator capacity and triage guidelines over $100$ bootstrapped data sets.
\tb{In Figure \ref{fig:number-of-deaths}, we compare the number of  deaths associated with various levels of ventilator capacity for different triage guidelines: our tree policies based only on SOFA,  based on SOFA and age, and based on SOFA and other covariates, computed with Algorithm \ref{alg:optimize then fit} (Figure \ref{fig:tree policies optimize then fit}) and Algorithm \ref{alg:optimize and fit} (Figure \ref{fig:tree policies optimize and fit}).
}
For comparison purposes, we also include the performance of the NYS triage algorithm and FCFS, and we display the performance of the optimal unconstrained policy for our MDP model (which we call \textit{MDP policy}) in Figure \ref{fig:MDP policies}. 
Recall that $p$ models the likelihood of death of a patient after exclusion from ventilator treatment. For brevity, we only present our estimations with $p=0.99$ in the main body of the paper. 
\paragraph{Comparing NYS and FCFS.}
We recall that the total number of deaths in our data set, i.e., with ample ventilator capacity, is 543 patients among our cohort of 807 patients (see Table \ref{tab:summary-statistics}). In our numerical experiments, we observe similar number of deaths between the NYS and FCFS policies (see Figure \ref{fig:number-of-deaths}).  For instance,  consider the number of deaths at a ventilator capacity of 180 ventilators. The average number of
deaths is 582.0 (CI: [581.2,582.9]) for the NYS guidelines and 581.7 (CI: [580.5,582.9]) for the FCFS guidelines. 
In Figure \ref{fig:number-of-deaths}, we observe this very small difference between the two policies for various levels of ventilator capacity between $250$ ventilators and $180$ ventilators, \tb{and this also holds for other models for the parameter $p \in [0,1]$ (see Remark \ref{rmk:impact parameter p} and Appendix \ref{app:sensitivity-analysis-parameter-p}).}
It is somewhat surprising to see the NYS and FCFS guidelines performing similarly in this cohort of COVID-19 acute respiratory failure.
The NYS guidelines were designed prior to the COVID-19 pandemic,  in part because of concerns that arose following the shortages experienced in New Orleans after Hurricane Katrina. Consequently, they were designed primarily for the case of ventilator shortages caused by \textit{disasters}, such as hurricanes, a Nor'easter, or mass casualty event. Therefore,  even though the NYS guidelines is based on the SOFA score, it ignores the specifics of respiratory distress caused by COVID-19.  For instance, this may indicate that COVID-19 natural history does not follow the 2 and 5 days reassessment time line, with the average SOFA score at $t=$2 days and at $t=$5 days being somewhat similar in our patient cohort.
The timing of reassessment in the NYS guidelines may need to be re-examined, otherwise the NYS triage algorithm is not able to substantially outperform FCFS.
\begin{figure}[htb]
\center
 \begin{subfigure}{0.3\textwidth}
\centering
         \includegraphics[width=1.0\linewidth]{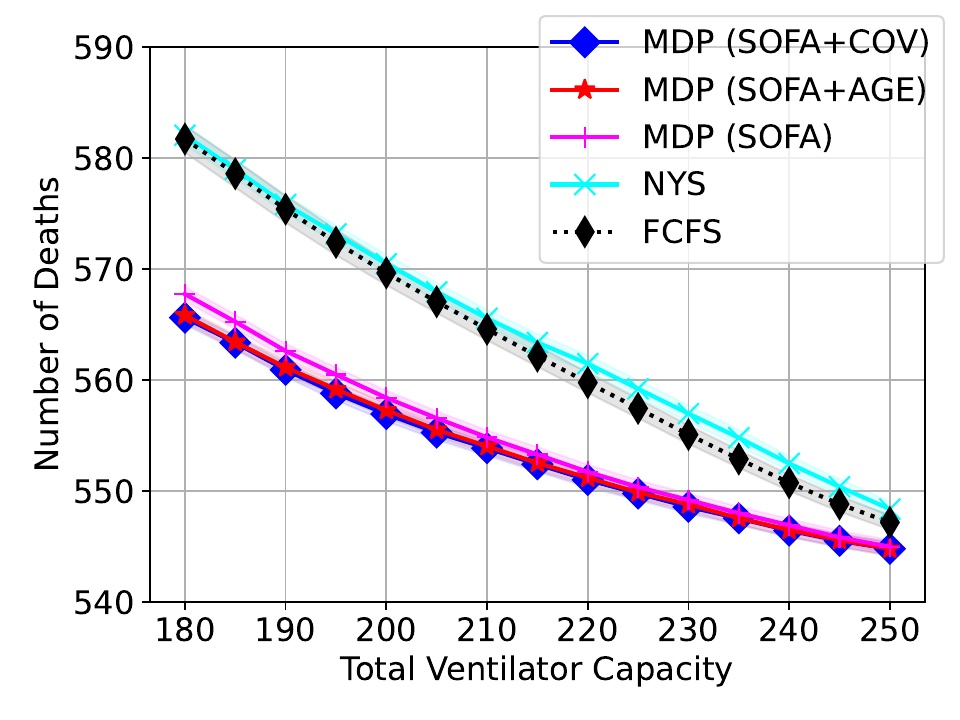}
                  \caption{MDP policies.}
         \label{fig:MDP policies}
  \end{subfigure}
   \begin{subfigure}{0.3\textwidth}
\centering
         \includegraphics[width=1.0\linewidth]{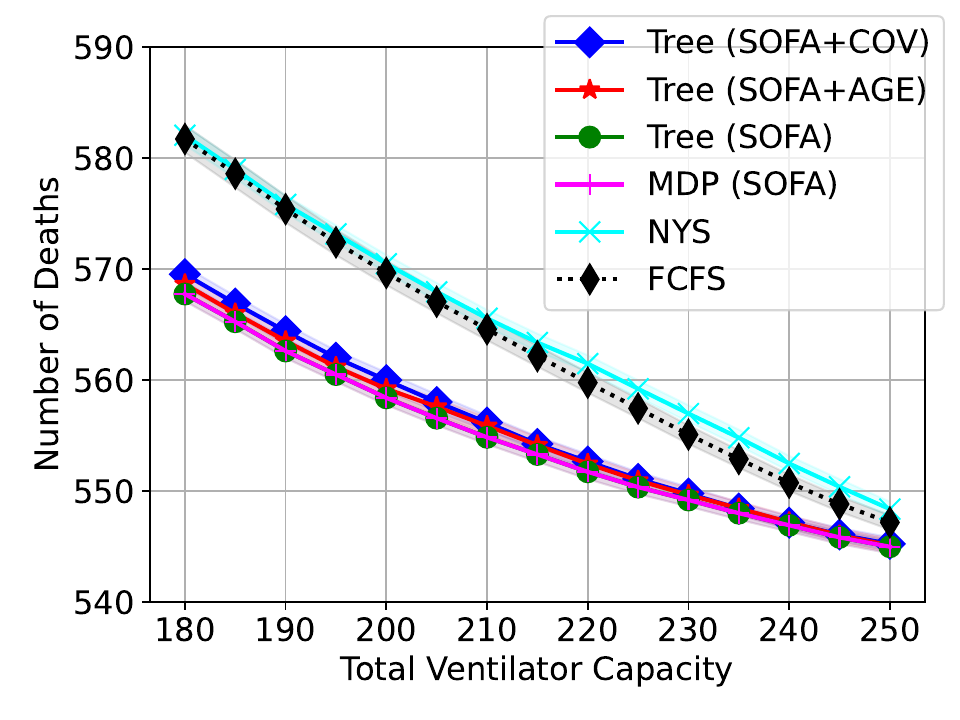}
               \caption{Tree policies from Alg. \ref{alg:optimize then fit}.}
         \label{fig:tree policies optimize then fit}
  \end{subfigure}
\begin{subfigure}{0.3\textwidth}
\centering
         \includegraphics[width=1.0\linewidth]{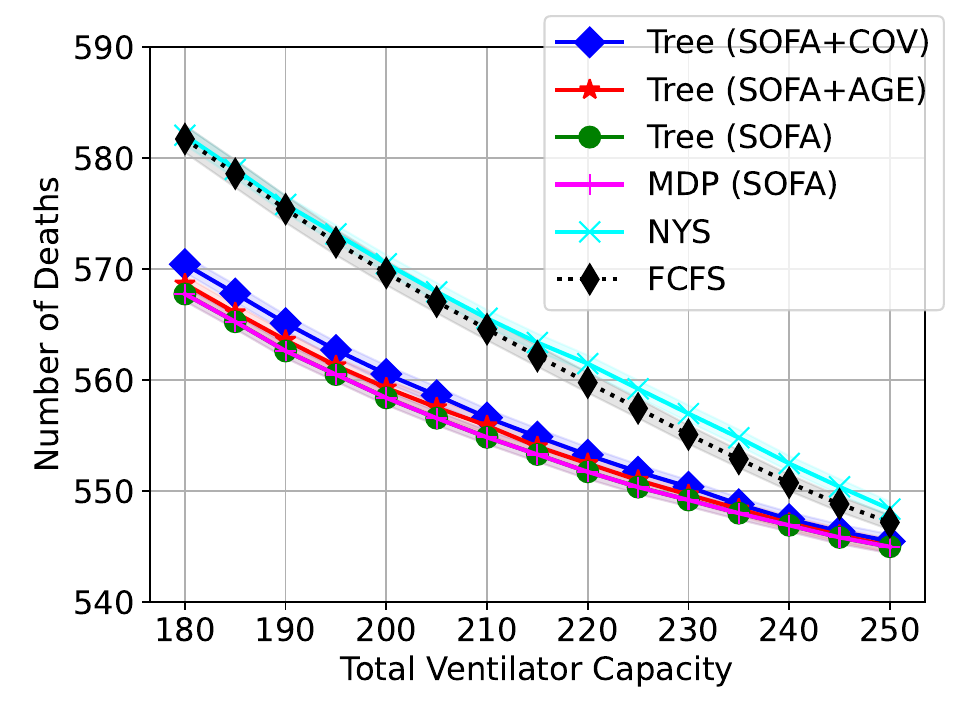}
               \caption{Tree policies from Alg. \ref{alg:optimize and fit}.}
         \label{fig:tree policies optimize and fit}
  \end{subfigure}
         \caption{Number of deaths for various triage guidelines at hypothetical levels of ventilator capacities, for $p=0.99$. }
         \label{fig:number-of-deaths}
\end{figure}
\tb{
\paragraph{Performance of MDP policies} We now analyze the performance of the MDP policies as in Figure \ref{fig:MDP policies}. We note that all three policies outperform the NYS guidelines: Still focusing on a ventilator capacity of 180 ventilators, there is a slight advantage for the MDP policy based on SOFA and other covariates (565.6 deaths, CI: [565.0,566.2]) and the MDP policy based on SOFA and age (565.8 deaths, CI: [565.2,566.5]) over the MDP policy based solely on SOFA (567.7 deaths, CI: [567.1,568.4]).
}

\paragraph{Performance of the interpretable policies.}
\tb{We now focus on the performance of the tree policies computed by Algorithm \ref{alg:optimize then fit} and Algorithm \ref{alg:optimize and fit}. When only using SOFA, both algorithms return the same tree policy, represented in Figure \ref{fig:tree-sofa}. 
We note that when only using information on SOFA, the performance of the tree policy represented in Figure \ref{fig:tree-sofa} matches the performance of the MDP policy.  In fact, the tree policy based on SOFA computed by Algorithm \ref{alg:optimize then fit} and Algorithm \ref{alg:optimize and fit} only differs from the MDP policy at one state: the MDP policy would maintain ventilation at $t=$ 5 days for a patient with a SOFA score of $17$, increasing compared to the last reassessment time, whereas the tree policies would proactively extubate such a patient. We provide a more detailed comparison in Appendix \ref{app:tree-policies}.
The difference with the average number of deaths for the NYS guidelines amounts to $2.6 \%$ of the 543 observed deaths in our data. Overall, based on SOFA only, our algorithms can return tree policies that have comparable performance with the optimal unconstrained allocation decisions (MDP policy), but that are much more interpretable. The tree policy based on SOFA (Figure \ref{fig:tree-sofa}) is a simple policy, that does not exclude any patient at triage (recall that the SOFA scores are always smaller than $18$ in our data set), and it only excludes patients with SOFA larger or equal than $16$ at reassessment when $t=$ 2 days. At reassessment when $t=$ 5 days, two groups of patients are proactively extubated: the first group of patients has SOFA larger or equal than $14$, while the second group of patients has SOFA $\in \{11,12,13\}$ and SOFA is decreasing. Thus, it is less aggressive than the NYS guidelines, which may exclude some patients with SOFA $> 7$ at 2 days and 5 days of intubation if their conditions are not improving (see Appendix \ref{app:details-NYS} for more details). This suggests that the NYS guidelines may be too proactive at excluding patients from ventilator treatments. 
\begin{figure}[htb]
\center
 \begin{subfigure}{0.3\textwidth}
\centering
         \includegraphics[width=0.9\linewidth]{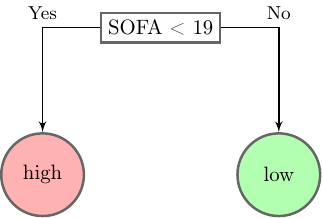}
         \caption{Triage.}
         \label{fig:tree-sofa-time-1}
  \end{subfigure}
   \begin{subfigure}{0.3\textwidth}
\centering
         \includegraphics[width=0.9\linewidth]{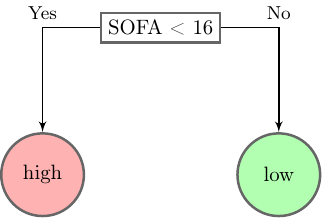}
         \caption{Reassessment after 2 days.}
         \label{fig:tree-sofa-time-48}
  \end{subfigure}
     \begin{subfigure}{0.3\textwidth}
\centering
         \includegraphics[width=0.9\linewidth]{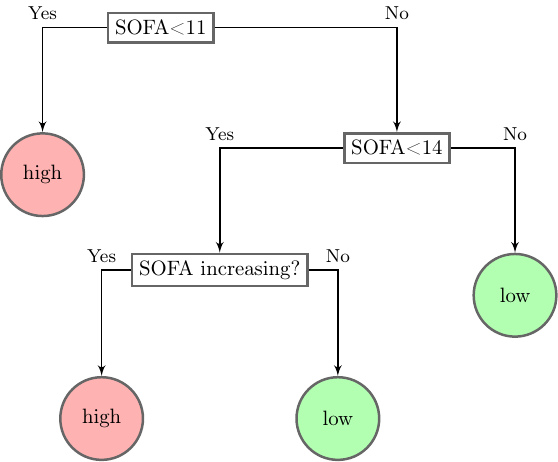}
         \caption{Reassessment after 5 days.}
         \label{fig:tree-sofa-time-120}
  \end{subfigure}
  \caption{Tree policies returned by Algorithm \ref{alg:optimize then fit} and Algorithm \ref{alg:optimize and fit}, only  based on SOFA scores.}\label{fig:tree-sofa}
\end{figure}
\tb{
We now focus on the tree policies obtained using Algorithm \ref{alg:optimize then fit} and Algorithm \ref{alg:optimize and fit} when the state space includes SOFA+AGE or SOFA+COVARIATES. For the sake of conciseness, we represent these trees in Appendix \ref{app:tree-policies}. We observe a gap between the performance of these interpretable policies compared to the MDP policies based on SOFA+AGE and SOFA+COVARIATES. 
However, including other covariates (such as demographics and comorbidities, see \textit{Tree policy (SOFA+COV)} or \textit{Tree policy (SOFA+AGE)} in Figure \ref{fig:number-of-deaths}) leads to a slightly higher number of deaths than policies only based on SOFA score.
This is coherent with the fact that the SOFA score itself has been shown to be a strong predictor of COVID-19 mortality~\citep{zhou2020clinical}. 
Therefore, the models that incorporate comorbidities and/or demographics do not necessarily outperform the policies only based on SOFA; indeed, in this example, they do not.
Categorical exclusion of some patients is considered unethical~\citep{white2020framework}, and we show that it is not necessarily associated with gains in terms of patients saved.

Overall, we note that there does not appear to be a significant performance loss between the optimal unconstrained policies and the interpretable policies returned by our algorithms, despite the optimal unconstrained policies being difficult to interpret in the first place. We also note that Algorithm \ref{alg:optimize then fit} and Algorithm \ref{alg:optimize and fit} return tree policies that are different and may even branch on different covariates, but whose performances are comparable; see Appendix \ref{app:tree-policies} for more details on our tree policies and on their comparisons with the MDP policies. It is plausible that for decision problems with longer horizons, i.e. with more periods of decisions, the performances of the tree policies returned by the two different algorithms may differ more. We leave this as an interesting direction for future research. For the sake of conciseness, we provide a more detailed comparison of the performances of our algorithms in the appendix (see Figure \ref{fig:number-of-deaths-per-alg} in Appendix \ref{app:comparison algo 1/2}).
}
}

\tb{
\paragraph{Changing the reassessment times.}
The reassessment times considered in this section (2 and 5 days after intubation) correspond to the NYS guidelines. However, given the evolution of the disease among SaRS-CoV-2 patients, it is natural to investigate if longer or shorter reassessment times could improve the overall survival rate in the patient population. We study this question numerically, by running the same numerical experiments as for Figure \ref{fig:number-of-deaths}, but where we learn new MDP policies and tree policies for each of the reassessment times (in days)  in $\{(2,3),(2,5),(2,7),(5,7),(5,9),(5,11)\}$. We then evaluate the performance of the new policies using our simulation model from Section \ref{subsec:triage-simulation-model}. The original framework described in NYS guidelines corresponds to the reassessment times $(2,5)$ days, i.e., a first reassessment after two days of intubation and a second reassessment after five days of intubation. Therefore, the set of possible reassessment times that we investigate here allows for either a shorter interval between reassessments ($(2,3)$), a longer interval between reassessments ($(2,7)$), or later reassessments ($(5,7),(5,9),(5,11)$). For the sake of conciseness, we focus on the number of deaths with $p=0.99$ and a capacity of $180$ ventilators available. We present our results in Figure \ref{fig:number-of-deaths (Changing Reassessment)}, where we show the average number of deaths (along with the confidence intervals) for each algorithm (unconstrained optimal policies, Algorithm \ref{alg:optimize then fit} and Algorithm \ref{alg:optimize and fit}), each state space (SOFA, SOFA+AGE and SOFA+COV), and each reassessment time. We observe that the performance of MDP policies is quite stable across all the reassessment times, with a slight advantage for $(2,7)$, i.e. for a longer period of time between the first and the second reassessment times, compared to the NYS guidelines. The performance of the tree policies may vary more across the different reassessment times, and we note that including more information (SOFA+COV or SOFA+AGE) appears to lead to a larger number of deaths, which is consistent with our findings from Figure \ref{fig:number-of-deaths}.

Overall, the simplicity of our algorithms allows the decision maker to investigate numerically the question of choosing optimal reassessment times. We choose to not display the performances of NYS guidelines for various reassessment times on Figure \ref{fig:number-of-deaths (Changing Reassessment)} for the sake of readability, but its average number of deaths was always larger than 580 for all the reassessment times considered in our experiments, significantly larger than for MDP policies and tree policies, a finding similar to our results from Figure \ref{fig:number-of-deaths}.
\begin{figure}[htb]
\center
 \begin{subfigure}{0.3\textwidth}
\centering
         \includegraphics[width=1.0\linewidth]{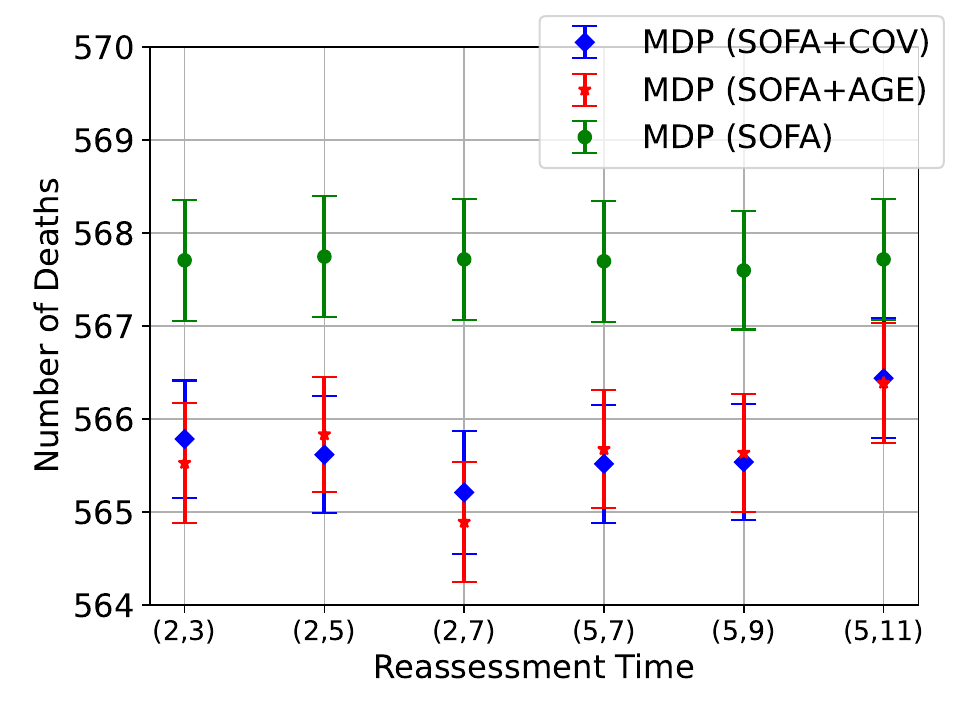}
                  \caption{MDP policies.}
         \label{fig:MDP policies (Changing Reassessment)}
  \end{subfigure}
   \begin{subfigure}{0.3\textwidth}
\centering
         \includegraphics[width=1.0\linewidth]{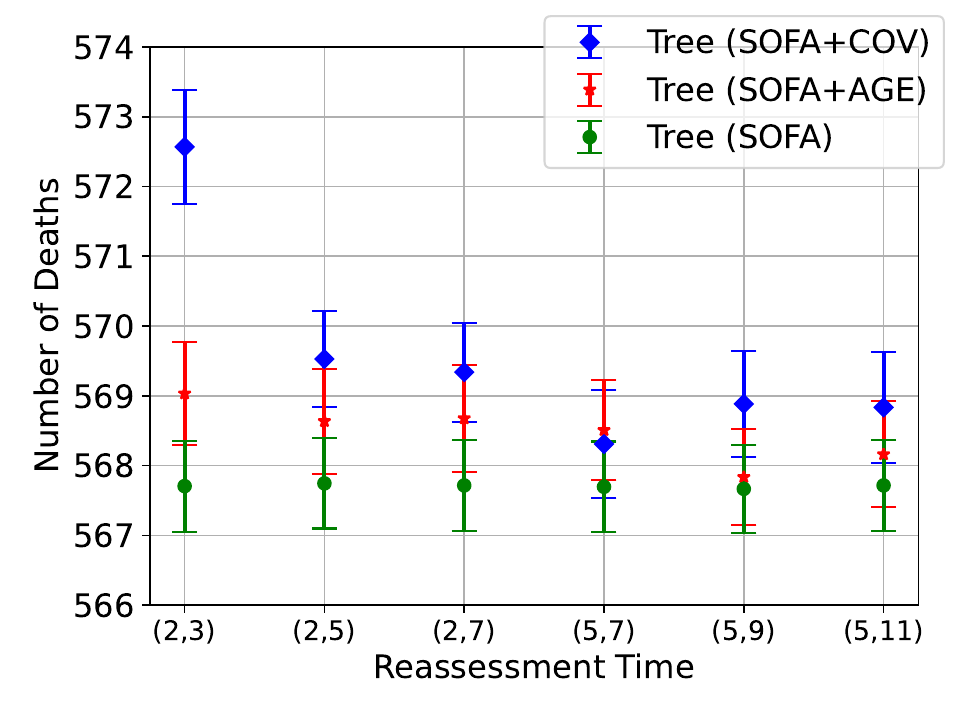}
               \caption{Tree policies from Alg. \ref{alg:optimize then fit}.}
         \label{fig:tree policies optimize then fit (Changing Reassessment)}
  \end{subfigure}
\begin{subfigure}{0.3\textwidth}
\centering
         \includegraphics[width=1.0\linewidth]{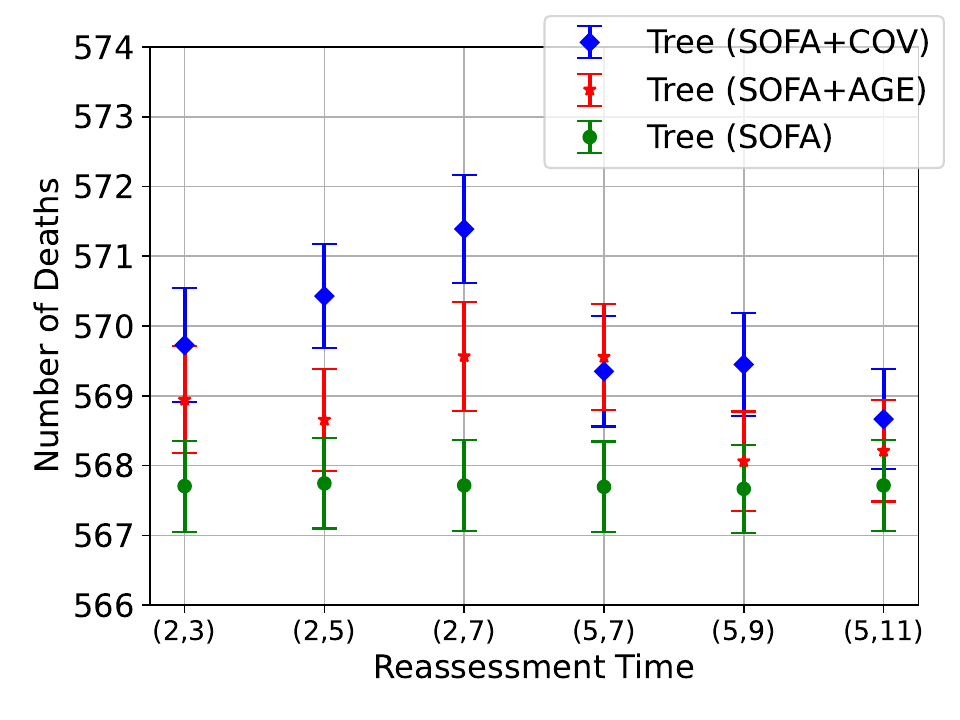}
               \caption{Tree policies from Alg. \ref{alg:optimize and fit}.}
         \label{fig:tree policies optimize and fit (Changing Reassessment)}
  \end{subfigure}
         \caption{Number of deaths for various triage guidelines at ventilator capacity$=180$, for $p=0.99$ and varying reassessment times.}
         \label{fig:number-of-deaths (Changing Reassessment)}
\end{figure}
}

\tb{
\begin{remark}[Impact of parameter $p$.]\label{rmk:impact parameter p}
We have presented our result with $p=0.99$ (probability of death of proactively extubated patients) in Figure \ref{fig:number-of-deaths}. In practice, it is plausible that the probability of death of proactively extubated patients depend on their SOFA score at the period of extubation, as well as on demographics and vital signs. We provide our sensitivity analysis as regards the parameter $p$ in Appendix \ref{app:sensitivity-analysis-parameter-p}, where we consider more elaborate models where the parameter $p$ increases when the SOFA score at the period of proactive extubation is large. All these more elaborate models for the parameter $p$ gives comparable performances for the policies studied in this paper (NYS, FCFS, MDP and tree policies based on Algorithm \ref{alg:optimize then fit} and Algorithm \ref{alg:optimize and fit}), as evident from Figures \ref{fig:number-of-deaths (Sensitive p0)}-\ref{fig:number-of-deaths (Sensitive p3)} in Appendix \ref{app:sensitivity-analysis-parameter-p}. This reinforces the robustness of our conclusions in this section.
\end{remark}
}

\tb{
\begin{remark}[Comparison with other algorithms.]
For the sake of completeness, we also perform additional comparisons with the ISP-LU method for computing ventilator allocation guidelines~\citep{anderson2023rationing}. This algorithm returns a threshold-based triage rule (without any reassessment), based on the predicted probability of death $\hat{P}$ (in case of intubation) and the predicted duration of intubation (length-of-use $\hat{L}$) for each patient requiring intubation. Only patients with a sufficiently large ratio $\hat{P}/\hat{L}$ are allocated a ventilator. \cite{anderson2023rationing} recommend predicting $\hat{P}$ with logistic regression and predicting $\hat{L}$ with a two-stage hurdle model combined with LASSO. We find that ISP-LU guidelines may perform on par with the MDP policies and tree policies, highlighting that interpretable policies based on decision trees may perform on par with other machine learning-based methods, with the added benefit of being interpretable. We provide our detailed comparison with ISP-LU methods in Appendix \ref{app:comparison ISP-LU}.
\end{remark}
}
\subsection{Observed survival rates among excluded patients}\label{sec:simu:observed survival rates}
 For each level of ventilator capacity and each triage guideline,  using our simulation model we can compute the list of patients that would have been excluded from ventilator treatments.  In an ideal situation,  triage guidelines would exclude from ventilator treatment patients who would not benefit from it, i.e.,  ideal triage guidelines would exclude the patients that would die \textit{even if intubated}.  Note that using our data set we are able to know if a patient would have survived in case of intubation (since all patients were intubated in our data set and we can see status at discharge). Therefore, we can compute the survival rates (in case of intubation) among excluded patients.
 Intuitively,  if the guidelines were perfect,  the survival rate (in case of intubation) among excluded patients would be close to 0\%.   Note that random exclusion would achieve an average survival rate (in case of intubation) among excluded patients of 32.7\%, the average survival rate in our cohort of patients.
\begin{figure}[htb]
\center                     
\includegraphics[width=0.9\linewidth]{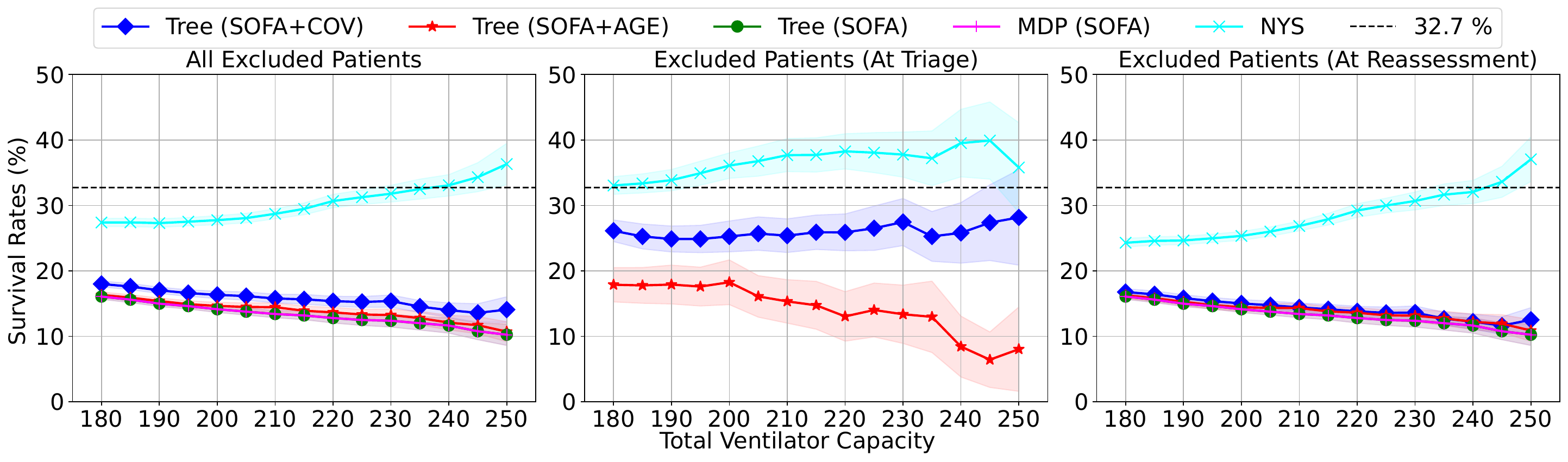}
  \caption{Observed survival rates (in the case of intubation) among all excluded patients for different policies. Note: For total ventilator capacity from $180$ to $250$, Tree (SOFA) and MDP (SOFA) have no patients extubated \textit{at triage}.}
  \label{fig:true_survival_exclusion}
\end{figure}

\begin{figure}[htb]
\center                     
\includegraphics[width=0.9\linewidth]{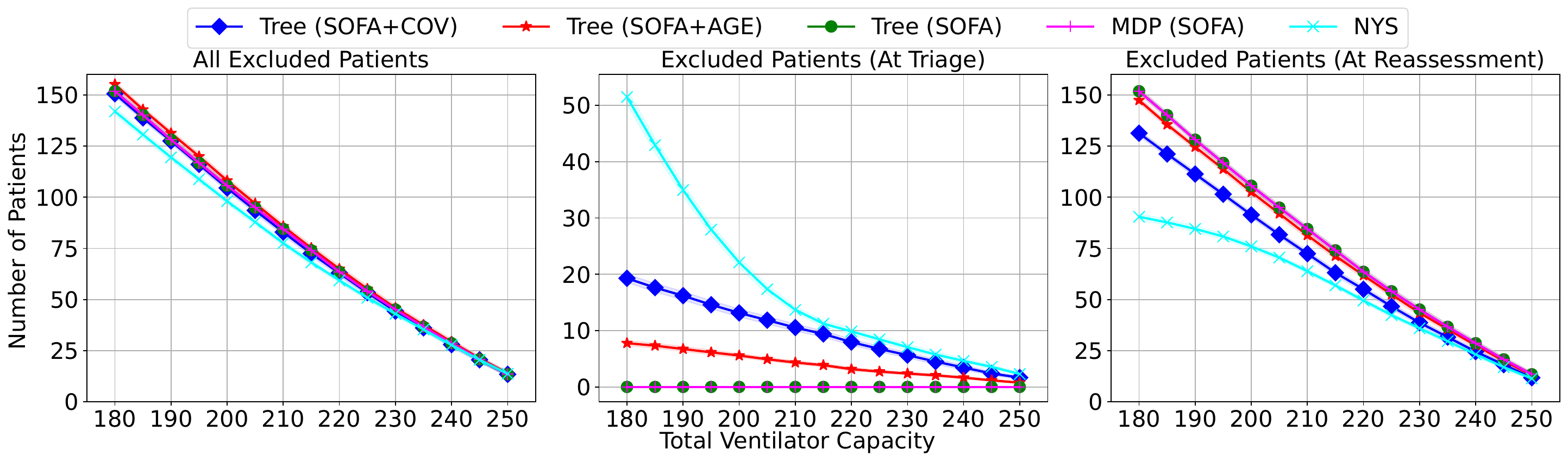}
  \caption{Number of excluded patients for different policies.}
  \label{fig:number_exclusion}
\end{figure}
We compare the survival rate (in case of intubation) among excluded patients for various guidelines in Figure \ref{fig:true_survival_exclusion}. \tb{For the sake of conciseness we only focus on NYS guidelines, the tree policies computed by Algorithm \ref{alg:optimize and fit}, and the MDP policy only based on SOFA.}
We note that among the patients excluded by the NYS guidelines \textit{at triage}, the survival rate is always above 32.7 \%, which means that \textit{at triage}, the NYS guidelines are misidentifying patients who would not benefit from ventilator treatment. \textit{At reassessment}, the NYS guidelines achieve lower survival rates (in case of intubation) among patients excluded, i.e. the NYS guidelines perform better at identifying patients who would not benefit from intubation at reassessment than at triage. \tb{In contrast, our novel triage guidelines based on SOFA, SOFA and age, and SOFA, demographics and comorbidities, are identifying for exclusion those patients with lower survival rates (when intubated).
We emphasize that the tree policy based on SOFA, as well as the MDP policy based on SOFA, does not exclude any patient \textit{at triage}. Additionally, among all policies, the survival rates (in case of intubation) among patients excluded \textit{at triage} is higher than the observed survival rates among patients excluded \textit{at reassessment}.
This stark discrepancy between performances \textit{at triage} and \textit{at reassessment} suggests that it is significantly harder to foreshadow the future evolution of patients' condition and status at discharge at triage than at reassessment, after the patient has remained intubated for several days. We also note that even though all the guidelines considered here exclude approximately the same number of patients (see left plot in Figure \ref{fig:number_exclusion}), it appears that NYS guidelines exclude much more patients at triage than the other policies, even though as we discuss above it is more difficult to assess the benefits of intubation at triage than at reassessment.
All the guidelines considered better identify patients who would not benefit from ventilator treatment, {\it after the patients have been intubated for at least 2 days}. 
This may suggest that more proactive decisions should be made about withdrawing support from patients who are not responding to intensive care treatment compared to patient exclusion at triage.}
Although ethicists generally consider the decision to proactively extubate a patient at reassessment morally equivalent to declining to offer intubation at triage, this action may cause more moral distress to clinicians carrying out the extubation~(see Chapter 5 in \citep{beauchamp2001principles}). 
\subsection{Discussion on our empirical results}\label{sec:simu:discussion}
\paragraph{Advantages and disadvantages of SOFA-based guidelines.} SOFA-based guidelines have multiple advantages. First,  the SOFA score has been shown to correlate with COVID mortality~\citep{zhou2020clinical}. They are simple to implement, as they rely on a single score, and can be deployed quickly in a number of different clinical and disaster scenarios before specifics of a new disease are known.  This is why, of the 26 states that have defined triage guidelines in the United States,
15 use the SOFA score to triage patients~\citep{babar2006direct}.
\tb{
 In terms of performance, our numerical experiments from Figure \ref{fig:number-of-deaths} show that SOFA policies may achieve a lower number of deaths than the official NYS guidelines and FCFS guidelines. This suggests that NYS guidelines, also solely relying on SOFA scores, need to be adjusted to the current disaster before being successfully implemented. As highlighted in Figure \ref{fig:number-of-deaths (Changing Reassessment)}, the reassessment times may also be slightly adjusted to a longer period of time between the two reassessment periods.
 It may be possible to improve the performance of triage policies when even more disease-specific data become available, however, this will not solve the problem of how to manage a novel disaster in the future.
}
Intuitively, incorporating more covariates in the decision model may help better inform triage guidelines. In contrast, our data show that using other clinical information (such as demographics and comorbidities) does not necessarily provide an improvement in the number of lives saved.
Additionally,  including comorbidity conditions may further disadvantage those who face structural inequities, since comorbidities such as diabetes and obesity are closely linked to social determinants of health~\citep{cockerham2017social}.  This detrimental impact may erode trust in medical institutions at large~\citep{auriemma2020eliminating}, which in turn may frustrate other critical public health interventions such as vaccination~\citep{bunch2021tale}.
Therefore, it is critical to counterbalance the utilitarian aims of saving more lives with fairness and the duty to patient care.  
This study provides some data to inform the process of choosing to implement (or not) triage guidelines: official guidelines (such as NYS guidelines) need to be re-adjusted to the specifics of COVID-19 patients before being implemented. Otherwise,  they may show little to no improvement compared to FCFS guidelines, and there may not be any ethical justification for unilaterally removing a patient from a ventilator and violating the duty to care for that patient.

\paragraph{Limitations.}
The strength of our analysis is based on real world data during a massive surge in our facility where ventilator capacity reached fullness. There are several limitations to this study. First, the results cannot be applied to other disease states, such as novel viruses that may arise in the future.  The model needs to be re-trained with new data for each specific patient population. Still, our numerical experiments rely on a relatively small number of patients, highlighting that our models may be used relatively early in the course of the disaster. Second, the observations occurred during the height of the pandemic in New York City when little was known about the specific management of COVID-19. Results may be different with better therapeutic options for the disease. However, this is also a strength of the study given that it matches the scenarios in which triage guidelines are meant to be deployed. Third, the results could be different under different surge conditions, e.g. if the rise in number of cases was sharper or more prolonged. Finally, the simulation cannot mimic real-world work flows that might have required some alterations of the movement of ventilators between patients. 

\tb{
From a modeling standpoint, our algorithms may not return an \textit{optimal} tree policy, since they only return Markovian tree policies. Our numerical experiments show little change in performance between the optimal {\em unconstrained} policies and the interpretable policies returned by our algorithms, but it is difficult to generalize this finding to other potential applications. We also note that we were able to compute {\em optimal} classification trees in the subroutine $\fit$ of our algorithms, using the \gosdt{} Python package. We observed empirically that using a heuristic such as \cart, which may be necessary for larger instances, may lead to a slight deterioration in performance.} Additionally, we use a \textit{nominal} MDP model and do not attempt to compute a \textit{robust} MDP policy~\citep{Iyengar,Kuhn,goyal2023robust}. The reason behind this is that we have a fairly small population of patients, so the confidence intervals in the estimation of our transition rates may be quite large, leading to overly conservative policies.  To mitigate this, we estimate the performances of a policy using our simulation model described in Section \ref{subsec:triage-simulation-model}, and not simply the return in the MDP model, which, of course, also depends on our parameter choices. Therefore, even though the triage guidelines computed with our algorithms depend on the (possibly miss-estimated) transition rates and our choices of rewards, the estimation of their performances is not, and is entirely data- and expert-driven, relying solely on our data set of 807 patients hospitalizations and our collaborations with practitioners at Montefiore.
\section{Conclusion}
In this work, we study the problem of computing interpretable resource allocation guidelines, focusing on the triage of ventilators for COVID-19 patients. We present an algorithm that computes a \textit{tree policy} for finite-horizon MDPs by interweaving algorithms for computing classification trees and algorithms solving MDPs.
Developing bounds on the suboptimality of this tree policy compared to the optimal unconstrained policy is an important future direction.  Additionally, we provide valuable insights on the performances of official triage guidelines for policy makers and practitioners.
We found that the New York State (NYS) guidelines may fail to lower the number of deaths, and performs similarly as the simpler First-Come-First-Served allocation rule.  Our empirical study finds that our interpretable tree policies based only on SOFA may improve upon the NYS guidelines, by being less aggressive in exclusions of the patients. 
Some medical studies have found that SOFA may not be informative for Sars-CoV-2 triage decisions~\citep{wunsch2020comparison}. We show that SOFA is not necessarily useless, but this depends greatly on how the decision maker uses it. \tb{Remarkably, our simulations also show that including much more information in the triage and reassessment recommendations (e.g., demographics and comorbidities) may not have any improvement in the survival among the patient population, on the top of raising important ethical issues.}  Overall,  our simulations of various triage guidelines for distributing ventilators during the COVID-19 pandemic revealed serious limitations in achieving the utilitarian goals these guidelines are designed to fulfill.  Guidelines that were designed \textit{prior} to the pandemic need to be reconsidered and modified accordingly to incorporate the novel data available.  Our work can help revise guidelines to better balance competing moral aims. 
\small
\bibliographystyle{plainnat}
\bibliography{ref}
\newpage
\normalsize
\begin{APPENDICES}
\section{Properties of Proposition \ref{prop:properties-optimal-policy}}\label{app:properties-opt-policy}
In this appendix we show the proof of our Proposition \ref{prop:properties-optimal-policy}, which established the main properties of optimal tree policies.
\proof{Proof of Proposition \ref{prop:properties-optimal-policy}.}
\begin{enumerate}
\item  \tb{\textit{The optimal tree policies for $T$ may depend on the initial distribution $\bm{p}_{1}$.} We consider an instance of the Optimal Tree Policy problem represented in Figure \ref{fig:counter-example-1}, where $H=1,  \X_{1} = \{ s_{1},s_{2} \}, \A_{1} = \{ a_{1},a_{2}\}$. The rewards are given by
\[ r_{1,s_{1},a_{1}} = 1,r_{1,s_{1},a_{2}} = 0,r_{1,s_{2},a_{1}} = 0, r_{1,s_{2},a_{2}} = 1. \]
\begin{figure}[h]
\center
\includegraphics[scale=0.5]{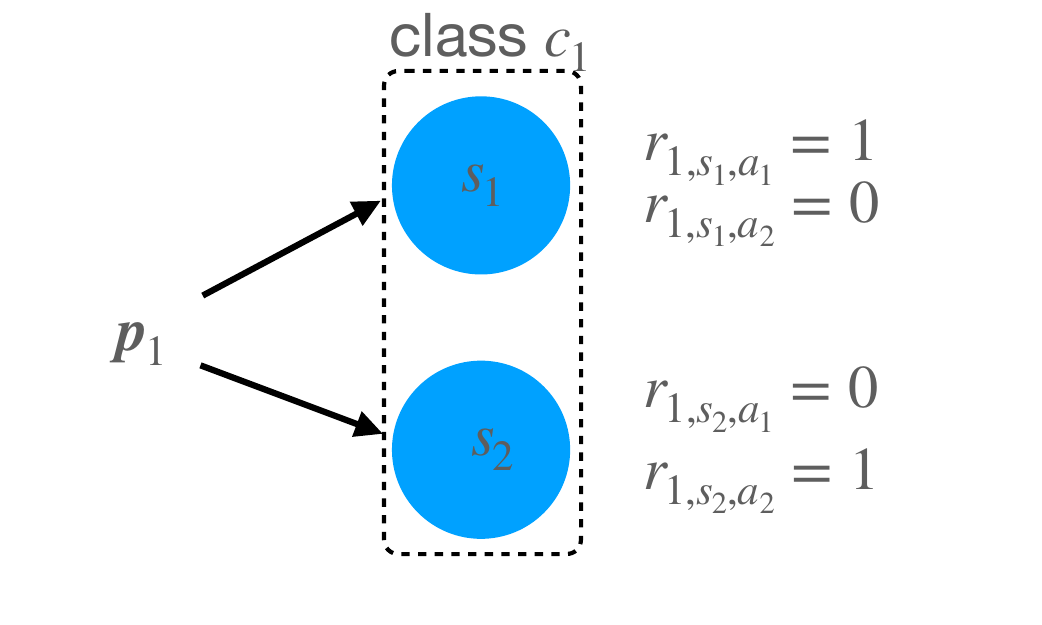}
\caption{Example of an Optimal Tree Policy instance where all optimal policies are dependent on the initial distribution $\bm{p}_{1}$. The same action has to be taken for the states in the class $c_{1}$ defined by the dashed rectangle.}\label{fig:counter-example-1}
\end{figure}
Assume that the set $\TT$ of admissible trees is such that for all tree $T \in \TT$, the states $s_{1},s_{2}$ belong to the same class $c_{1}$. Therefore, the same action ($a_{1}$ or $a_{2}$) has to be taken in both $s_{1}$ and $s_{2}$. Therefore, a tree policy is simply a labeling rule that maps $\{c_{1}\}$ to an action in $\{a_{1},a_{2}\}$. If the initial distribution $\bm{p}_{1}$ over $\{s_{1},s_{2}\}$ is $p_{1,s_{1}}=1, p_{1,s_{2}}=0$, then the optimal action to maximize the return is to choose action $a_{1}$ for the class $c_{1}$. Otherwise, if the initial distribution $\bm{p}_{1}$ is $p_{1,s_{1}}=0, p_{1,s_{2}}=1$, then the optimal action to maximize the return is to choose action $a_{2}$  for the class $c_{1}$. Therefore, the optimal tree policy depends on the initial distribution $\bm{p}_{1}$.}
\item \tb{\textit{The optimal tree policies for $T$ may be history-dependent.} We provide the following simple instance represented in Figure \ref{fig:counter-example-2}, which builds upon the previous example.
We consider an Optimal Tree Policy instance where \[H=2, \X_{1} = \{ s_{1},s'_{1}\}, \A_{1} = \{a_{1}\},\X_{2} = \{ s_{2},s_{3} \}, \A_{2} = \{a_{2},a_{3} \}.\]
\begin{figure}[h]
\center
\includegraphics[scale=0.5]{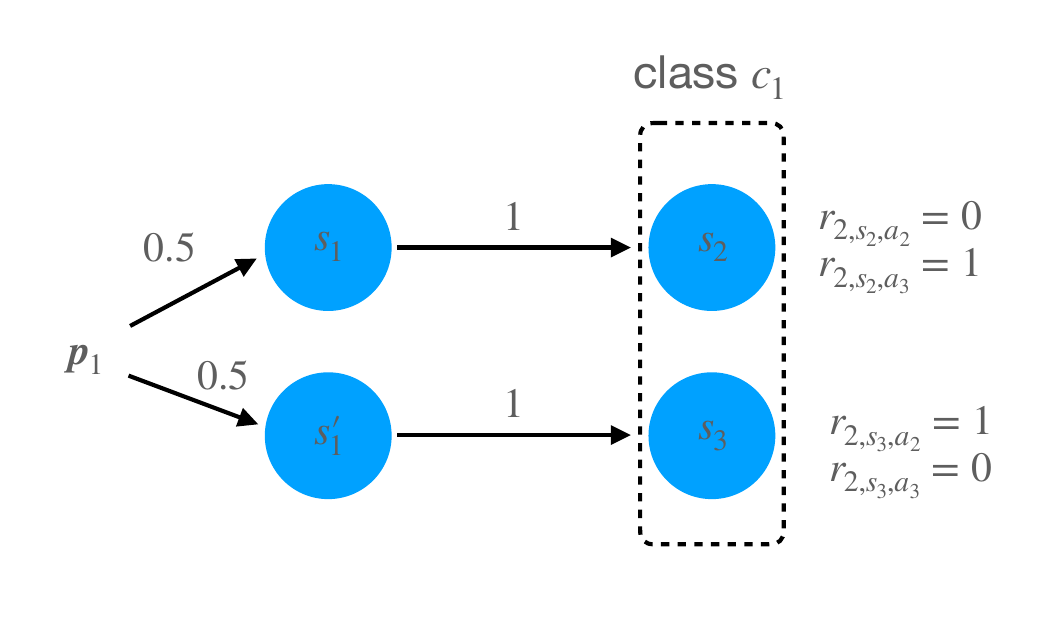}
\caption{Example of an Optimal Tree Policy instance where all optimal policies are history-dependent. The same action has to be taken for the states in the region defined by the dashed rectangle. There are no rewards at $t=1$. The optimal policy at $t=2$ is dependent upon visiting $s_{1}$ or $s'_{1}$ at $t=1$. The transitions probabilities are indicated above the transitions.}\label{fig:counter-example-2}
\end{figure}
There are two states $s_{1}$ and $s'_{1}$ at $t=1$, and $p_{1,s_{1}} = p_{1,s'_{1}} = 0.5$.  There is only one action $a_{1}$ in both $s_{1}$ and $s'_{1}$ and the rewards at $t=1$ are $r_{1,s_{1}a_{1}} = r_{1,s'_{1}a_{1}} = 0$.
There are two states $s_{2},s_{3}$ at period $t=2$ and both belong to the same class $c_{1}$. Therefore, the same action has to be chosen in both states $s_{2}$ and $s_{3}$.
If the decision maker is in $s_{1}$ at period $t=1$, s/he transitions to $s_{2}$ with probability $1$. If the decision maker is in $s'_{1}$ at period $t=1$, s/he transitions to $s_{3}$ with probability $1$.
At period $t=2$, the rewards are given by $r_{2,s_{2},a_{2}} = 0,r_{2,s_{2},a_{3}} = 1, r_{2,s_{3},a_{2}} = 1,r_{2,s_{3},a_{3}} = 0.$
Note that there is only one action $a_{1}$ available at states $s_{1}$ and $s'_{1}$, and that tree policies are a priori history-dependent. Therefore, a tree policy is a map $\{(s_{1},a_{1},c_{1}),(s'_{1},a_{1},c_{1})\} \rightarrow \{a_{2},a_{3}\}$. Additionally, if the history prior to period $t=2$ is $(s_{1},a_{1})$, then the distribution over the states $(s_{2},s_{3})$ visited at $t=2$ is $\left(1,0\right)$. 
Therefore, in this case, the optimal action at $t=2$ is to choose $a_{3}$. However,  if the history prior to period $t=2$ is $(s'_{1},a_{1})$, then the distribution over the states $(s_{2},s_{3})$ visited at $t=2$ is $\left(0,1\right)$, and the optimal action at $t=2$ is to choose $a_{2}$.  Therefore, we see that the optimal decision rule $\pi^{*}_{2}$ for period $t=2$ depends on the history prior to period $t=2$.  The optimal policy, which is history-dependent, achieves a return of $1$. Let us now compute the return of any Markovian (randomized or deterministic) policy. In our simple example, a Markovian policy can be represented by a scalar $\theta \in [0,1]$ representing the probability of choosing action $a_{2}$ in class $c_{1}$. With a small abuse of notation, let us write $R(\theta)$ for the return of a Markovian policy parametrized by $\theta$. With this notation, we have
\begin{align}
R(\theta) & = 0.5 \cdot v_{1,s_{1}} + 0.5 \cdot v_{1,s'_{1}} \label{eq:pf-ex-2-0}  \\
& = 0.5 \cdot v_{2,s_{2}} + 0.5 \cdot v_{2,s_{3}} \label{eq:pf-ex-2-1} \\
& = 0.5 \cdot \left( \theta \cdot (0) + (1-\theta) \cdot 1\right) + 0.5 \cdot \left( \theta \cdot 1 + (1-\theta) \cdot (0) \right) \label{eq:pf-ex-2-2} \\
& =  0.5 \cdot (1-\theta) + 0.5 \cdot \theta \nonumber  \\
& =  0.5 \nonumber
\end{align}
where \eqref{eq:pf-ex-2-0} follows from the definition of the reward and the value functions, \eqref{eq:pf-ex-2-1} follows from the deterministic transition from $s_{1}$ to $s_{2}$ and from $s'_{1}$ to $s_{3}$, and \eqref{eq:pf-ex-2-2} follows from the definition of $\theta$ and the rewards at $t=2$.
We conclude that the return of any Markovian policy is $0.5$, which is strictly smaller than $1$, the return of an optimal history-dependent policy. We therefore conclude that all optimal policies are history-dependent for the example in Figure \ref{fig:counter-example-2}.}
\item \tb{ \textit{There always exists an optimal tree policy for $T$ that is deterministic (even though it may be history-dependent).} Let us consider a sequence of trees $T$ with a randomized optimal tree policy $\pi\opt = \pi\opt_{1}, ..., \pi\opt_{H}$. By definition, there exists a period $t \in [1,H]$ such that $\pi\opt_{t}$ is randomized. We will construct another policy optimal $\pi'$ such that $\pi'_{t}$ is deterministic. 

In particular, we choose $\left(\pi'_{1},...,\pi'_{t-1}\right) = \left( \pi\opt_{1},...,\pi\opt_{t-1} \right)$. This implies that the distribution $\bm{\nu}_{t}$ on the set of states $\X_{t}$ visited at period $t$ are the same for $\pi\opt$ and for $\pi'$, and that the cumulative expected reward $R_{1 \rightarrow t-1}$ from period $t=1$ to period $t=t-1$ are the same.
We also choose $\left( \pi'_{t+1},...,\pi'_{H} \right) = \left( \pi\opt_{t+1},...,\pi\opt_{H} \right)$. This implies that the value functions are the same at $t+1$: $v^{\pi\opt}_{t+1,s} = v^{\pi'}_{t+1,s}, \forall \; s \in \X_{t+1}.$ We write $\bm{v}_{t+1}$ for the vector $\bm{v}_{t+1} = \bm{v}^{\pi'}_{t+1}$.

Now let us choose $\pi'_{t}$ in a deterministic way. In particular, we choose $\pi'_{t}$ for period $t$ as the solution of
\begin{equation}\label{eq:pi-prime-def}
\max_{\pi \in \Delta_{t}} \sum_{c \in [K_{t}]} \sum_{s \in \X_{t}} \nu_{s} \sum_{a \in \A_{t}} 1_{ \{ T_{t}(s) = c \} } \pi_{ca} \left( r_{sa} +  \sum_{s' \in \X_{t+1}} P_{t,sas'} v_{t+1,s'} \right),
\end{equation}
with $\Delta_{t} = \times_{c \in [K_{t}]} \Delta \left(\A_{t}\right)$. Note that we can recover a deterministic decision rule $\pi'_{t}$ because \eqref{eq:pi-prime-def} is a linear program over $\Delta_{t}$.
Adding the constant $ R_{ 1 \rightarrow t-1}$, we see that $\pi'_{t}$ maximizes
\[\pi \mapsto R_{1 \rightarrow t-1} + \sum_{c \in [K_{t}]} \sum_{s \in \X_{t}} \nu_{s} \sum_{a \in \A_{t}} 1_{ \{ T_{t}(s) = c \} } \pi_{ca} \left( r_{sa} + \sum_{s' \in \X_{t+1}} P_{t,sas'} v_{t+1,s'} \right),\]
where $T$ is the tree chosen by $\pi\opt$ for period $t$.
Because $\bm{\nu}_{s}$ is the distribution over the states visited at period $t$ induced by both $\pi'$ and $\pi\opt$,  and because $\bm{v}_{t+1}$ is the value function of both $\pi'$ and $\pi\opt$ after period $t$, we have that
\[ R(\pi') = R_{1 \rightarrow t-1} +  \sum_{c \in [K_{t}]} \sum_{s \in \X_{t}} \nu_{s} \sum_{a \in \A_{t}} 1_{ \{ T_{t}(s) = c \} } \pi'_{t,ca} \left( r_{sa} + \sum_{s' \in \X_{t+1}} P_{t,sas'} v_{t+1,s'} \right).\]
But this means that
\begin{align}
R(\pi') & = R_{ 1 \rightarrow t-1} +  \sum_{c \in [K_{t}]} \sum_{s \in \X_{t}} \nu_{s} \sum_{a \in \A_{t}} 1_{ \{ T_{t}(s) = c \} } \pi'_{t,ca} \left( r_{sa} + \sum_{s' \in \X_{t+1}} P_{t,sas'} v_{t+1,s'} \right) \nonumber \\
& \leq R_{1 \rightarrow t-1} + \sum_{c \in [K_{t}]} \sum_{s \in \X_{t}} \nu_{s} \sum_{a \in \A_{t}} 1_{ \{ T_{t}(s) = c \} } \pi\opt_{t,ca} \left( r_{sa} + \sum_{s' \in \X_{t+1}} P_{t,sas'} v_{t+1,s'} \right) \label{eq:pf-counter-ex-3} \\
& = R(\pi\opt) \label{eq:pf-counter-ex-4}
\end{align}
where \eqref{eq:pf-counter-ex-3} follows from $\pi'_{t}$ as an optimal solution to \eqref{eq:pi-prime-def} and \eqref{eq:pf-counter-ex-4} follows from $R_{1 \rightarrow t-1}$ being the return obtained by $\pi\opt$ between time $1$ and time $t-1$, $\bm{\nu}_{s}$ being the distribution on $\X_{t}$ induced by $\pi\opt$, and $\bm{v}_{t+1}$ being the value function of $\pi\opt$ at time $t+1$.
Therefore, we constructed a deterministic policy $\pi'$ with $R_(\pi') \geq R(\pi\opt)$. Since $\pi\opt$ is an optimal policy, $\pi'$ is also optimal, and $\pi'_{t}$ is deterministic. Therefore, we can always replace any randomized decision rule in $\pi\opt$ with a deterministic decision rule, without decreasing the return. This concludes the proof.}
\end{enumerate}
\hfill \Halmos
\endproof
\tb{
\section{Proof of Section \ref{sec:markovian-policies}}\label{app:proof-markovian}
\proof{Proof of Proposition \ref{prop:markovian-tree-policies-deterministic}.}
Note that the proof of the optimality of deterministic tree policies (last statement of Proposition \ref{prop:properties-optimal-policy}) does not rely on the assumption that the policy $\pi\opt$ is history-dependent. In particular, the exact same proof technics show that if $\pi\opt$ is an optimal Markovian tree policy, and if $\pi\opt_{t}$ is randomized for some $t \in [H]$, we can construct a Markovian tree policy $\pi'$ such that $\pi'_{t}$ is deterministic and $R(\pi')=R(\pi\opt)$. This is because the decision rule $\pi'_{t}$ constructed from \eqref{eq:pi-prime-def} is Markovian and deterministic. \hfill \Halmos
\endproof

\proof{Proof of Proposition \ref{th:equivalence-cart-otp}.}
\begin{enumerate}
\item  \textit{Any instance of \eqref{eq:otp} with $H=1$ can be reduced to an instance of \eqref{eq:cart}.} Let us consider a tree MDP instance $\M=(H,\X_{1},\A_{1},\bm{P}_{1},\bm{r}_{1}, \bm{p}_{1})$ with horizon $H=1$ and $T \in \TT$ a sequence of decision trees in the set of admissible tree $\TT$. Since $H=1$, note that $T = \left( T_{1} \right)$ where $T_{1} \in \T(\X_{1},[K_{1}])$. Let $\pi \in \Pi_{T}$. Following Proposition \ref{prop:properties-optimal-policy}, an optimal tree policy can always be chosen deterministic. Therefore, we will restrict ourselves to deterministic tree policies $\pi: [K_{1}] \rightarrow \A_{1}$. By definition, the return $R(\pi)$ of a deterministic tree policy $\pi$ is
\[ R(\pi) = \sum_{c =1}^{K_{1}} \sum_{s \in \X_{1}} 1_{ \{ T(s) = c \}} \left(p_{1,s}r_{s\pi(c)}\right)\]
with $\left(p_{1s}\right)_{s \in \X_{1}}$ the initial probability distribution over the set of states $\X_{1}$.
This is exactly the expression of the classification error $C_{\sf tree}( \cdot,\cdot )$ as in \eqref{eq:classification-error-randomized}, for an instance of \eqref{eq:cart} built as follow: the number of observation $m$ is $|\X_{1}|$, the set of data points $\{\bm{x}_{1},...,\bm{x}_{m}\}$ is $\X_{1}$, the set of labels $\LL$ is $\A_{1}$, and the misclassification costs $\left(\omega_{i,\ell}\right)_{i,\ell}$ are defined as  $\omega_{i,\ell} = - m \cdot p_{1,\bm{x}_{i}}r_{\bm{x}_{i},\ell}$ for $i=1,...,m$ and $\ell \in \A_{1}$. The policy $\pi$ in the MDP instance $\M$ plays the role of the labeling rule $\mu$ in the instance of \eqref{eq:cart}, and we note that maximizing the return for this instance of \eqref{eq:otp} is equivalent to minimizing the misclassification cost for this instance of \eqref{eq:cart}. Therefore, \eqref{eq:otp} with horizon $H=1$ reduces to \eqref{eq:cart}.
\item \textit{Any instance of \eqref{eq:cart} can be reduced to an instance of  \eqref{eq:otp} with $H=1$.}

Let us consider an instance of \eqref{eq:cart}. In particular, consider a set of data points $\{\bm{x}_{1},...,\bm{x}_{m}\}$  with $\bm{x}_{i} \in \XX \subset \R^{p}$, a finite set of labels $\LL$ and some misclassification costs $\left( \omega_{i,\ell} \right)_{i,\ell}$. We are also given a set of admissible trees $\T\left(\XX,[K]\right)$. The goal is to find a decision tree in $\T\left(\XX,[K]\right)$ and a labeling rule $\mu$ that minimizes the classification error $C_{\sf tree}(T,\mu)$ defined as in \eqref{eq:classification-error-randomized}.

We can reduce this instance of \eqref{eq:cart} to an instance of \eqref{eq:otp}. In particular, we define the MDP instance $\M=(H,\X_{1},\A_{1},\bm{P}_{1},\bm{r}_{1}, \bm{p}_{1})$ as follows:
\begin{enumerate}
\item The horizon is $H=1$.
\item The set of states is the set of possible features: $\X_{1}=\{ \bm{x}_{1},...,\bm{x}_{m}\}$.
\item The set of actions is the set of labels: $\A_{1} = \LL$.
\item The choice of $\bm{P}_{1}$ (transition kernels) does not play any role, because $H=1$.
\item The rewards $\left( r_{1,sa} \right)_{s,a}$  are constructed to reflect the misclassification costs $\left(\omega_{i,\ell} \right)_{i,\ell}$. In particular,  if $s \in \X_{1}$ and $a \in \A_{1}$, by construction $s = x_{i}$ and $a = \ell$ for some $i \in [m]$ and some label $\ell$. The return $r_{sa}$ is chosen as $-\omega_{i,\ell}$.
\item The initial distribution $\bm{p}_{1}$ is uniform across all states.
\end{enumerate}
With this formulation,  for any admissible tree $T \in \T\left(\X,[K]\right)$, a deterministic tree policy $\pi \in \Pi_{T}$ is equivalent to assigning a label in $\LL$ to each point $\bm{x}_{i}$ for every $i \in [m]$,  with the additional constraint that the same action is chosen for all observations belonging to the same class.  The return $R(\pi)$ associated with a policy is the opposite of the classification cost $C_{\sf tree}(T,\pi)$ of the tree $T$.  Therefore,  solving the \eqref{eq:otp} problem (maximizing the return $R(\pi)$) with this MDP instance solves the \eqref{eq:cart} problem (minimizing the classification error).
\end{enumerate}
\hfill \Halmos
\endproof
}
\tb{
\proof{Proof of Proposition \ref{prop:tree-policy-complexity}.}
Our proof is based on a reduction from the minimum vertex cover problem, which is known to be NP-complete (Theorem 7.44,  \cite{sipser1996introduction}). We first formally introduce our notations for graphs, we then define the minimum vertex cover problem ({\sf MVC}), and we finally show how to reduce  ({\sf MVC}) to computing an optimal Markovian tree policy for a specific tree MDP instance.
\paragraph{Graph notations.} A (directed) graph is a pair $(\mcV,\mcE)$ where $\mcV$ is a finite set of vertices, and $\mcE \subseteq \mcV \times \mcV$ is a set of admissible edges. For any edge $e \in \mcE$, we have $e = \left(v_{1},v_{2}\right)$ with $v_{1},v_{2} \in \mcV$, and we write $v^{1}(e) = v_{1}$ and $v^{2}(e) = v_{2}$.
\paragraph{Minimum vertex cover problem.} We are given a graph $(\mcV,\mcE)$. A cover $\mcC$ is a subset of $\mcV$ such that: $\forall \; e \in \mcE, \exists \; v \in \mcC, v \in \{v^{1}(e),v^{2}(e)\}$. 
A cover always exists, since $\mcC=\mcV$ is a cover. A vertex in a cover $\mcC$ is said to be {\em selected}, and an edge $e \in \mcE$ such that there exists $v \in \mcE$ for which $\exists \; v \in \mcC, v \in \{v^{1}(e),v^{2}(e)\}$ is said to be {\em covered}. The  Minimum Vertex Cover problem ({\sf MVC}) is to find a cover with the smallest cardinality.
\paragraph{Reduction of {\sf MVC} to \eqref{eq:otp-markov}.} We now show how to reduce the problem {\sf MVC} to the problem of computing an optimal Markovian tree policy compatible with a given sequence of decision trees $T$. Let $(\mcV,\mcE)$ an instance of {\sf MVC} and consider the MDP instance represented in Figure \ref{fig:MDP-np-hard}. There are $H=3$ periods. The states at each period are represented in Figure \ref{fig:MDP-np-hard-instance}. The set $\X_{1}$ of states at time $t=1$ is $\X_{1}=\mcV \bigcup \{v^{1}(e) \; | \; e \in \mcE\}$. The set $\X_{2}$ of states at time $t=2$ is $\X_{2}=\mcV \bigcup \{v^{2}(e) \; | \; e \in \mcE\} \bigcup \{ \sf selected, covered \}$. The set $\X_{3}$ of states at time $t=3$ is the set of terminal states $\{ \sf selected, not \; selected, covered, not \; covered\}$, where there are no actions to choose. The initial distribution is uniform over the state in $\X_{1}$. There are two actions $\{0,1\}$ to be chosen at time $t=1$ and $t=2$. The transitions for each of these actions are deterministic and are represented in Figure \ref{fig:MDP-np-hard-select} and Figure \ref{fig:MDP-np-hard-no-select}. Intuitively, choosing action $1$ corresponds to selecting a vertex or covering an edge. There is no action to  be chosen at the terminal states at $t=3$. The instantaneous rewards are defined as follows:
\begin{itemize}
\item For any state in $\X_{1}$ or $\X_{2}$, the instantaneous reward is $0$.
\item At $t=3$, there is a reward of $0$ for the state {\sf selected}, a reward of $1/(|\mcV|+1)$ for the state {\sf not selected}, a reward of $1$ for the state {\sf covered} and a reward of $0$ for the state {\sf not covered}.
\end{itemize}
Note that with this definition of the instantaneous reward, the decision maker is penalized for selecting a vertex and for not covering an edge, but rewarded for covering an edge and not selecting a vertex.

We now define the tree constraints that will ensure that covered edges exactly correspond to edges attached to selected vertices.
In particular, we want to enforce that at time $t=1$, the action chosen at a vertex $v \in \mcV$ is the same at all states $v^{1}(e)$ for which $v^{1}(e) = v$. We also want to enforce that at time $t=2$, the action chosen at a vertex $v \in \mcV$ is the same at all states $v^{2}(e)$ for which $v^{2}(e) = v$. Note that our definition of decision tree (Definition \ref{def:decision-tree}) considers univariate splits over components of a vector. Let us map each state $v$ corresponding to a vertex $v \in \mcV$ at time $t=1$ or $t=2$ to a vector $\bm{x}^{v} \in \R^{|\mcV|}$, with $x^{v}_{v}=1$ and $x^{v}_{v'} = 0, \forall \; v' \in \mcV \setminus \{v\}$. We map each state $v^{1}(e)$ corresponding to an edge $e \in \mcE$ at time $t=1$ to a vector $\bm{x}^{1,e}$ defined as $x^{1,e}_{v^{1}(e)}=1$ and $x^{1,e}_{v'} = 0, \forall \; v' \in \mcV \setminus \{v^{2}(e)\}$. 
We define the same embedding at time $t=2$ to map each state $v^{2}(e)$ corresponding to an edge $e \in \mcE$ to a vector $\bm{x}^{2,e}$ defined, with $x^{2,e}_{v^{2}(e)}=1$ and $x^{2,e}_{v'} = 0, \forall \; v' \in \mcV \setminus \{v^{1}(e)\}$. With these notations, let us consider the decision tree $T$ defined as in Figure \ref{fig:np-hard-tree-constraint} for the tree constraints at $t=1$ and $t=2$.

We will show that finding an optimal Markovian policy that is compatible with the sequence of decision trees $(T,T)$ solves {\sf MVC}, i.e., it provides a cover that selects a minimum number of vertices. Let $\pi=(\pi_{1},\pi_{2})$ be such a policy. Up to the multiplicative constant $1/(|\mcV|+|\mcE|)$ (appearing because of the uniform initial distribution), the return of $\pi$ is equal to the sum of the values of the states in $\X_{1} = \mcV \bigcup \{v^{1}(e) \; | \; e \in \mcE\}$. 

Let us first compute the values of the states in $\X_{1}$ corresponding to vertices $ v \in \mcV$. We can write $\mcV = \mcV_{0}(\pi) \bigcup \mcV_{1}(\pi)$, with $\mcV_{0}(\pi)$ the set of states not selected at $t=1$ and $\mcV_{1}(\pi)$ the set of states selected at $t=1$: \[\mcV_{0}(\pi) = \{ v \in \mcV \; | \; \pi_{1}(v) = 0\},\mcV_{1} = \{ v \in \mcV \; | \; \pi_{1}(v) = 1\}.\]
Any state in $\mcV_{1}(\pi)$ terminates in {\sf selected} at $t=3$, ending up with a value of $0$. The states in $\mcV_{0}(\pi)$ can be divided in $\mcV_{0}(\pi) = \mcV_{00}(\pi) \bigcup \mcV_{01}(\pi)$ with $\mcV_{00}(\pi)$ the set of states not selected at $t=1$ and not selected at $t=2$, and $\mcV_{01}(\pi)$ the set of states not selected at $t=1$ but selected at $t=2$: \[\mcV_{00}(\pi) = \{ v \in \mcV_{0}(\pi) \; | \; \pi_{2}(v) = 0\},\mcV_{01}(\pi) = \{ v \in \mcV_{0}(\pi) \; | \; \pi_{2}(v) = 1\}.\] The states in $\mcV_{00}(\pi)$ terminate in {\sf not selected} at $t=3$ and the states in $\mcV_{01}(\pi)$ terminate in {\sf selected} at $t=3$. Therefore, the value of any $v \in \mcV_{01}(\pi)$ is $0$, while the value for any $v \in \mcV_{00}(\pi)$ is $1/(|\mcV|+1)$. 

We now compute the values of the states in $\X_{1}$ corresponding to $\{v^{1}(e) \; | \; e \in \mcE\}$. We use the notations $\mcV_{1}(\pi),\mcV_{0}(\pi),\mcV_{00}(\pi),\mcV_{01}(\pi)$ introduced in the previous paragraph. Let $e \in \mcE$.
If the vertex corresponding to $v^{1}(e) \in \mcV$ is in $\mcV_{1}(\pi)$,  then because of the tree constraints we also have that $\pi_{1}(v^{1}(e)) = 1$: the state $v^{1}(e)$ terminates in {\sf covered}  at $t=3$, yielding a value of $1$ for $v^{1}(e)$. Otherwise, $\pi_{1}(v^{1}(e)) = 0$. Because of the deterministic transitions represented in Figure \ref{fig:MDP-np-hard-no-select}, the value of $v^{1}(e)$ is the value of $v^{2}(e)$. If the vertex corresponding to $v^{2}(e) \in \mcV$ is in $\mcV_{01}(\pi)$, the state $v^{1}(e)$ terminates in {\sf covered} at $t=3$ and the value of $v^{1}(e)$ is $1$. Otherwise, the state $v^{2}(e)$ transitions to {\sf not covered} at $t=3$ and the value of $v^{1}(e)$ is $0$.

Overall, we have shown that the return $R(\pi)$ of a Markovian tree policy $\pi$ is 
\[ R(\pi)=\frac{1}{|\mcE| + |\mcV|} \times \left( \frac{|\mcV_{00}(\pi)|}{|\mcV|+1} + |\{ e \in \mcE \; | \; v^{1}(e) \in \mcV_{1}(\pi) \text{ or } v^{2}(e) \in \mcV_{1}(\pi) \}| \right).\]
Note that the multiplicative factor $1/(|\mcE| + |\mcV|)$ comes from the initial distribution $\bm{p}_{1}$ being uniform over $\X_{1}$.
Let $\mcC(\pi) = \{v \in \mcV \; | \; \pi_{1}(v) = 1 \text{ or } \pi_{2}(v)=1\} \subseteq \mcV$. Then $\mcV_{00}$ corresponds to the set of vertices not selected by $\mcC(\pi)$ and $\{ e \in \mcE \; | \; v^{1}(e) \in \mcV_{1}(\pi) \text{ or } v^{2}(e) \in \mcV_{1}(\pi) \}$ corresponds to the set of edges covered by $\mcC(\pi)$.
Because $|\mcV|/(|\mcV|+1|) < 1$, we see that for any optimal Markovian tree policy $\pi\opt$, $\mcC(\pi\opt)$ must be a cover: $\{ e \in \mcE \; | \; v^{1}(e) \in \mcV_{1}(\pi\opt) \text{ or } v^{2}(e) \in \mcV_{1}(\pi\opt) \} = \mcE$. Additionally, since $\pi\opt$ maximizes $\pi \mapsto R(\pi)$, it must select the smallest number of vertices possible, with $\mcC(\pi\opt)$ remaining a cover. This shows that $\mcC(\pi\opt)$ is an optimal solution to {\sf MVC}. This shows that any instance of {\sf MVC} can be reduced to computing an optimal Markovian tree policies for the tree MDP instance represented in Figure \ref{fig:MDP-np-hard}. This concludes the proof of Proposition \ref{prop:tree-policy-complexity}.
\begin{figure}[htb]
\center
 \begin{subfigure}{0.48\textwidth}
\centering
         \includegraphics[width=\linewidth]{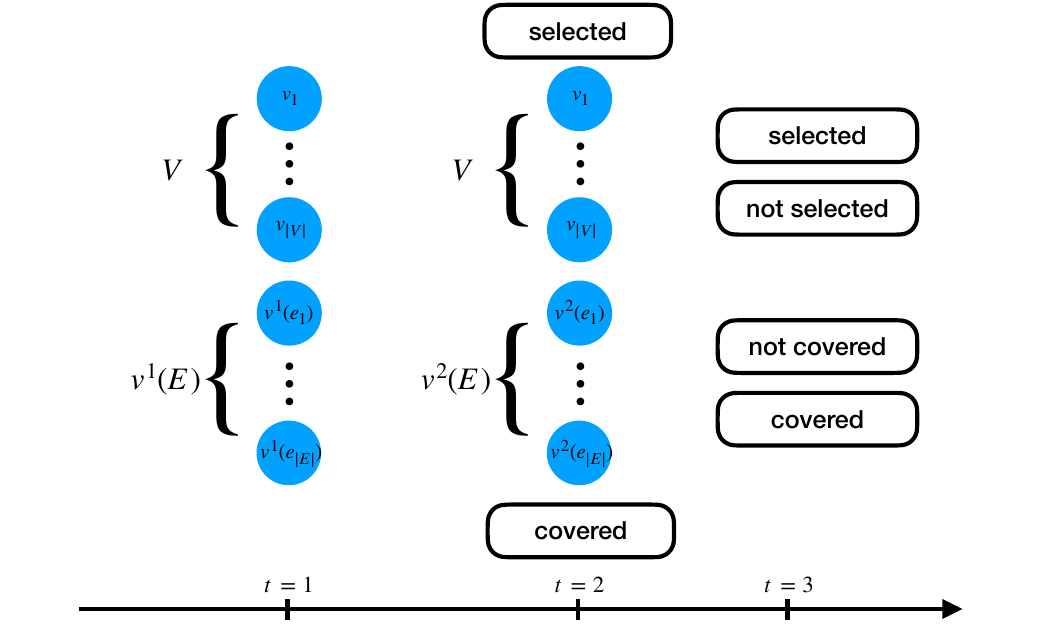}
         \caption{MDP instance.}
         \label{fig:MDP-np-hard-instance}
  \end{subfigure}
   \begin{subfigure}{0.48\textwidth}
\centering
         \includegraphics[width=\linewidth]{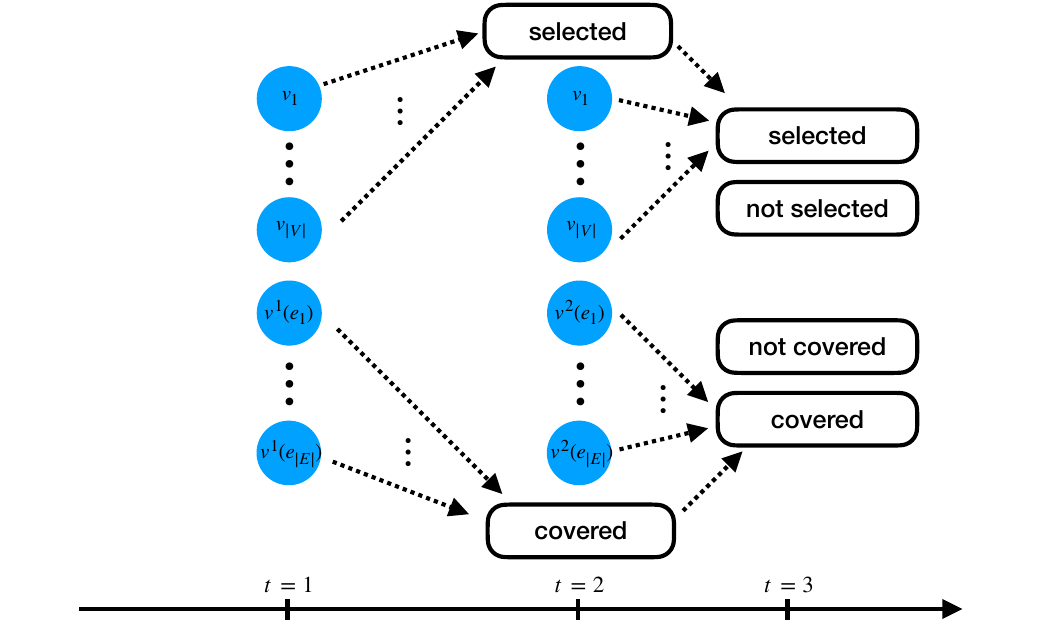}
         \caption{Transition for $a=1$.}
         \label{fig:MDP-np-hard-select}
  \end{subfigure}
     \begin{subfigure}{0.48\textwidth}
\centering
         \includegraphics[width=\linewidth]{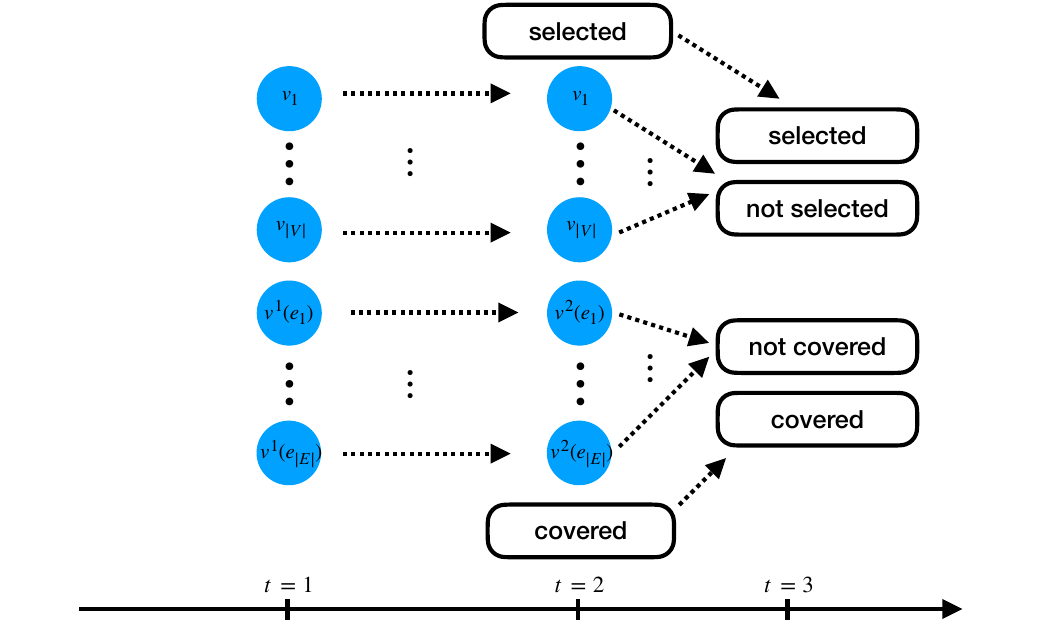}
         \caption{Transition for $a=0$.}
         \label{fig:MDP-np-hard-no-select}
  \end{subfigure}
     \begin{subfigure}{0.48\textwidth}
\centering
         \includegraphics[width=\linewidth]{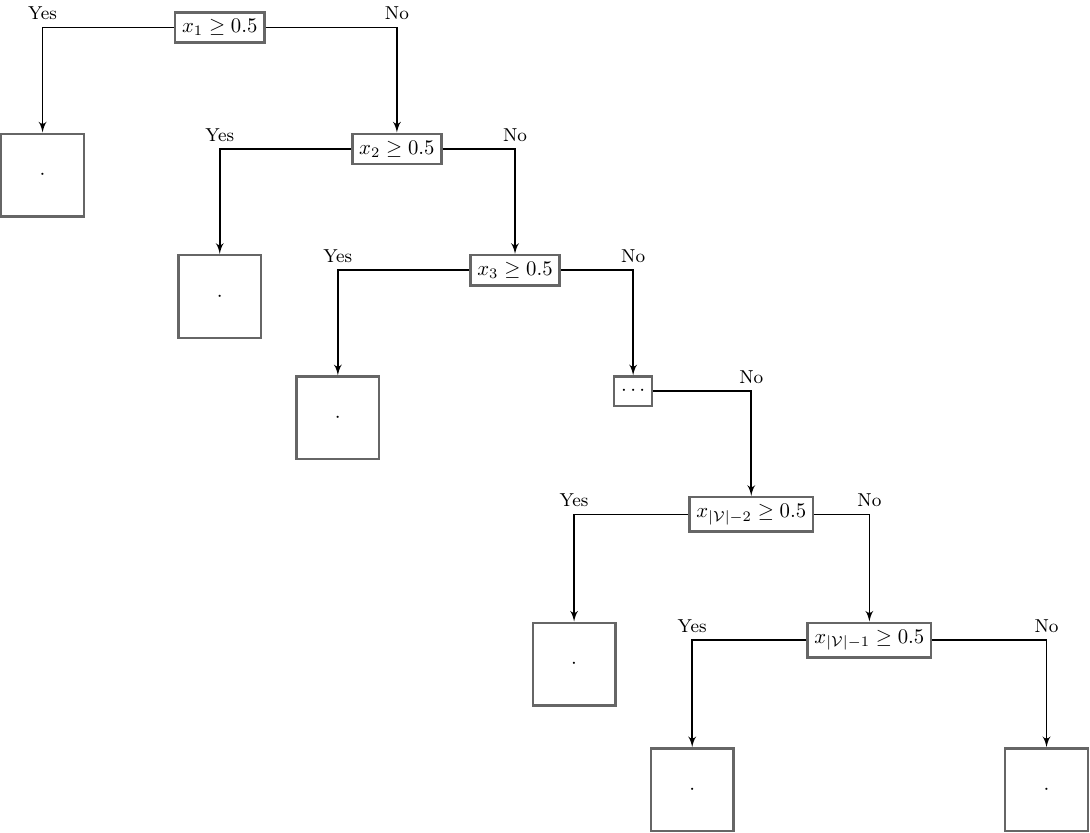}
         \caption{Decision tree $T$.}
         \label{fig:np-hard-tree-constraint}
  \end{subfigure}
  \caption{States, transitions and tree constraints (at both $t=1$ and $t=2$) in our tree MDP instance for solving the vertex cover problem.}
  \label{fig:MDP-np-hard}
\end{figure}
\hfill \Halmos
\endproof
}
\section{Details on the New York State guidelines for triage and reassessment}\label{app:details-NYS}
In New York State (NYS), the Crisis Standards of Care guidelines were codified by the NYS Taskforce on Life and the Law in the 2015 Ventilator Triage Guidelines~\citep{zuckerventilator}.  These guidelines outline clinical criteria for triage, including exclusion criteria and stratification of patients using the Sequential Organ Failure Assessment (SOFA) score. We present these guidelines in a tree policy form in Figure \ref{fig:NYS-guidelines}.
The goal of the NYS guidelines is to save the maximum number of lives. Note the important distinction with the number of \textit{life-years} saved; age should only be used as a tie-breaker. In particular,  the NYS does not use categorical exclusion of specific patients sub-populations, based on demographics (such as age, BMI) or comorbid health conditions (such as history of diabetes or congestive heart failures).

In particular,  prior to reaching ventilator capacity constraint, ventilators are allocated first-come-first-served. When the capacity constraint is reached, new patients are triaged using SOFA scores. i.e., at $t=0$ day:
\begin{itemize}
\item Those with SOFA $>$ 11 are categorized as blue (lowest priority) and do not receive a ventilator.
\item Those with 1 $<$ SOFA $<$ 8 are categorized as red (highest priority) and receive a ventilator first.
\item Those with SOFA between 8 and 11 are categorized as yellow (intermediate priority) and receive a ventilator as long as they are available and all patients in the red category have received a ventilator.
\item Those with SOFA = 0 are categorized as green (lowest priority, same as blue) and do not receive a ventilator.
\end{itemize}
\tb{ The NYS guidelines for $t=0$ day of intubation are represented in Figure \ref{fig:NYS-time-1}.
 At $t=2$ days and $t=5$ days, patients on ventilators are reassessed and categorized as blue, yellow or red, depending of the current SOFA score and the change since the previous SOFA score. This is represented in Figure \ref{fig:NYS-time-48} and Figure \ref{fig:NYS-time-120}.} Patients in blue and yellow categories are removed from the ventilator if a new patient with a higher priority requires one; patients in the blue category are removed first. 
 
 \tb{
 Importantly, note that in our representation of the NYS policies, we translate the priority classes \textit{blue} and \textit{green} as \textit{low} priority, \textit{yellow} as medium (med.) priority, and \textit{red} as \textit{high} priority. This is because patients in both \textit{blue} and \textit{green} priority classes will be proactively extubated.}
\begin{figure}[htb]
\center
 \begin{subfigure}{0.45\textwidth}
\centering
         \includegraphics[width=0.9\linewidth]{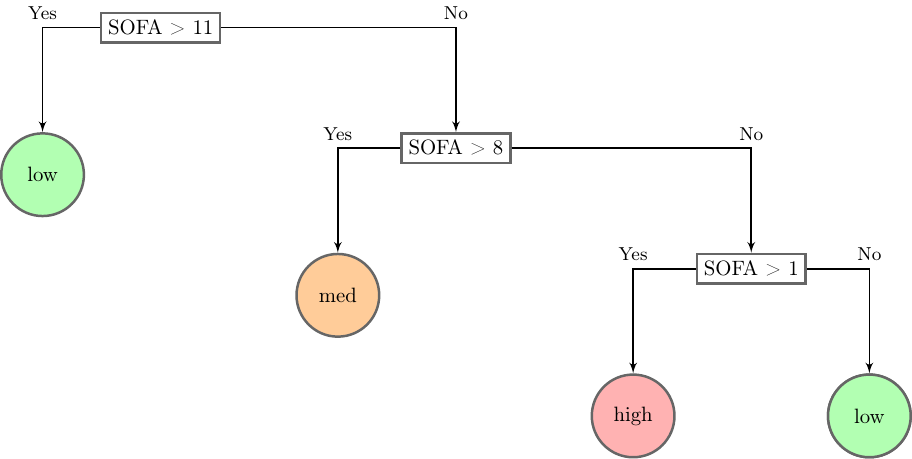}
         \caption{Triage.}
         \label{fig:NYS-time-1}
  \end{subfigure}
   \begin{subfigure}{0.45\textwidth}
\centering
         \includegraphics[width=0.9\linewidth]{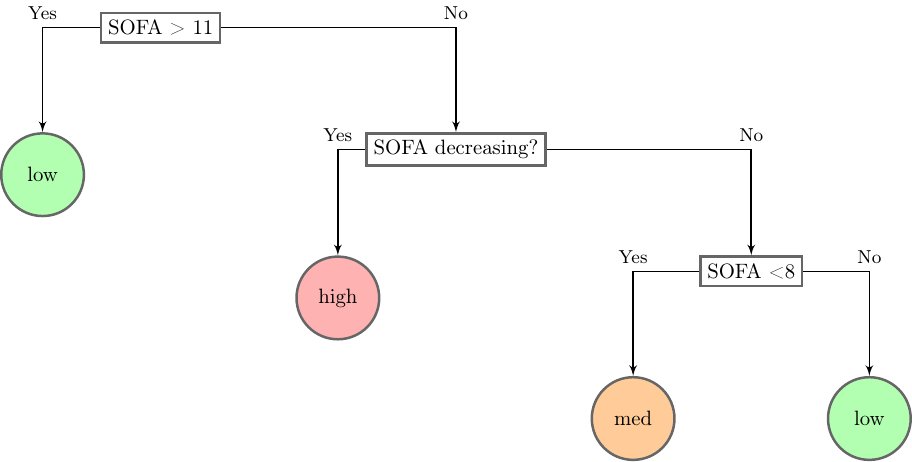}
         \caption{Reassessment at $t=2$ days.}
         \label{fig:NYS-time-48}
  \end{subfigure}
     \begin{subfigure}{0.45\textwidth}
\centering
         \includegraphics[width=0.9\linewidth]{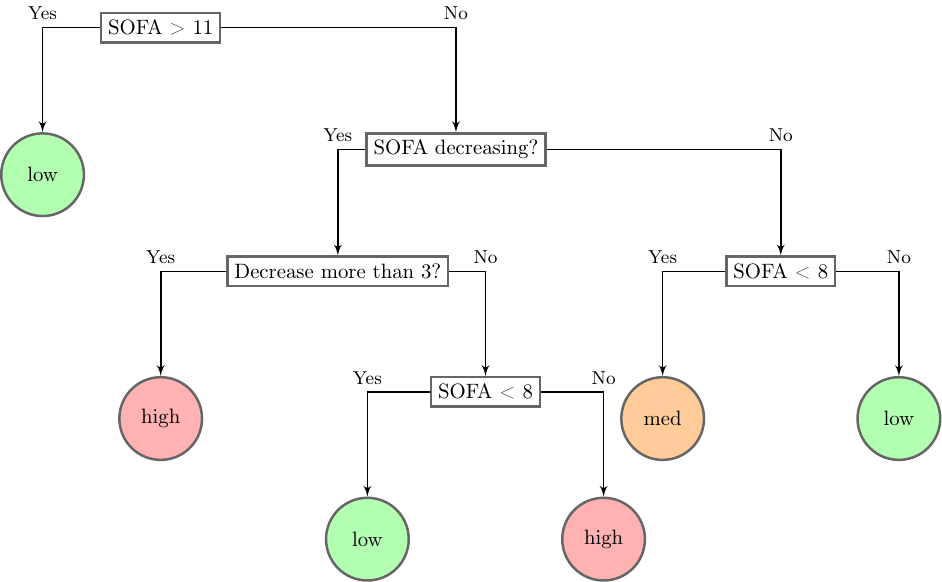}
         \caption{Reassessment at $t=5$ days.}
         \label{fig:NYS-time-120}
  \end{subfigure}
  \caption{Representation of the NYS guidelines as tree policies.}\label{fig:NYS-guidelines}
\end{figure}

\tb{
\section{Details on the clusters for our MDP models}\label{app:details cluster}
In this appendix, we provide more details on the clusters built for our MDP model with SOFA+AGE and SOFA+COVARIATES as sets of states. We use the {\sf KMEANS} function from the {\sf scikitlearn} Python package to build the $10$ clusters for our population of patients. Recall that we cluster the patients only using the information that remain constant during the hospital stay, e.g. age, BMI, dementia status, etc.. The summary statistics for our clusters are given in Table \ref{tab:cluster-sofa-age} for using a state set based on SOFA+AGE and in Table \ref{tab:cluster-sofa-cov} for using a state set based on SOFA+COVARIATES. 

The clusters from Table \ref{tab:cluster-sofa-age} are relatively easy to interpret. Since there is only one covariate used for clustering (the age of the patient), we simply obtains a partition of the entire range of observed age in our data $\{21,...,97\}$, into the $10$ subintervals $\{21,...,35\}$, $\{36, ... ,45\}$, ..., $\{85,...,97\}$. The clusters from Table \ref{tab:cluster-sofa-cov} include many covariates and are therefore more difficult to interpret, although we can notice the discrepancies between the clusters for some specific covariates, e.g. the average age of $29.4$ in cluster $1$ vs. the average age of $86.0$ in cluster $10$, or the much larger proportion of patients with dementia in cluster $9$ (25.5 \%) and cluster $10$ (49.3 \%) compared to the rest of the clusters.

\begin{table}[htb]
\centering
\footnotesize
 \begin{tabular}{||l |  c | c | c| c| c| c| c| c| c| c ||}
 \hline
Information $\setminus$ Cluster number& 1 & 2 & 3 & 4 &  5& 6& 7& 8& 9& 10 \\ [0.5ex]
 \hline\hline
 Number of patients (n) & 25 & 59 & 63 & 106 & 115 & 89 & 124 & 118 & 68 & 40 \\
 \hline
 Average age (year) & 29.4 & 41.9 & 49.5  & 55.6 & 60.9  & 65.4 & 69.7  & 75.3  & 80.9 & 88.3  \\
Minimum age (year) & 21.0 & 36.0 & 46.0 & 53.0 & 59.0 & 64.0 & 68.0 & 73.0 & 79.0 & 85.0\\
Maximum age (year) & 35.0 & 45.0 & 52.0  & 58.0 & 63.0 & 67.0 & 72.0 & 78.0 & 84.0 & 97.0\\
 [1ex]
 \hline
 \end{tabular}
 \caption{Summary statistics for our $10$ clusters only using age.}
 \label{tab:cluster-sofa-age}
\end{table}

\begin{table}[htb]
\centering
\footnotesize
\begin{tabular}{||l |  c | c | c| c| c| c| c| c| c| c ||}
 \hline
Information $\setminus$ Cluster number& 1 & 2 & 3 & 4 &  5& 6& 7& 8& 9& 10 \\ [0.5ex]
 \hline\hline
 Number of patients (n) & 23 & 30 & 70 & 106 & 101 & 35 & 156 & 113 & 106 & 67 \\
 \hline
 Average age (year) & 29.4  & 43.5 & 44.5 & 56.2 & 58.2 & 61.5  & 66.8 & 73.3 & 76.2  & 86.0  \\
Minimum age (year) & 21.0 & 31.0 & 37.0 & 49.0 & 50.0 & 54.0 & 61.0 & 66.0 & 72.0 & 81.0\\
Maximum age (year) & 36.0 & 52.0 & 51.0  & 61.0 & 65.0 & 70.0 & 71.0 & 84.0 & 81.0 & 97.0\\
\hline 
 Average Charlson score & 1.1  & 2.2 & 2.2 & 2.5 & 2.8 & 3.7  & 3.3 & 2.6 & 3.5  & 3.6  \\
  Average BMI & 31.1  & 46.6 & 28.7 & 25.7 & 35.0 & 46.5  & 26.7 & 36.4 & 26.3  & 25.7  \\
  Malignancy (\%) & 0.0 \% & 0.0 \% & 1.4 \% & 6.6 \% & 2.0 \% & 11.4 \% & 5.8 \% & 0.9 \% & 7.6 \% & 10.5 \%  \\ 
 Tumor (\%) & 4.4 \% & 0.0 \% & 0.0 \% & 3.8 \% & 0.0 \% & 0.0 \% & 1.9 \% & 0.0 \% & 0.0 \% & 3.0 \%  \\ 
  Diabetes (\%) & 8.7 \% & 26.7 \% & 21.4 \% & 17.9 \% & 34.7 \% & 42.7 \% & 41.7 \% & 35.4 \%  & 38.7 \% & 31.3 \% \\ 
  Dementia (\%) & 0.0 \% & 0.0 \% & 0.0 \% & 4.7 \%  & 3.0 \% & 2.9 \% & 10.3 \% & 6.2 \% & 25.5 \% & 49.3 \%\\ 
 Congestive heart failure (\%) & 13.0 \%  & 16.7 & 11.4 \% & 14.2 \% & 17.8 \% & 40.0 \% & 16.7 \% & 13.3 \% & 23.6 \% & 30.0 \% \\ 
  HIV (\%) & 0.0  \% & 0.0  \% & 1.4 \% & 1.9 \% & 3.0 \% & 0.0 \%  & 1.3 \% & 0.9 \% & 1.9 \%  & 0.0 \%  \\ 
  Liver disease (\%) & 4.4 \% & 0.0  \% & 1.4 \% & 3.8 \% & 2.0 \% & 2.9 \% & 0.1 \% & 0.0 \% & 0.9 \% & 1.5 \% \\ 
 [1ex]
 \hline
 \end{tabular}
 \caption{Summary statistics for our $10$ clusters only using age, demographics and comorbities.}
 \label{tab:cluster-sofa-cov}
\end{table}
}

\clearpage

\section{Tree policies and MDP computed by our  MDP model}\label{app:tree-policies}
\tb{
We present here the tree policies computed by our algorithms. We describe the MDP policies computed by our algorithms when the state space only includes SOFA, but we do not attempt to describe the MDP policies when the state space is SOFA+AGE and SOFA+COVARIATES, as the large dimensions of the state spaces make them difficult to describe. We note that the tree policies returned by Algorithm \ref{alg:optimize then fit}, which fits trees to each decision rule, achieve a very good accuracy and provides some intuition on the optimal unconstrained policies. }
\tb{
\paragraph{MDP policies (SOFA).}
The MDP policy only based on SOFA does not exclude any patients at triage. After 2 days of intubation, it only excludes patients with a SOFA score of $17$, larger than the SOFA score at intubation. After 5 days of intubation, it only excludes four types of patients: a patient with a decreasing SOFA score compared to the last reassessment and a SOFA score of $11$ or $12$, or a patient with an increasing SOFA score since the last period and a SOFA score of $14$ or $15$. 
}
\paragraph{Tree policies.}
\tb{For our MDP instance represented in Figure \ref{fig:MDP-model}, a tree policy is a sequence of three trees $T_{1},T_{2},T_{3}$ (one for triage and two for reassessments), and labeling rules $\mu_1,\mu_2,\mu_3$ mapping patients to the action {\sf maintain} or the action {\sf exclude}. For the sake of comparisons with the NYS guidelines represented in Figure \ref{fig:NYS-guidelines}, which assigns patients to priority classes ({\sf high, medium, low}), we decided to write {\sf high} for the action {\sf maintain} and {\sf low} for the action {\sf exclude} in the figures representing our tree policies. This also improves the readability of the figures given the small amount of space available at each leaf of the trees.} 

\tb{When the state space is SOFA, both Algorithm \ref{alg:optimize then fit} and Algorithm \ref{alg:optimize and fit} return the same tree policies, represented in Figure \ref{fig:tree-sofa}. The tree policies returned by Algorithm \ref{alg:optimize then fit} when the state space is SOFA+AGE or SOFA+COVARIATES are represented in Figure \ref{fig:tree-sofa-age-algo-1} and Figure \ref{fig:tree-sofa-cov-algo-1}, and the tree policies returned by Algorithm \ref{alg:optimize and fit} when the state space is SOFA+AGE or SOFA+COVARIATES are represented in Figure \ref{fig:tree-sofa-age-algo-2} and Figure \ref{fig:tree-sofa-cov-algo-2}. We note that in some cases, a small SOFA score will lead to exclusion from ventilator treatment (e.g. at triage in the tree policy based on SOFA+AGE returned by Algorithm, see Figure \ref{fig:tree-sofa-age-time-1-algo-1}), while in most other cases, a small SOFA score is an indication that the patient has a good chance of survival and is used to determine intubation. We attribute this contrast to the trees returned by Algorithm \ref{alg:optimize then fit} and Algorithm \ref{alg:optimize and fit} being obtained by fits to {\em unconstrained} policies that are difficult to interpret.
\paragraph{Accuracy of the trees.}
We also report here the accuracy obtained in the $\fit$ functions for the trees computed by Algorithm \ref{alg:optimize then fit} and Algorithm \ref{alg:optimize and fit}. Note that while Algorithm \ref{alg:optimize then fit} fits trees to {\em different} policies: Algorithm \ref{alg:optimize then fit} fits trees to the optimal unconstrained policies, while Algorithm \ref{alg:optimize and fit} fits trees to a policy that is not necessarily unconstrained optimal (see Proposition \ref{prop:analysis of our algorithms}). We present the accuracy depending on the state space used in the MDP model.
\begin{itemize}
    \item {\em SOFA.} Recall that here Algorithm \ref{alg:optimize then fit} and Algorithm \ref{alg:optimize and fit} returns the same tree policy (Figure \ref{fig:tree-sofa}). The accuracy is 1.0 at triage, 1.0 at $t=$ 2 days, and $0.98$ at $t=$ 5 days. Therefore, the tree policy from Figure \ref{fig:tree-sofa} almost exactly captures the decisions from the optimal unconstrained policy (MDP SOFA).
    \item {\em SOFA+AGE.} For the trees returned by Algorithm \ref{alg:optimize then fit} at triage, $t=$ 2 days and $t=$ 5 days (Figure \ref{fig:tree-sofa-age-algo-1}), the accuracies are 0.87, 0.84, 0.84. For the trees returned by Algorithm \ref{alg:optimize and fit} (Figure \ref{fig:tree-sofa-age-algo-2}), the accuracies are 0.90,0.79, 0.84.
    \item {\em SOFA+COV.} For the trees returned by Algorithm \ref{alg:optimize then fit} at triage, $t=$ 2 days and $t=$ 5 days (Figure \ref{fig:tree-sofa-cov-algo-1}), the accuracies are 0.99, 0.83, 0.81. For the trees returned by Algorithm \ref{alg:optimize and fit} (Figure \ref{fig:tree-sofa-cov-algo-2}), the accuracies are 0.97,0.86, 0.81.
\end{itemize}

}


\begin{figure}[htb]
\center
 \begin{subfigure}{0.7\textwidth}
\centering
         \includegraphics[width=0.9\linewidth]{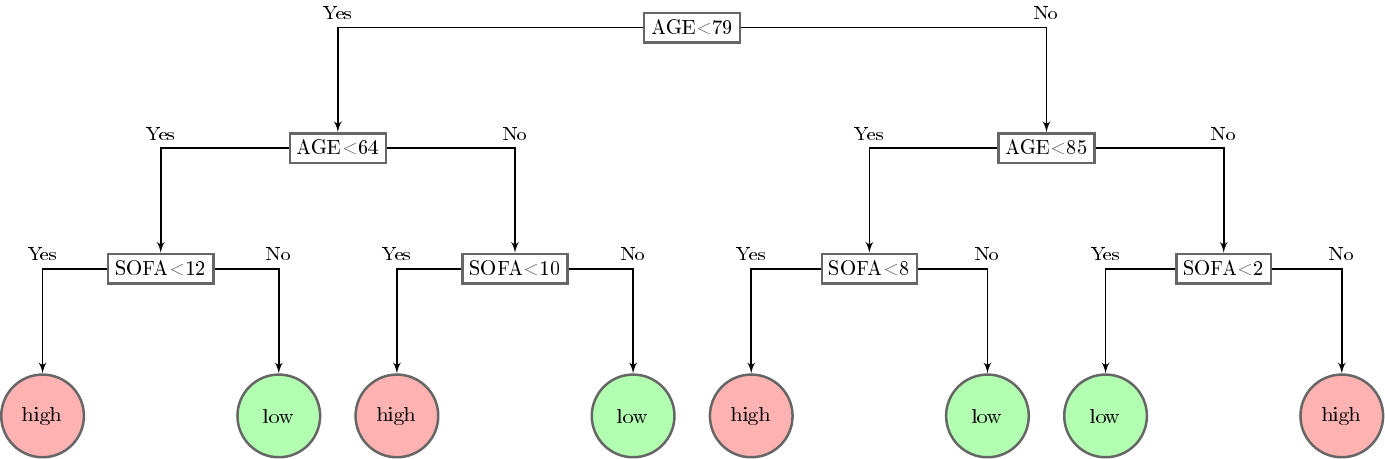}
         \caption{Triage.}
         \label{fig:tree-sofa-age-time-1-algo-1}
  \end{subfigure}
   \begin{subfigure}{0.7\textwidth}
\centering
         \includegraphics[width=0.9\linewidth]{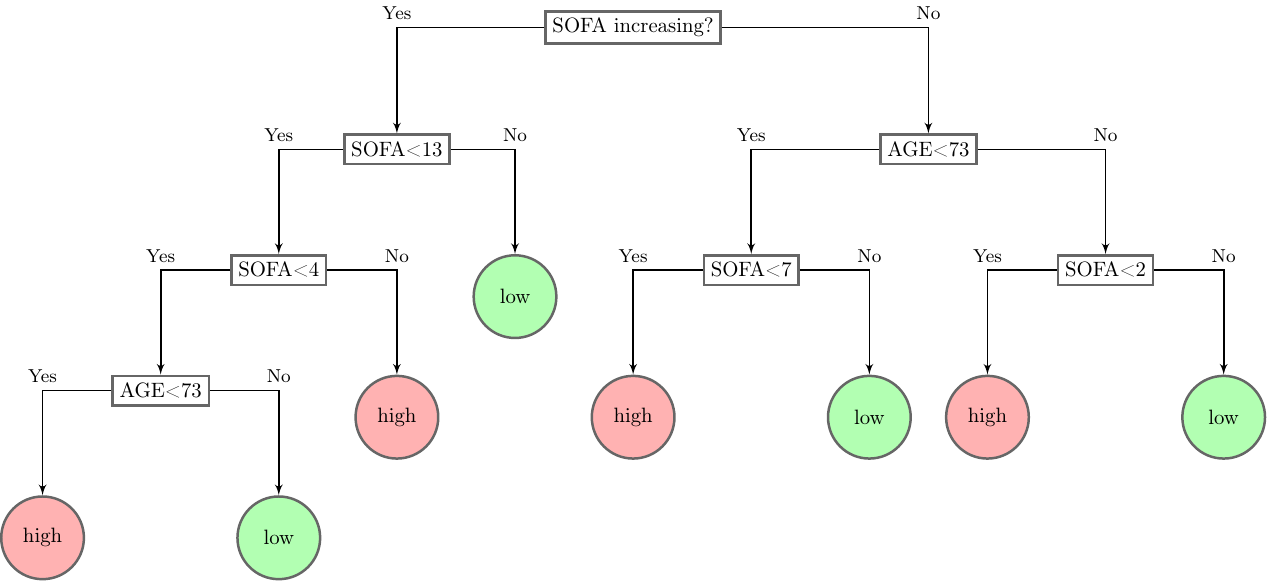}
         \caption{Reassessment at $t=$ 2 days.}
         \label{fig:tree-sofa-age-time-48-algo-1}
  \end{subfigure}
     \begin{subfigure}{0.7\textwidth}
\centering
         \includegraphics[width=0.9\linewidth]{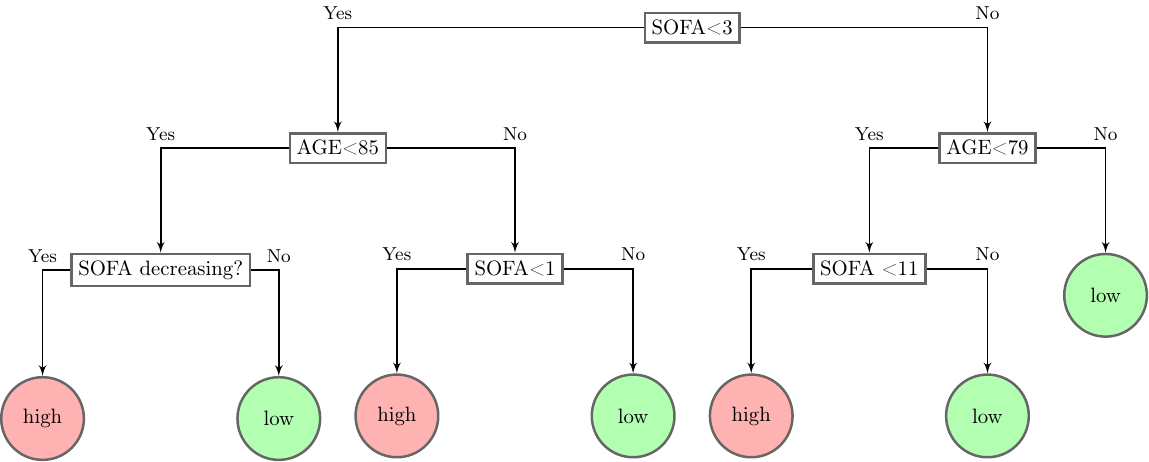}
         \caption{Reassessment at $t=$ 5 days.}
         \label{fig:tree-sofa-age-time-120-algo-1}
  \end{subfigure}
  \caption{Tree policies returned by Algorithm \ref{alg:optimize then fit}, based on SOFA and age.}\label{fig:tree-sofa-age-algo-1}
\end{figure}

\begin{figure}[htb]
\center
 \begin{subfigure}{0.7\textwidth}
\centering
         \includegraphics[width=0.9\linewidth]{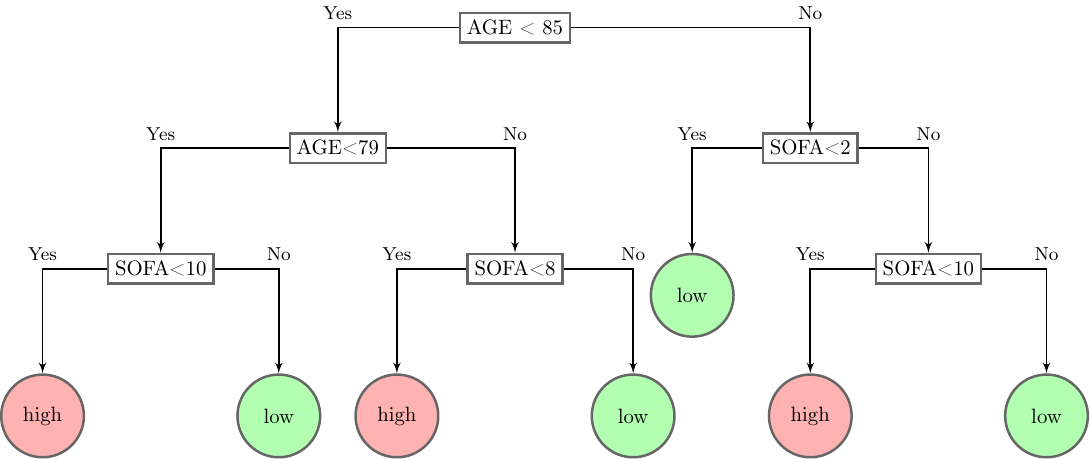}
         \caption{Triage.}
         \label{fig:tree-sofa-age-time-1-algo-2}
  \end{subfigure}
   \begin{subfigure}{0.7\textwidth}
\centering
         \includegraphics[width=0.9\linewidth]{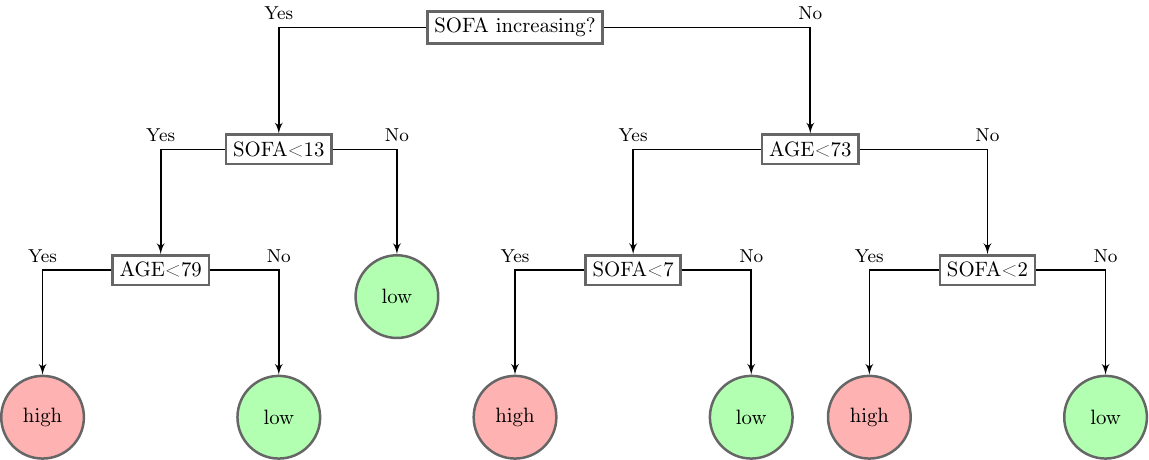}
         \caption{Reassessment at $t=$ 2 days.}
         \label{fig:tree-sofa-age-time-48-algo-2}
  \end{subfigure}
     \begin{subfigure}{0.7\textwidth}
\centering
         \includegraphics[width=0.9\linewidth]{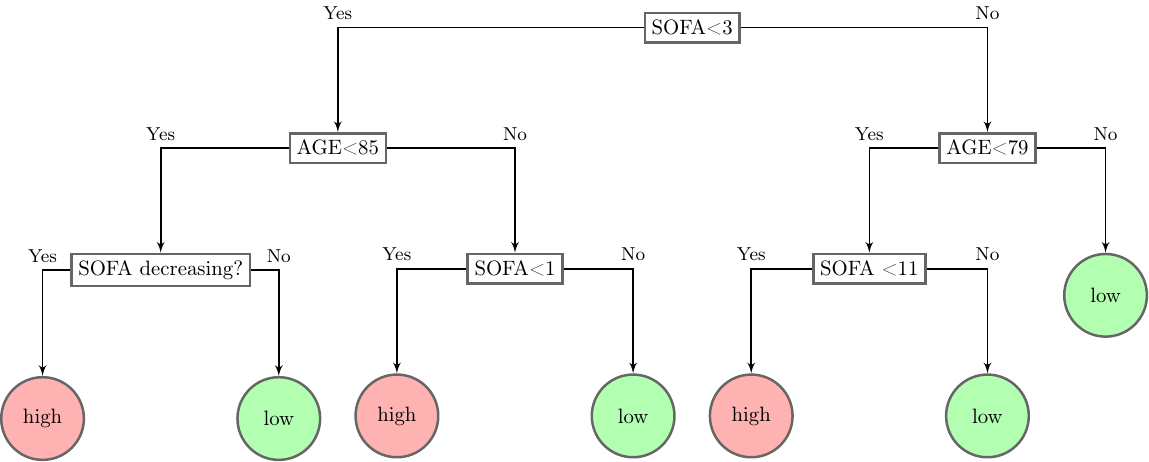}
         \caption{Reassessment at $t=$ 5 days.}
         \label{fig:tree-sofa-age-time-120-algo-2}
  \end{subfigure}
  \caption{Tree policies returned by Algorithm \ref{alg:optimize and fit}, based on SOFA and age.}\label{fig:tree-sofa-age-algo-2}
\end{figure}

\begin{figure}[htb]
\center
 \begin{subfigure}{0.5\textwidth}
\centering
         \includegraphics[width=0.9\linewidth]{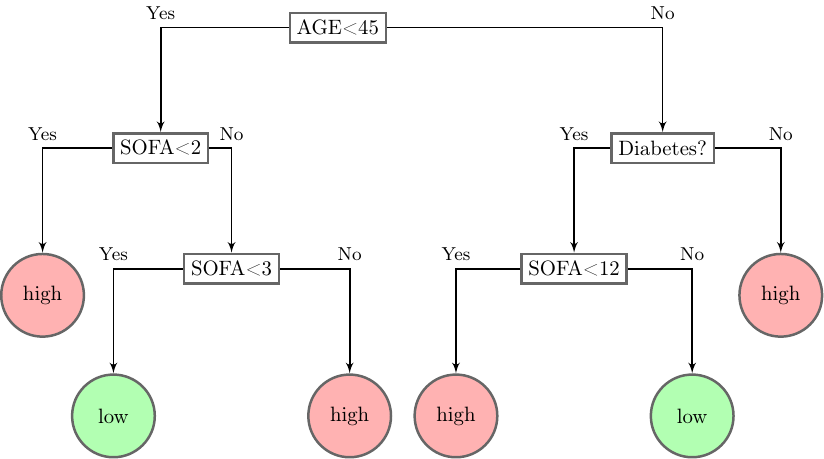}
         \caption{Triage.}
         \label{fig:tree-sofa-cov-time-1-algo-1}
  \end{subfigure}
   \begin{subfigure}{0.7\textwidth}
\centering
         \includegraphics[width=0.9\linewidth]{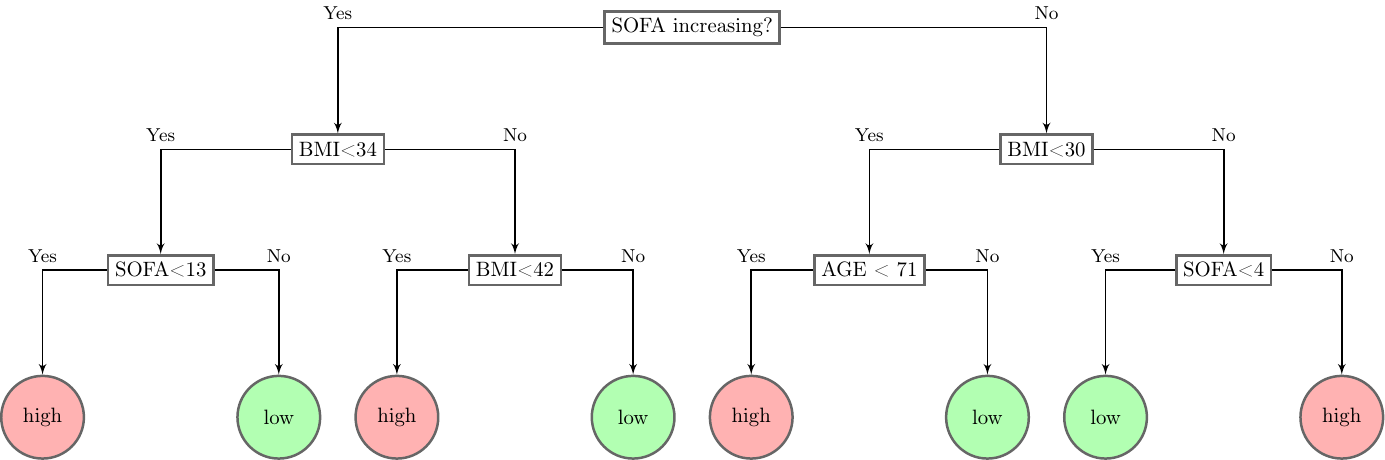}
         \caption{Reassessment at $t=$ 2 days.}
         \label{fig:tree-sofa-cov-time-48-algo-1}
  \end{subfigure}
     \begin{subfigure}{0.7\textwidth}
\centering
         \includegraphics[width=0.9\linewidth]{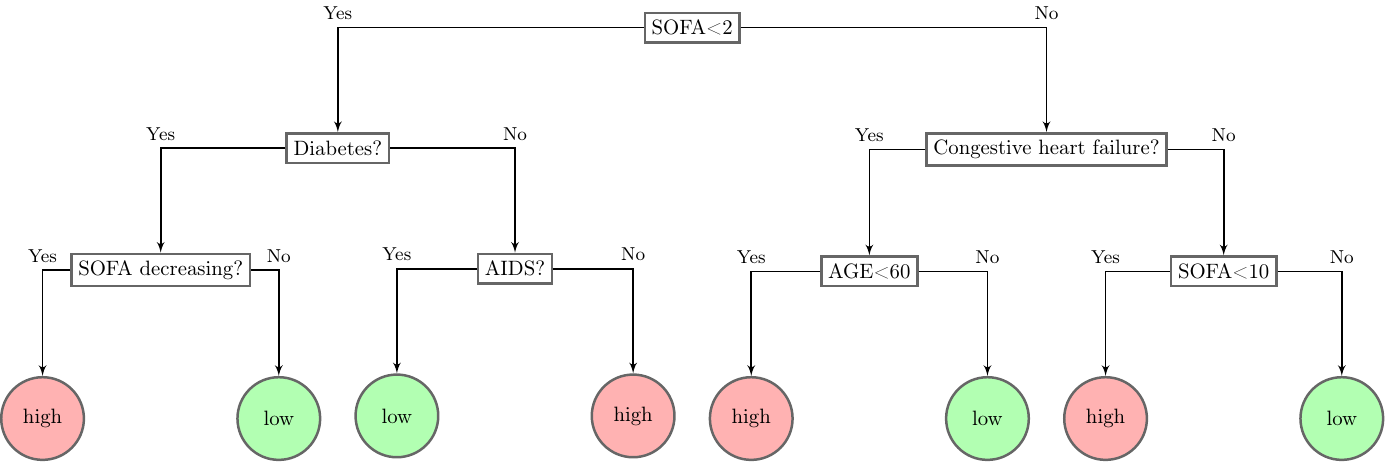}
         \caption{Reassessment at $t=$ 5 days.}
         \label{fig:tree-sofa-cov-time-120-algo-1}
  \end{subfigure}
  \caption{Tree policies returned by Algorithm \ref{alg:optimize then fit}, based on SOFA, age, demographics and comorbities.}\label{fig:tree-sofa-cov-algo-1}
\end{figure}

\begin{figure}[htb]
\center
 \begin{subfigure}{0.5\textwidth}
\centering
         \includegraphics[width=0.9\linewidth]{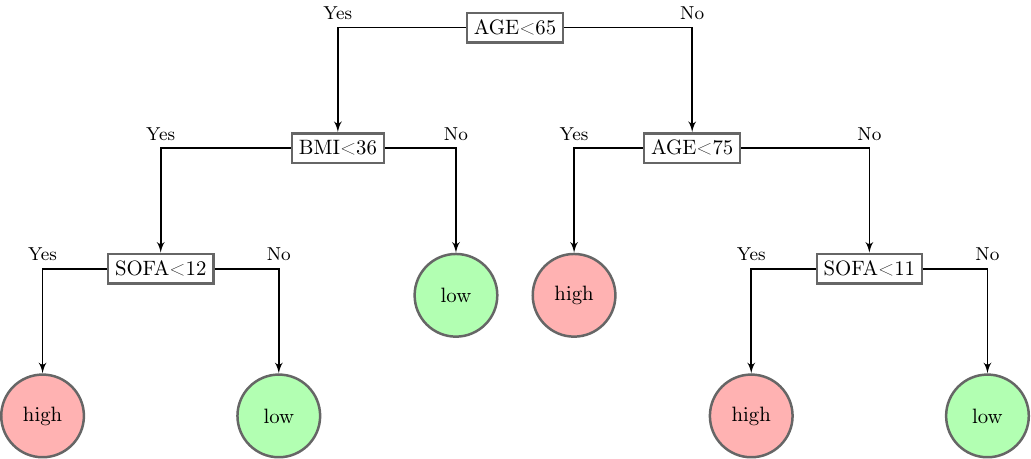}
         \caption{Triage.}
         \label{fig:tree-sofa-cov-time-1-algo-2}
  \end{subfigure}
   \begin{subfigure}{0.7\textwidth}
\centering
         \includegraphics[width=0.9\linewidth]{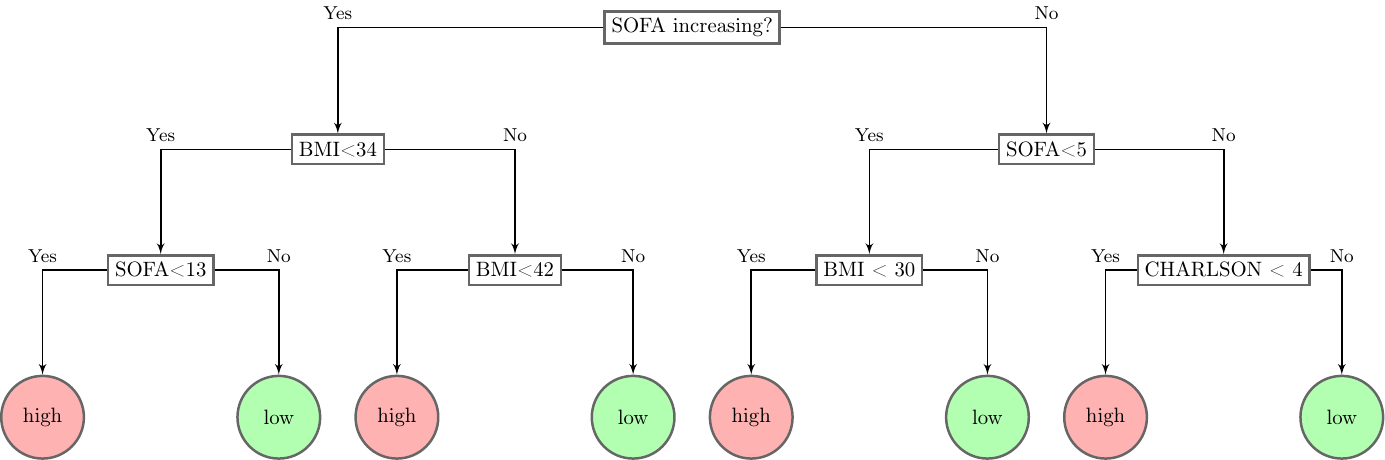}
         \caption{Reassessment at $t=$ 2 days.}
         \label{fig:tree-sofa-cov-time-48-algo-2}
  \end{subfigure}
     \begin{subfigure}{0.7\textwidth}
\centering
         \includegraphics[width=0.9\linewidth]{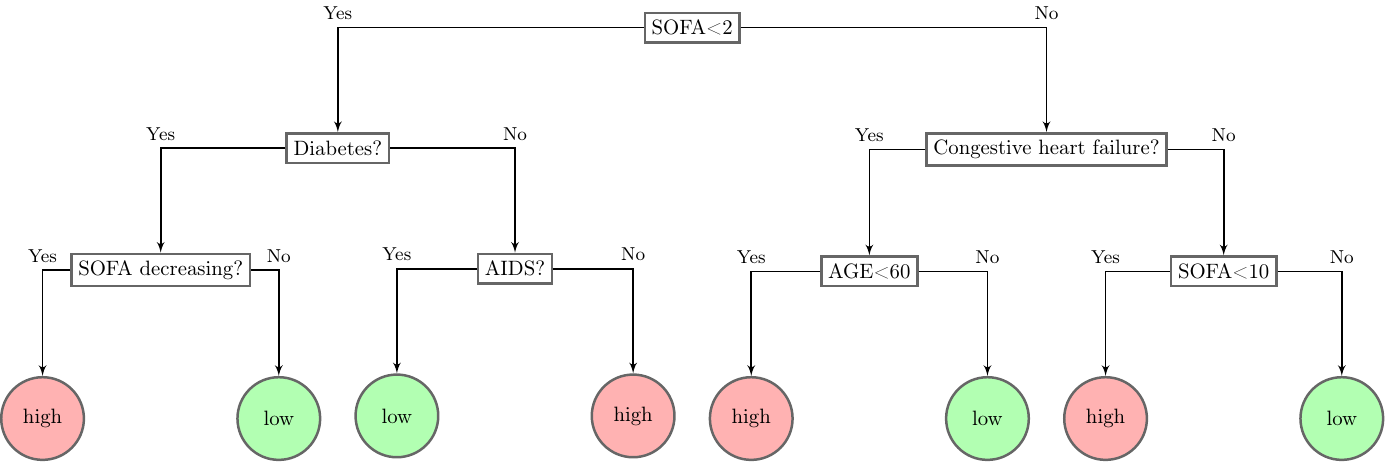}
         \caption{Reassessment at $t=$ 5 days.}
         \label{fig:tree-sofa-cov-time-120-algo-2}
  \end{subfigure}
  \caption{Tree policies returned by Algorithm \ref{alg:optimize and fit}, based on SOFA, age, demographics and comorbities.}\label{fig:tree-sofa-cov-algo-2}
\end{figure}


\clearpage

\section{Comparing the performances of Algorithm \ref{alg:optimize then fit} and Algorithm \ref{alg:optimize and fit}}\label{app:comparison algo 1/2}

\tb{
In this appendix, we provide a more detailed comparison of the performances of the unconstrained MDP policies with the tree policies obtained by Algorithm \ref{alg:optimize then fit} and Algorithm \ref{alg:optimize and fit}. To do so, we reorganize the numerical results from Figure \ref{fig:number-of-deaths}, grouping them by the information used in the state space: only SOFA (Figure \ref{fig:sofa-comp}), SOFA+AGE (Figure \ref{fig:sofa-age-comp}) and SOFA+COVARIATES (Figure \ref{fig:sofa-cov-comp}). This allows us to show the performances of all algorithms on the same figure. We note that when only SOFA is used, tree policies can match the performance of optimal unconstrained policies. When more information is used, Algorithm \ref{alg:optimize then fit} and Algorithm \ref{alg:optimize and fit} exhibit similar performances, with a slight advantage for Algorithm \ref{alg:optimize then fit} when SOFA+COV is used (Figure \ref{fig:sofa-cov-comp}) and insignificant performance difference between both algorithms when only SOFA+AGE is used (Figure \ref{fig:sofa-age-comp}). 
}

\begin{figure}[htb]
\center
 \begin{subfigure}{0.3\textwidth}
\centering
         \includegraphics[width=1.0\linewidth]{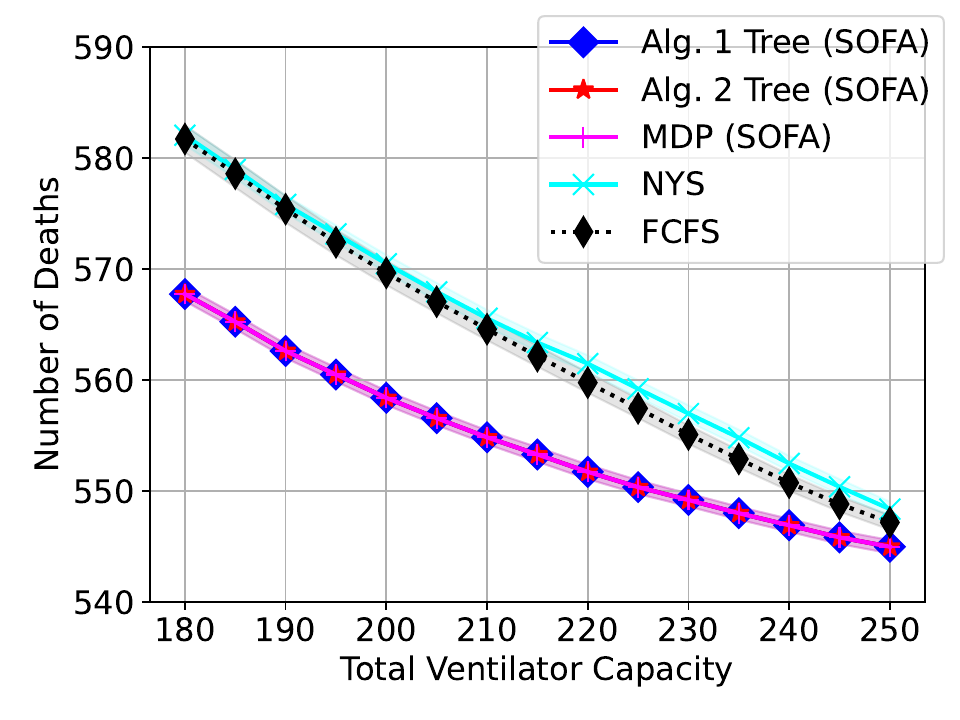}
                  \caption{SOFA.}
         \label{fig:sofa-comp}
  \end{subfigure}
   \begin{subfigure}{0.3\textwidth}
\centering
         \includegraphics[width=1.0\linewidth]{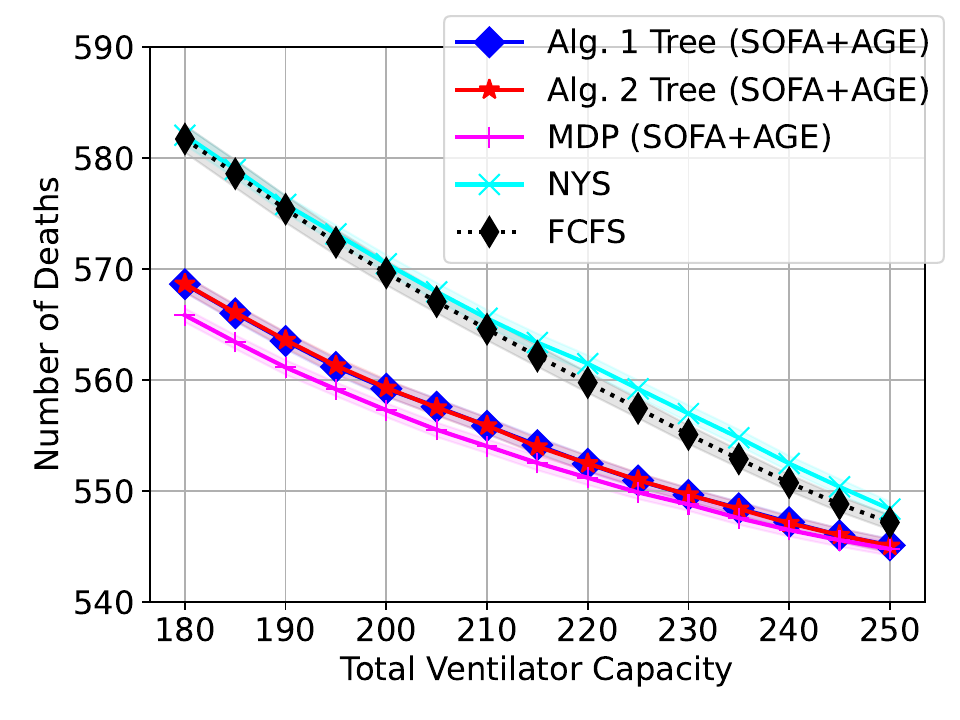
}
               \caption{SOFA+AGE.}
         \label{fig:sofa-age-comp}
  \end{subfigure}
\begin{subfigure}{0.3\textwidth}
\centering
         \includegraphics[width=1.0\linewidth]{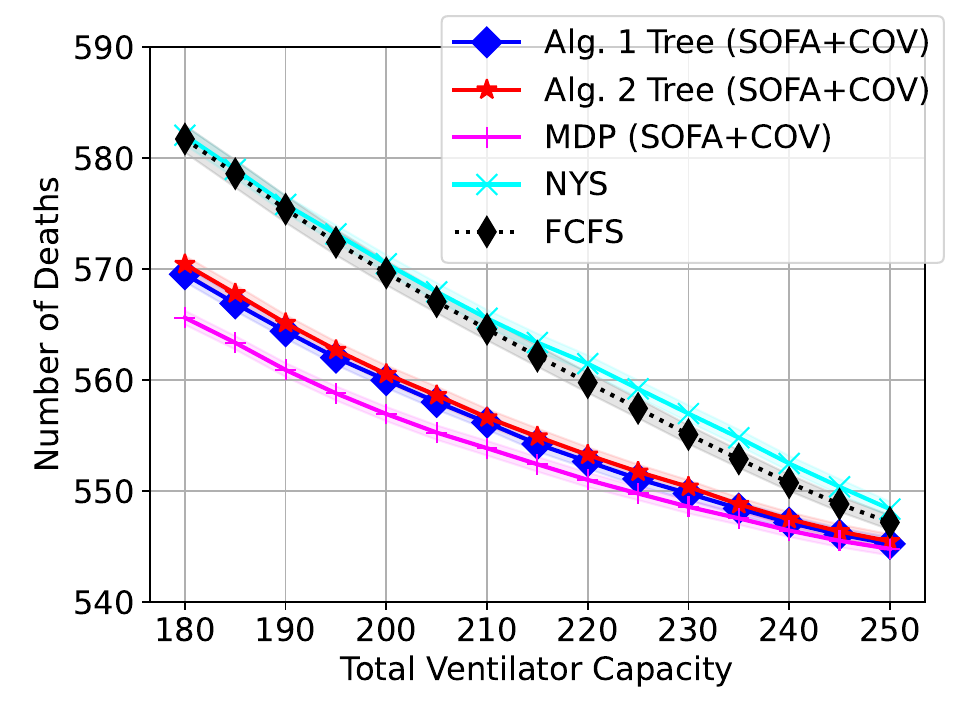
}
               \caption{SOFA+COV.}
         \label{fig:sofa-cov-comp}
  \end{subfigure}
         \caption{Number of deaths for various triage guidelines at hypothetical levels of ventilator capacities, for $p=0.99$. We group the performances based on the state space used by the algorithms: only using SOFA, using SOFA+AGE, using SOFA+COVARIATES.}
         \label{fig:number-of-deaths-per-alg}
\end{figure}
\clearpage 
\section{Sensitivity analysis for the parameter $p$}\label{app:sensitivity-analysis-parameter-p}

\tb{
In this appendix we provide some more numerical experiments to investigate the impact of the choice of the parameter $p$ on our results from Section \ref{sec:simu}. As already discussed in the main body of the paper, $p$ may depend on the demographics and vital signs of the patient, as well as the timing of proactive extubation. For the sake of simplicity, we consider a model where the parameter $p$ depends on the value of the SOFA score of the patient at the time of proactive extubation, compared to a threshold $\tau$. In particular, we consider that $p$ is such that $p=p_{{\sf up}} \in [0,1]$ if the SOFA score at proactive extubation is larger or equal to $\tau$, or equal to $p=p_{{\sf down}} \leq p_{{\sf up}}$ if the SOFA score at proactive extubation is strictly smaller than $\tau$. in our simulations we explore values of $\tau \in \{8,11\}$ (as these thresholds are present in the NYS guidelines) and values of $(p_{{\sf up}},p_{{\sf down}}) \in \{(0.99,0.95),(0.95,0.90)\}$ (after discussing this issue with clinicians). As we can observe from Figures \ref{fig:number-of-deaths (Sensitive p0)}-\ref{fig:number-of-deaths (Sensitive p3)}, all these models for the parameter $p$ lead to comparable performances for the policies studied in this paper (NYS, FCFS, MDP and tree policies based on Algorithm \ref{alg:optimize then fit} and Algorithm \ref{alg:optimize and fit}). This highlights the robustness of our conclusions from Section \ref{sec:simu}.
}

\begin{figure}[htb]
\center
 \begin{subfigure}{0.3\textwidth}
\centering
         \includegraphics[width=1.0\linewidth]{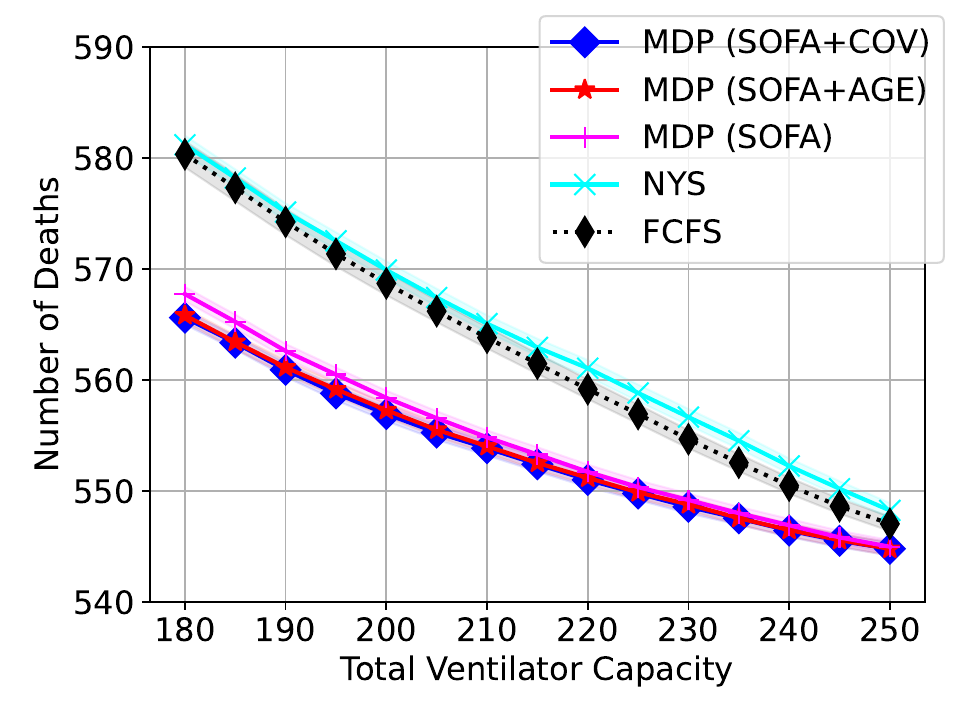}
                  \caption{MDP policies.}
         \label{fig:MDP policies Sensitive p0}
  \end{subfigure}
   \begin{subfigure}{0.3\textwidth}
\centering
         \includegraphics[width=1.0\linewidth]{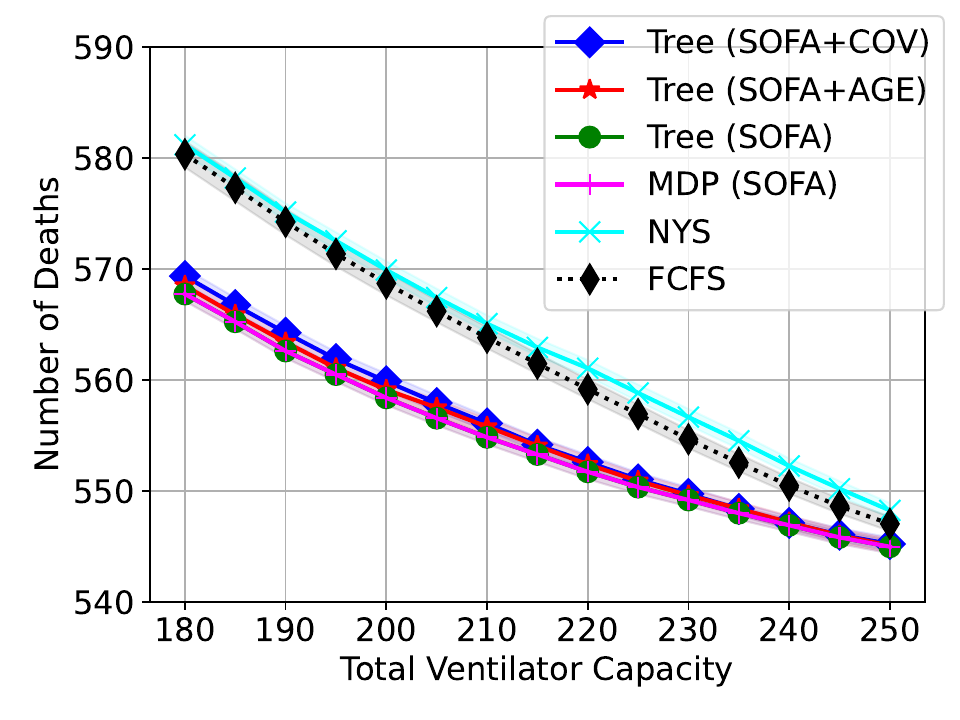}
               \caption{Tree policies from Alg. \ref{alg:optimize then fit}.}
         \label{fig:tree policies optimize then fit Sensitive p0}
  \end{subfigure}
\begin{subfigure}{0.3\textwidth}
\centering
         \includegraphics[width=1.0\linewidth]{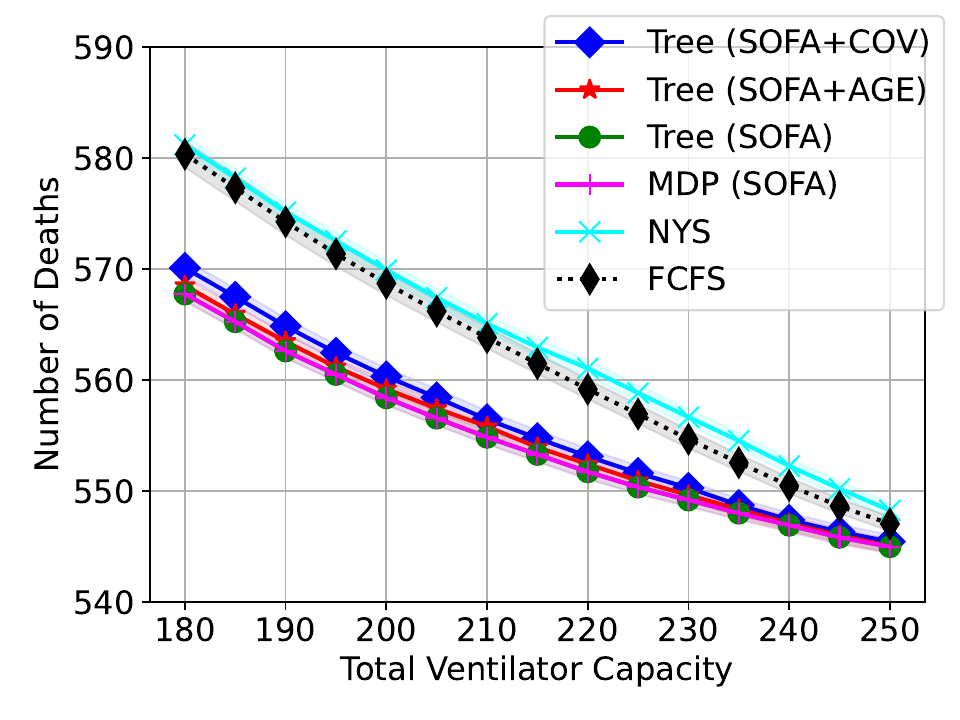}
               \caption{Tree policies from Alg. \ref{alg:optimize and fit}.}
         \label{fig:tree policies optimize and fit (Sensitive p0)}
  \end{subfigure}
         \caption{Number of deaths for various triage guidelines at hypothetical levels of ventilator capacities, for $\tau = 8$ and $(p_{{\sf up}},p_{{\sf down}}) = (0.99,0.95)$. }
         \label{fig:number-of-deaths (Sensitive p0)}
\end{figure}

\begin{figure}[htb]
\center
 \begin{subfigure}{0.3\textwidth}
\centering
         \includegraphics[width=1.0\linewidth]{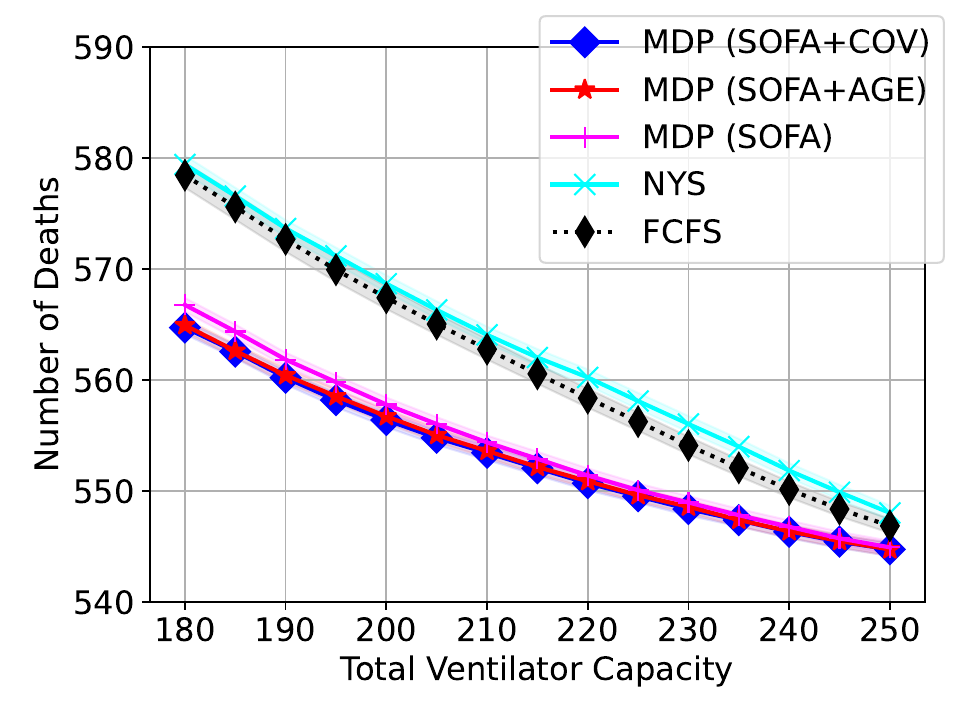}
                  \caption{MDP policies.}
         \label{fig:MDP policies Sensitive p1}
  \end{subfigure}
   \begin{subfigure}{0.3\textwidth}
\centering
         \includegraphics[width=1.0\linewidth]{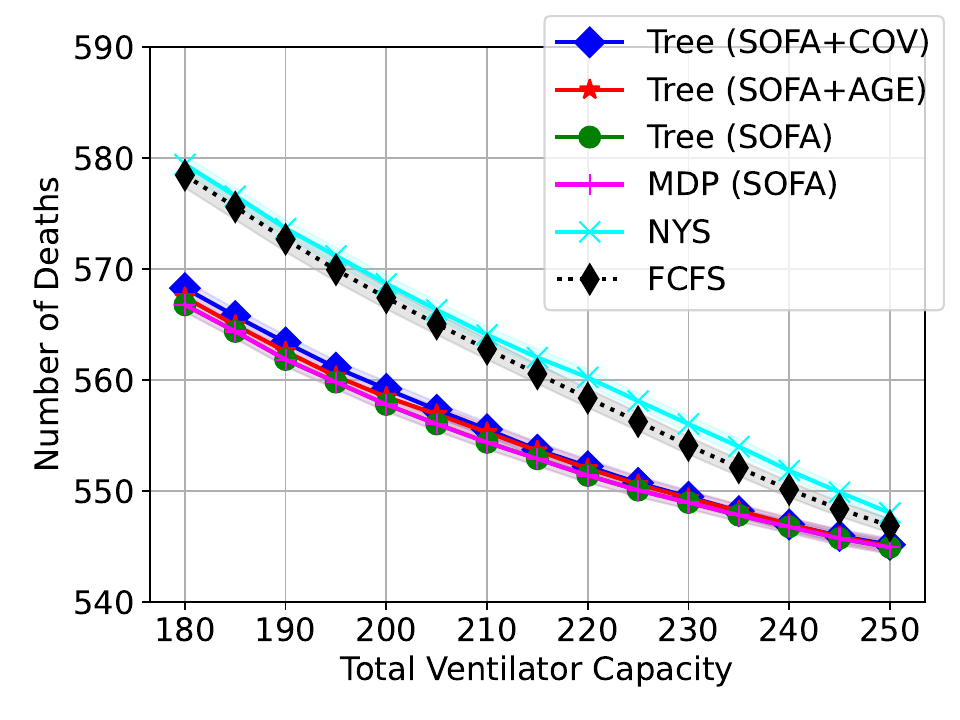}
               \caption{Tree policies from Alg. \ref{alg:optimize then fit}.}
         \label{fig:tree policies optimize then fit Sensitive p1}
  \end{subfigure}\begin{subfigure}{0.3\textwidth}
\centering
         \includegraphics[width=1.0\linewidth]{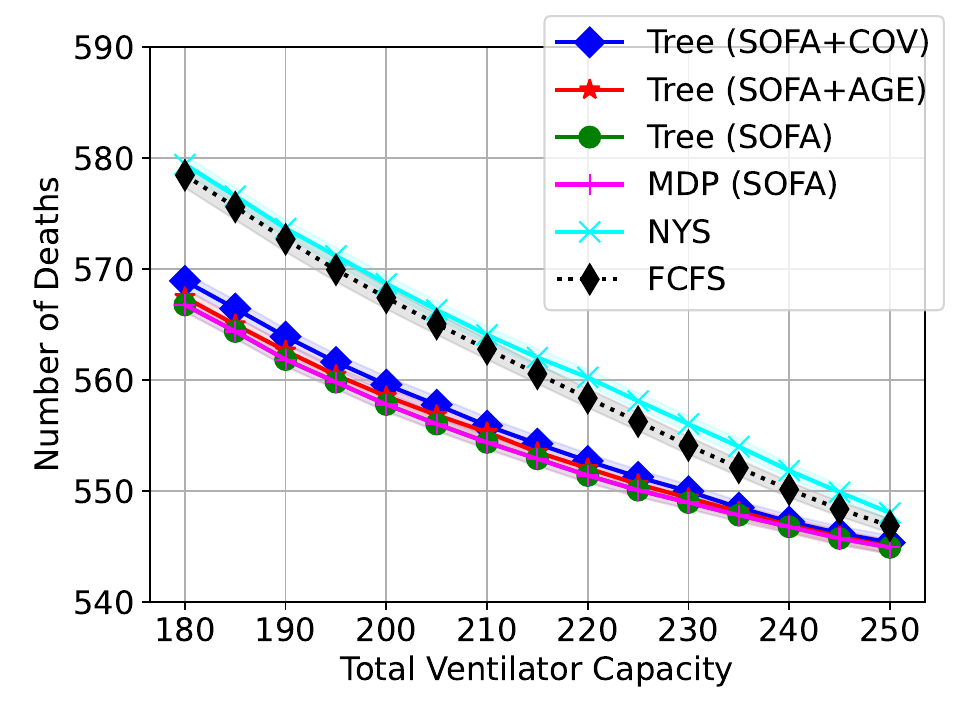}
               \caption{Tree policies from Alg. \ref{alg:optimize and fit}.}
         \label{fig:tree policies optimize and fit (Sensitive p1)}
  \end{subfigure}
         \caption{Number of deaths for various triage guidelines at hypothetical levels of ventilator capacities, for $\tau = 8$ and $(p_{{\sf up}},p_{{\sf down}}) = (0.95,0.90)$. }
         \label{fig:number-of-deaths (Sensitive p1)}
\end{figure}

\begin{figure}[htb]
\center
 \begin{subfigure}{0.3\textwidth}
\centering
         \includegraphics[width=1.0\linewidth]{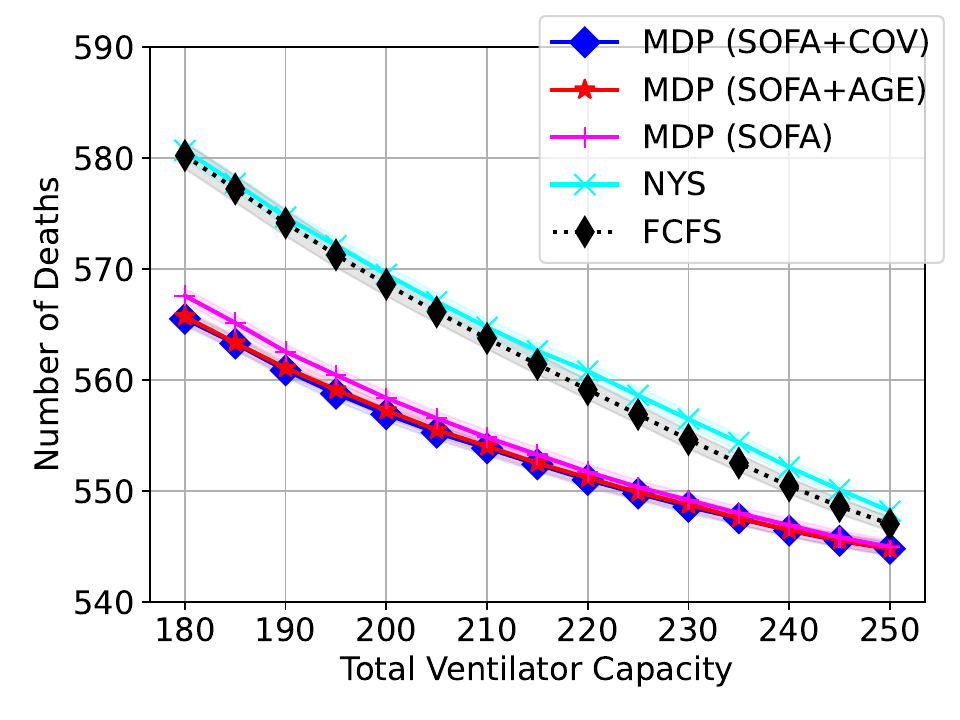}
                  \caption{MDP policies.}
         \label{fig:MDP policies Sensitive p2}
  \end{subfigure}
   \begin{subfigure}{0.3\textwidth}
\centering
         \includegraphics[width=1.0\linewidth]{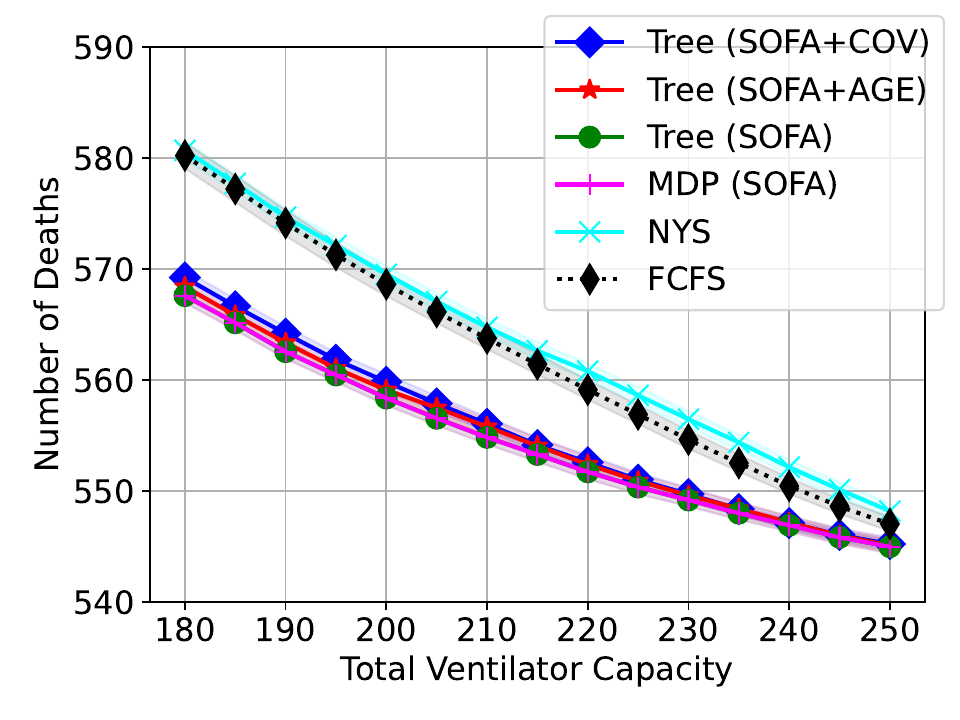}
               \caption{Tree policies from Alg. \ref{alg:optimize then fit}.}
         \label{fig:tree policies optimize then fit Sensitive p2}
  \end{subfigure}
\begin{subfigure}{0.3\textwidth}
\centering
         \includegraphics[width=1.0\linewidth]{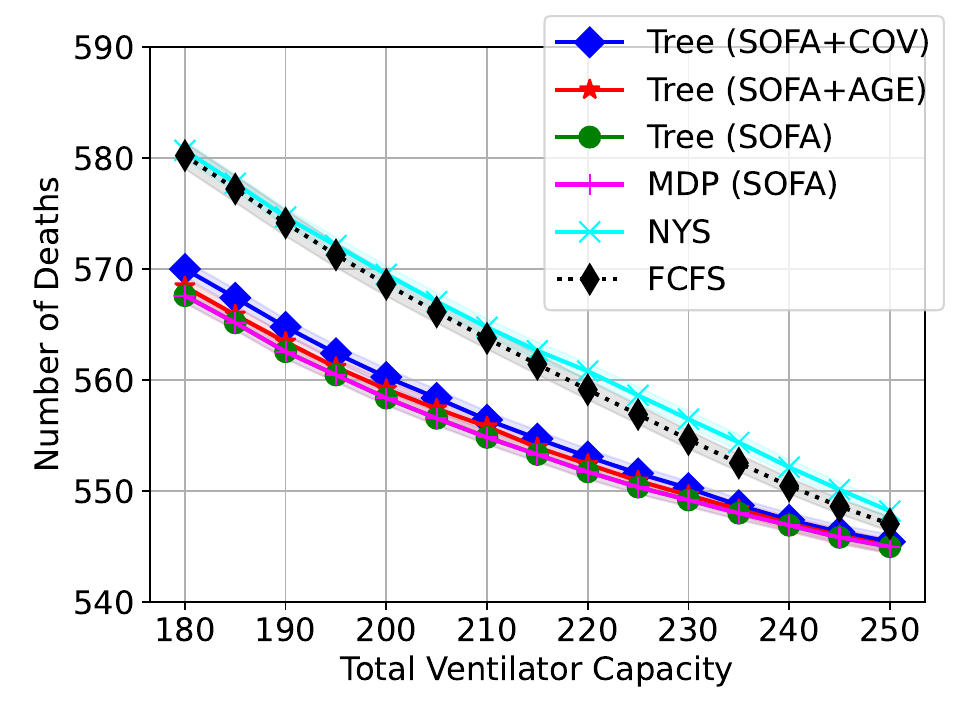}
               \caption{Tree policies from Alg. \ref{alg:optimize and fit}.}
         \label{fig:tree policies optimize and fit (Sensitive p2)}
  \end{subfigure}
         \caption{Number of deaths for various triage guidelines at hypothetical levels of ventilator capacities, for $\tau = 11$ and $(p_{{\sf up}},p_{{\sf down}}) = (0.99,0.95)$. }
         \label{fig:number-of-deaths (Sensitive p2)}
\end{figure}

\begin{figure}[htb]
\center
 \begin{subfigure}{0.3\textwidth}
\centering
         \includegraphics[width=1.0\linewidth]{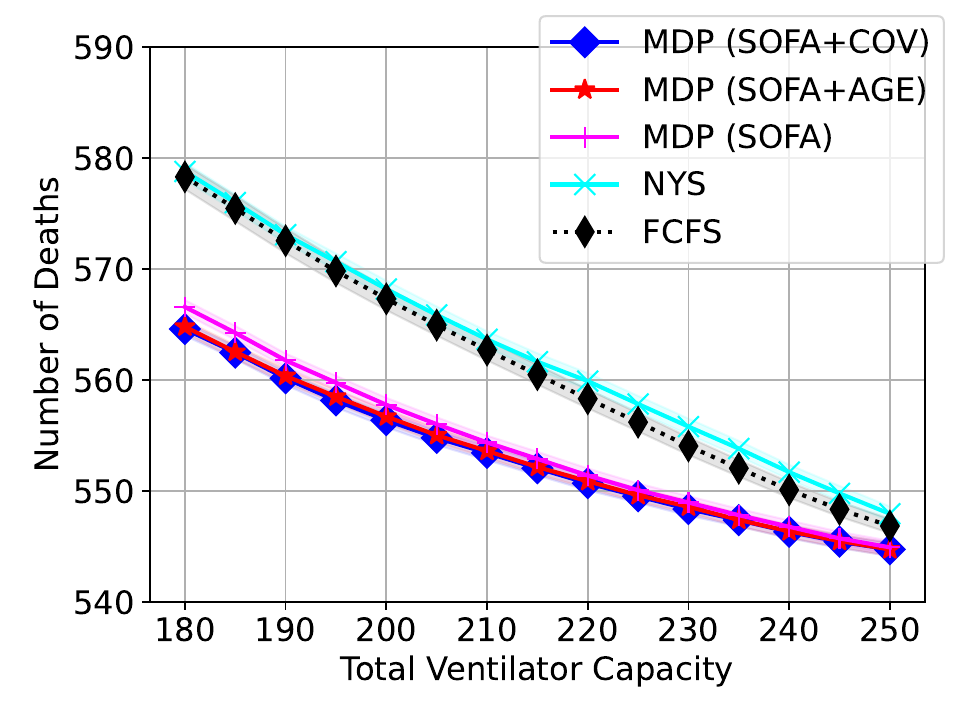}
                  \caption{MDP policies.}
         \label{fig:MDP policies Sensitive p3}
  \end{subfigure}
   \begin{subfigure}{0.3\textwidth}
\centering
         \includegraphics[width=1.0\linewidth]{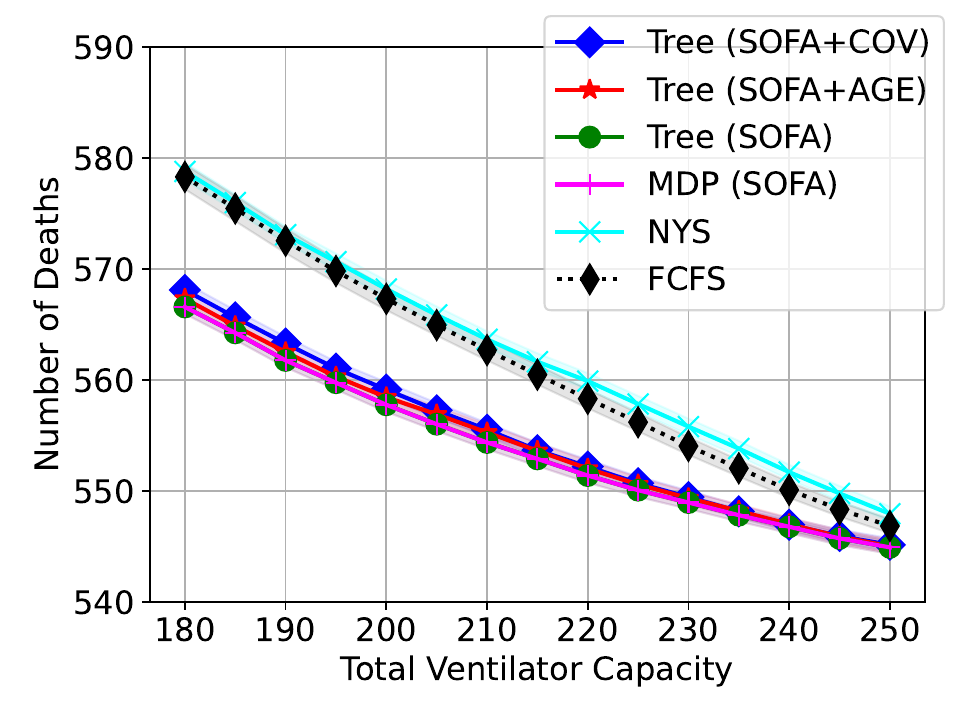}
               \caption{Tree policies from Alg. \ref{alg:optimize then fit}.}
         \label{fig:tree policies optimize then fit Sensitive p3}
  \end{subfigure}
\begin{subfigure}{0.3\textwidth}
\centering
         \includegraphics[width=1.0\linewidth]{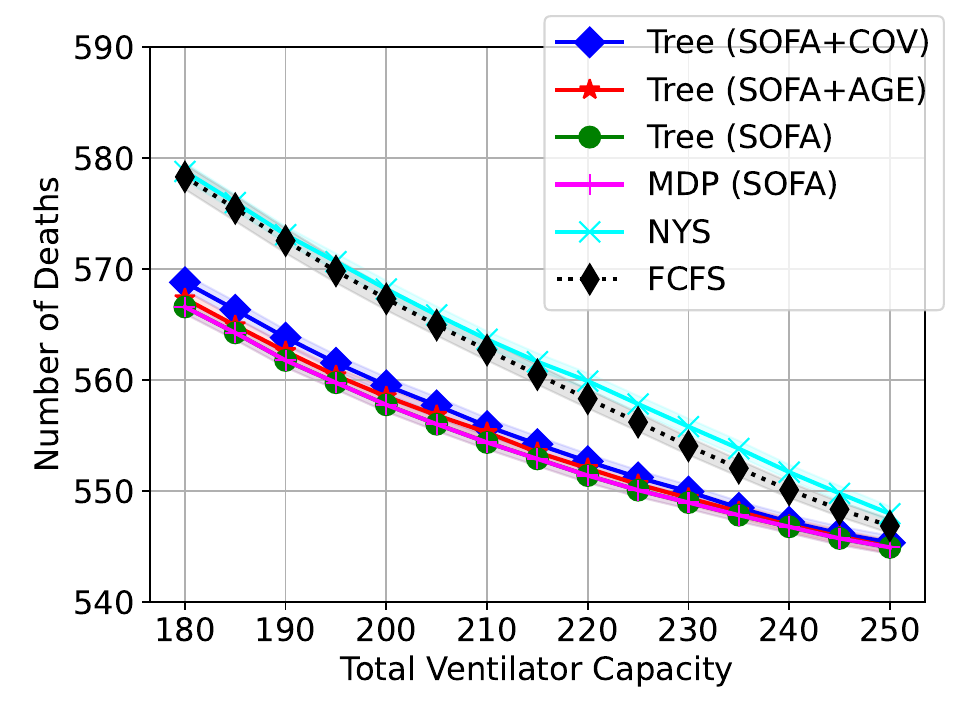}
               \caption{Tree policies from Alg. \ref{alg:optimize and fit}.}
         \label{fig:tree policies optimize and fit (Sensitive p3)}
  \end{subfigure}
         \caption{Number of deaths for various triage guidelines at hypothetical levels of ventilator capacities, for $\tau = 11$ and $(p_{{\sf up}},p_{{\sf down}}) = (0.95,0.90)$. }
         \label{fig:number-of-deaths (Sensitive p3)}
\end{figure}

\clearpage
\tb{
\section{Sensitivity analysis for the rewards parameters}\label{app:sensitivity-analysis}
In this appendix, we provide some more numerical experiments to investigate the impact of the choice of the rewards parameters $C,\rho$ and $\gamma$ outlined in Section \ref{subsec:MDP-triage-model}. The simulations presented in the main body use the parameters $C = 100, \rho = 0.9, \gamma = 0.5$, so that $C \rho^2 \gamma^2 \approx 20$, i.e., the ``worst" {\sf Alive} state ($A_{3}^{\sf ex}$) has a reward $20$ times larger than the ``best" {\sf Deceased} state ($D_{1}^{\sf ex}$). Here, we provide additional numerical experiments for the choices $(\rho, \gamma) \in \{(0.8,0.4), (0.9,0.5), (0.99,0.6)\}$ with varying $C \in \{10, 40, 70, 100, 130, 160\}$ in the corresponding region satisfying $C \rho^2 \gamma^2 \in [1,35]$. When $C \rho^2 \gamma^2 \approx 1$, the MDP model does not distinguish (in terms of instantaneous rewards) the state $D_{1}^{\sf ex}$ and $A_{3}^{\sf ex}$, which in turn means that the MDP formulation is pointless for our purpose of maximizing the number of patients that survive. On the contrary, when $C \rho^2 \gamma^2 \approx 35$, the reward in $A_{3}^{\sf ex}$ is 35 times larger than the reward in $D_{1}^{\sf ex}$. For every choice of the parameters $(C,\rho,\gamma)$, we compute an optimal unconstrained policy as well as the tree policies returned by Algorithm \ref{alg:optimize then fit} and Algorithm \ref{alg:optimize and fit}, and we estimate their performances using the simulation model from Section \ref{subsec:triage-simulation-model}. For the sake of conciseness, we only focus on the case where the ventilator capacity is $180$ ventilators. We observe in Figure \ref{fig:number-of-deaths (Sensitive Rewards)} that all the plausible values of $(C,\rho,\gamma)$ lead to comparable performances for the policies computed. The performances are particularly stable, so long as $C \rho^2 \gamma^2$ is much larger than $1$. This highlights the robustness of our conclusions from Section \ref{sec:simu}. When $C \rho^2 \gamma^2$ approaches $1$, the performance of some policies may degrade (e.g. for the tree (SOFA+COV) returned by Algorithm \ref{alg:optimize then fit} when $(\rho,\gamma) = (0.8,0.4)$, see Figure \ref{fig:tree policies optimize then fit Sensitive Rewards}).
}

\begin{figure}[htb]
\center
\begin{subfigure}{\textwidth}
\centering
\includegraphics[width=0.9\linewidth]{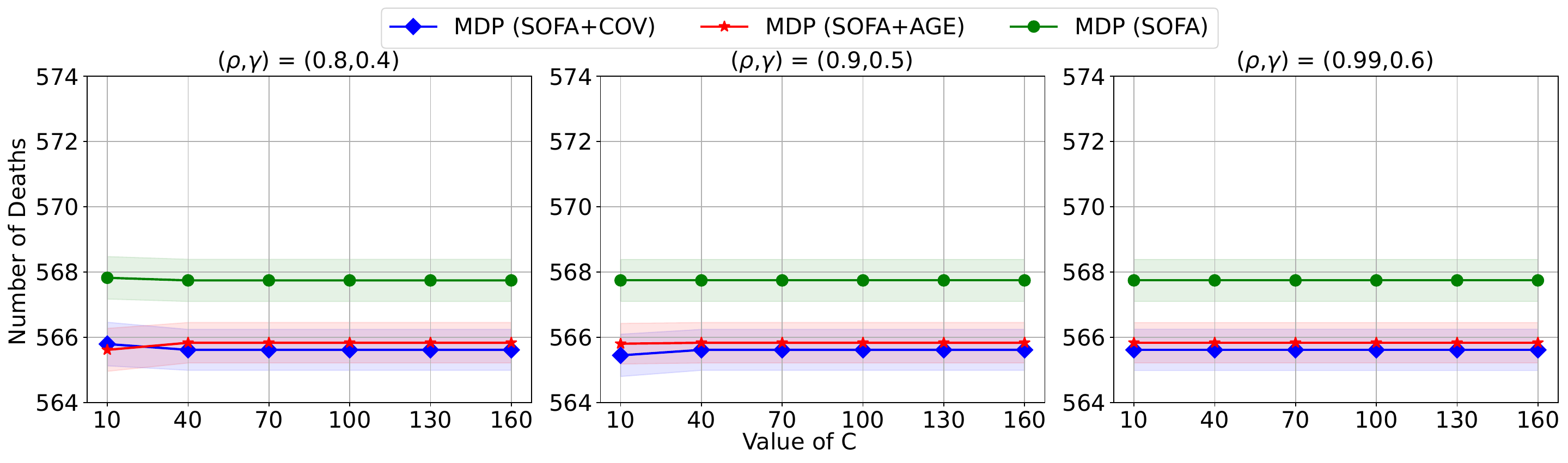}
\caption{MDP policies.}
         \label{fig:MDP policies Sensitive Rewards}
\end{subfigure}\\
\begin{subfigure}{\textwidth}
\centering
\includegraphics[width=0.9\linewidth]{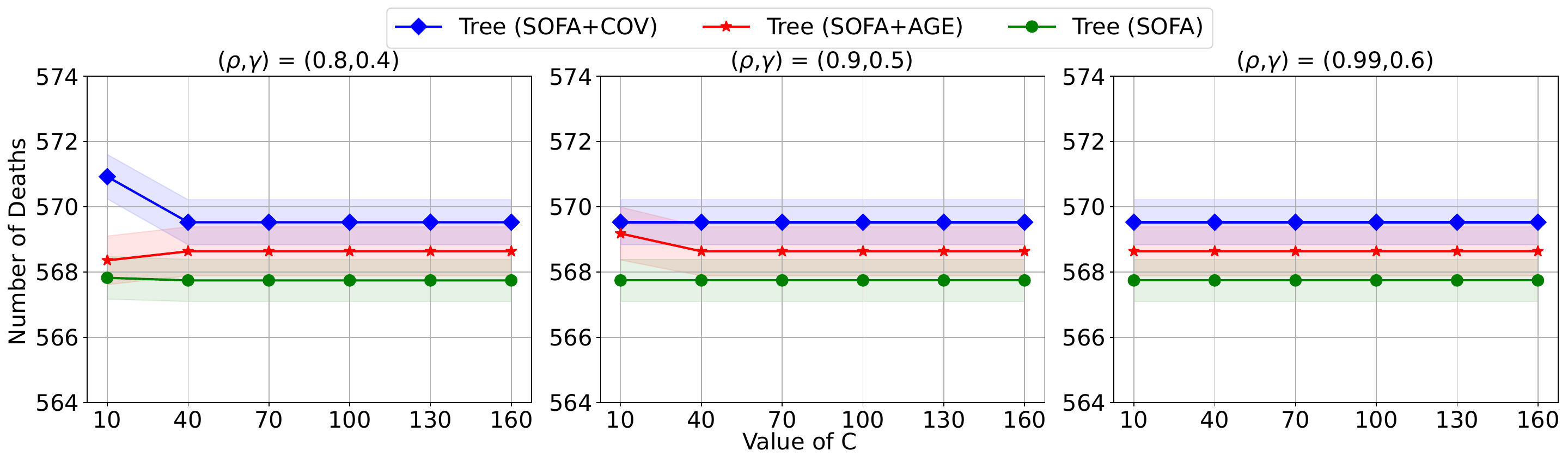}
\caption{Tree policies from Alg. \ref{alg:optimize then fit}.}
         \label{fig:tree policies optimize then fit Sensitive Rewards}
\end{subfigure}\\
\begin{subfigure}{\textwidth}
\centering
\includegraphics[width=0.9\linewidth]{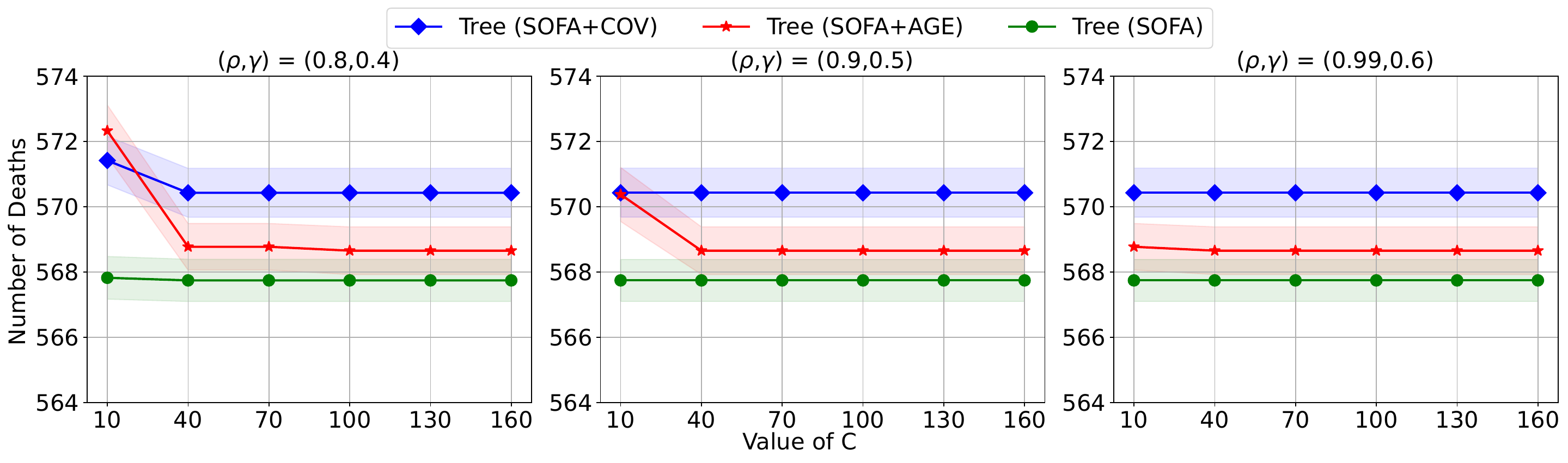}
\caption{Tree policies from Alg. \ref{alg:optimize and fit}.}
         \label{fig:tree policies optimize and fit Sensitive Rewards}
\end{subfigure}
  \caption{Number of deaths for various triage guidelines at hypothetical levels of ventilator capacities, for $C \in \{10, 40, 70, 100, 130, 160\}$ and $(\rho, \gamma) \in \{(0.8,0.4), (0.9,0.5), (0.99,0.6)\}$. }
         \label{fig:number-of-deaths (Sensitive Rewards)}
\end{figure}
\clearpage 

\tb{
\section{Comparisons with ISP-LU algorithms}\label{app:comparison ISP-LU}
In this appendix we first describe the ISP-LU algorithm for computing ventilator allocation guidelines, as introduced in \cite{anderson2023rationing}. We then investigate the best empirical setup for ISP-LU before comparing the performance of this method with our algorithms. Note that the algorithms that we call ISP-LU differ slightly from the ones described in \cite{anderson2023rationing} since our data does not include vital signs. Therefore, the ISP-LU guidelines presented here differ from the ones obtained in \cite{anderson2023rationing}, but we do our best to provide a fair comparison with our algorithms and to find the best setup for the ISP-LU guidelines within the scope of our decision problem.
\paragraph{The ISP-LU guidelines.} The ISP-LU guidelines are based on a threshold $\tau \geq 0$ and on patient-wise predictions. For each patient, the ISP-LU guidelines predict their probability of survival $\hat{P}$ given their covariate information (SOFA scores, demographics, and comorbidities), as well as their length-of-use $\hat{L}$ (duration of intubation). A ventilator is allocated to the patient if $\hat{P}/\hat{L} \geq \tau$. Following \cite{anderson2023rationing}, the prediction of $\hat{P}$ is performed with logistic regression. The prediction of $\hat{L}$ is performed by first clustering the patients as ``likely to die" ($\hat{P} \leq 0.3$), ``likely to survive" ($\hat{P} \geq 0.7$) and ``hard to predict" ($0.3 < \hat{P} < 0.7$). A LASSO regression is then used to predict the duration of intubation for each of these three classes. Similarly as all the other guidelines tested in this paper, we consider that the guidelines are only used once all ventilators are used; in this case, intubations are decided following the ISP-LU decision rule, and if necessary the patient with the largest SOFA is proactively extubated to intubate the new patient requiring a ventilator.
\paragraph{Practical implementation.} For the sake of comparing with the guidelines obtained in our paper, we consider ISP-LU guidelines based solely on SOFA, based on SOFA+AGE and based on SOFA+COVARIATES. We use {\sf scikitlearn} for the logistic regression and the LASSO regressions for predicting the probability and duration of intubation of each patient in our cohort. For choosing the threshold $\tau$, we explore values $\tau \in \{0.01,0.02,0.03,0.04\}$, chosen given the empirical distribution of the thresholds $\hat{P}/\hat{L}$ for the patients in our cohorts (see Table \ref{tab:ratio isp-lu}).
\begin{table}[htb]
\centering
\footnotesize
 \begin{tabular}{||l |  c | c | c ||}
 \hline
Predicted ratio $\hat{P}/\hat{L}$ & SOFA & SOFA+AGE & SOFA+COV  \\ [0.5ex]
 \hline\hline
Mean (IQR) & 0.031 (0.026-0.035) & 0.030 (0.024-0.036) & 0.031 (0.021 - 0.039) \\
Minimum & 0.017 & 0.011 & 0.002  \\
Maximum & 0.036 & 0.067 & 0.083 \\
 [1ex]
 \hline
 \end{tabular}
 \caption{Summary statistics the predicted ratios $\hat{P}/\hat{L}$ based on different state space (SOFA, SOFA+AGE and SOFA+COVARIATES).}
 \label{tab:ratio isp-lu}
\end{table}
\paragraph{Best setup for ISP-LU guidelines.}
We start by comparing the best choice of $\tau$ for the different ISP-LU guidelines. We present our results in Figure \ref{fig:isp-lu-finding-tau}. As a benchmark we also show the performance of the FCFS guidelines.
We find that the best setups for ISP-LU correspond to smaller values of $\tau$ ($\tau = 0.01$ for ISP-LU based on SOFA+AGE and SOFA, $\tau = 0.02$ for ISP-LU based on SOFA+AGE). We also note that when $\tau$ is too large, the performances of ISP-LU match the performance of FCFS, since ISP-LU never decides to prioritize patients: for instance, for ISP-LU based on SOFA, the maximum ratio $\hat{P}/\hat{L}$ is $0.036$, and therefore, choosing $\tau=0.04$ results in implementing FCFS (see Figure \ref{fig:isplu-sofa}). To better understand the intubation decision of the ISP-LU guidelines, we show their number of excluded patients at triage in Figure \ref{fig:isp-lu-nb-exclusion}, where the increase in the number of excluded patients at triage is clearly visible when $\tau$ augments. Note that when $\tau$ is smaller than the minimum observed ratio $\hat{P}/\hat{L}$ in the cohort (e.g. $\tau = 0.01$ for ISP-LU based on SOFA only with a minimum ratio of $0.017$), the ISP-LU guidelines always intubate an incoming patient. The fact that this simple policy performs well is coherent with the fact that our tree policies based on SOFA only also never exclude any patient from intubation at triage (although it may only excludes patients at reassessment, see Figure \ref{fig:tree-sofa}).
Overall, the ISP-LU guidelines with the best performances in our framework are the ISP-LU guidelines based on SOFA+AGE with a threshold $\tau = 0.02$. For completeness, we represent this guideline in Figure \ref{fig:ISP LU SOFA AGE heatmap}, using a heatmap and the same notations as for our tree policies: the letter `L' stands for {\em low} and indicates patients with a low priority score for intubation, i.e., patients that will be excluded from ventilator treatment at triage. The letter `H' stands for {\em high} and represents patients that will be intubated. We can see that the ISP-LU guidelines may not lend themselves to easy, clear interpretability in terms of SOFA and age.
\begin{figure}[htb]
\center
 \begin{subfigure}{0.3\textwidth}
\centering
         \includegraphics[width=1.0\linewidth]{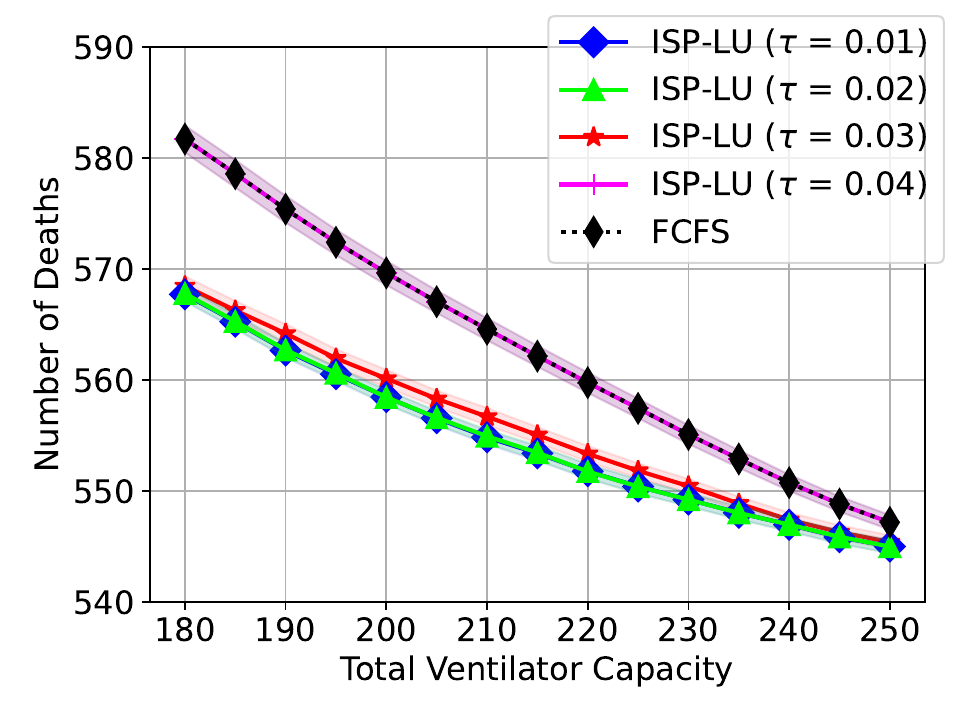}
                  \caption{SOFA.}
         \label{fig:isplu-sofa}
  \end{subfigure}
 \begin{subfigure}{0.3\textwidth}
\centering
         \includegraphics[width=1.0\linewidth]{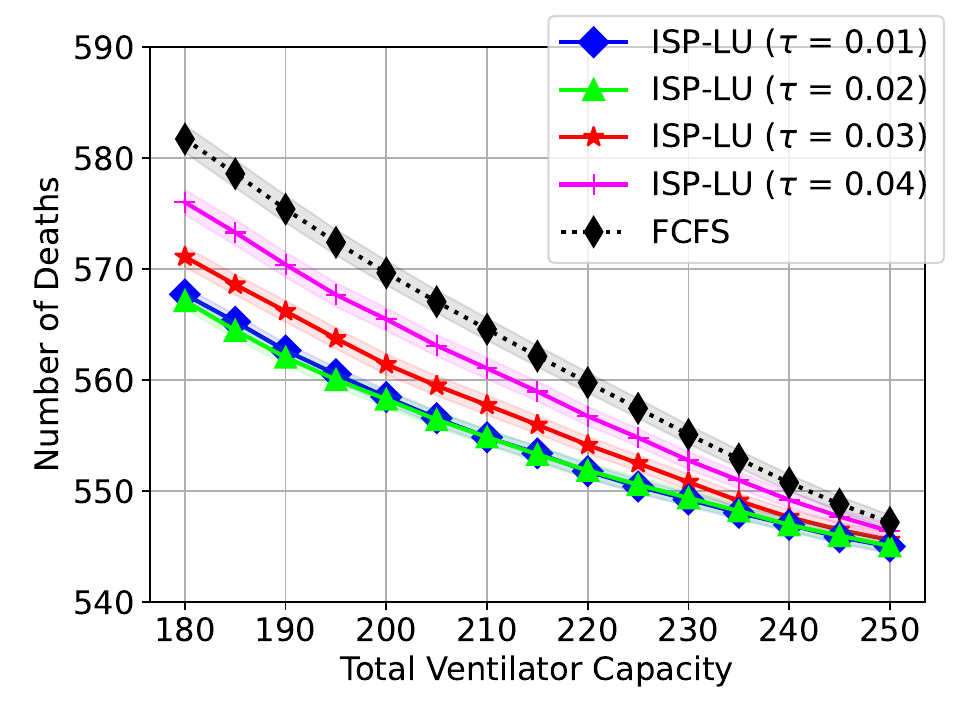}
                  \caption{SOFA+AGE.}
         \label{fig:isplu-sofa-age}
  \end{subfigure}
   \begin{subfigure}{0.3\textwidth}
\centering
         \includegraphics[width=1.0\linewidth]{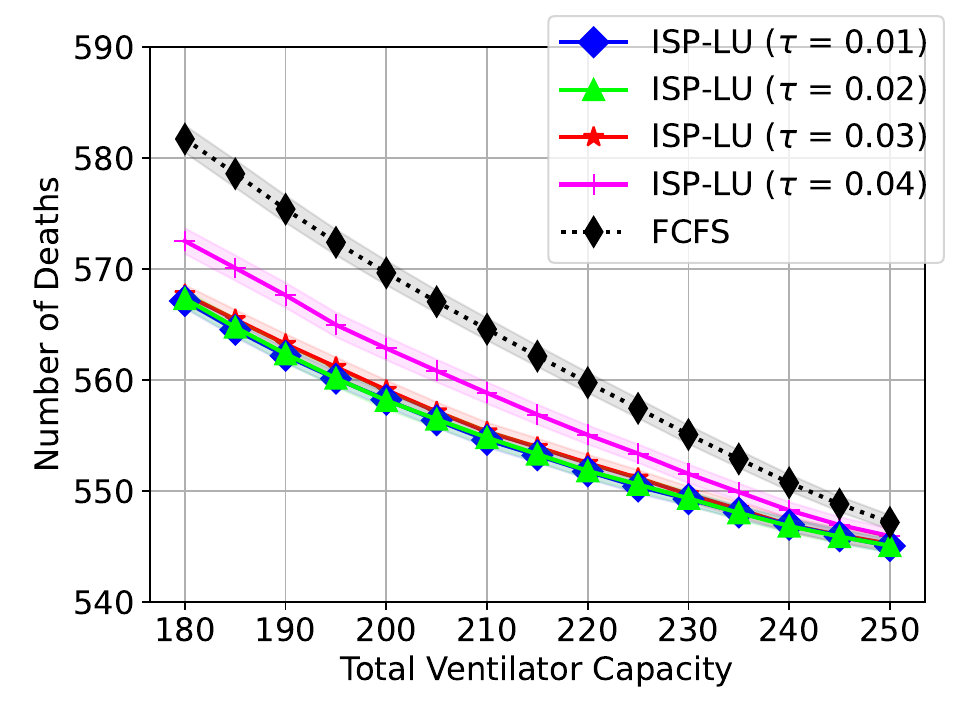}
                  \caption{SOFA+COV.}
         \label{fig:isplu-sofa-cov}
  \end{subfigure}
         \caption{Performance of ISP-LU guidelines based on different state space (SOFA, SOFA+AGE and SOFA+COV) for various values of the thresholds $\tau$.}
         \label{fig:isp-lu-finding-tau}
\end{figure}

\begin{figure}[htb]
\center
 \begin{subfigure}{0.3\textwidth}
\centering
         \includegraphics[width=1.0\linewidth]{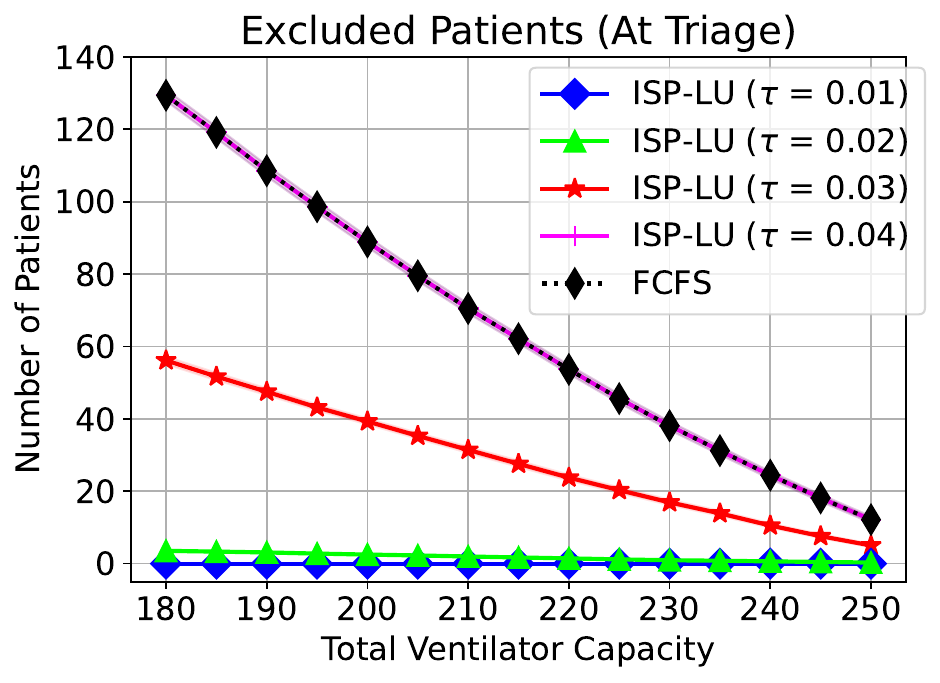}
                  \caption{SOFA.}
         \label{fig:isplu-sofa-nb-exclusion}
  \end{subfigure}
 \begin{subfigure}{0.3\textwidth}
\centering
         \includegraphics[width=1.0\linewidth]{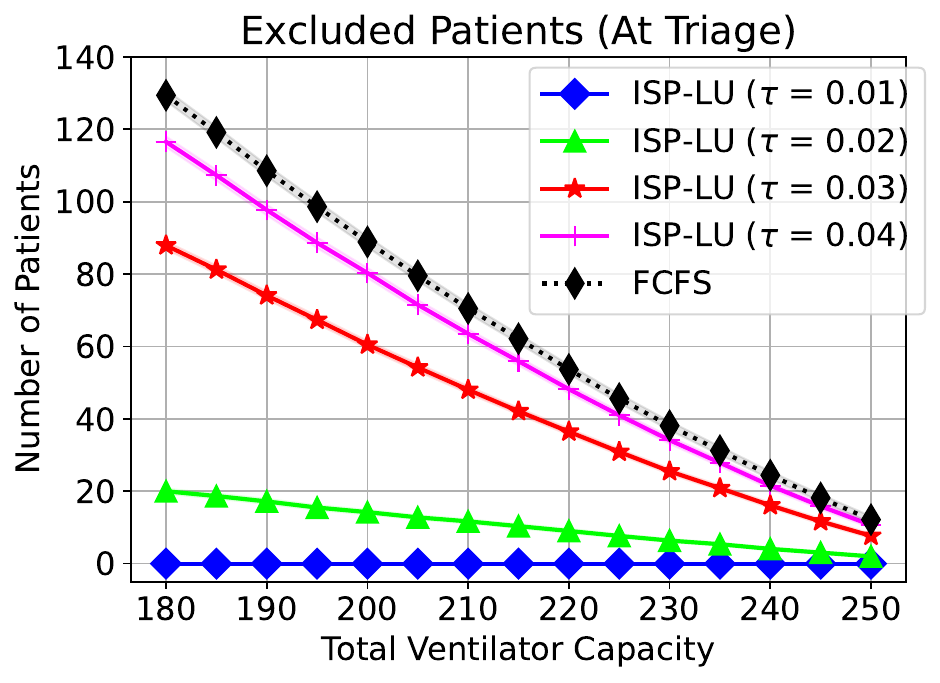}
                  \caption{SOFA+AGE.}
         \label{fig:isplu-sofa-age-nb-exclusion}
  \end{subfigure}
   \begin{subfigure}{0.3\textwidth}
\centering
         \includegraphics[width=1.0\linewidth]{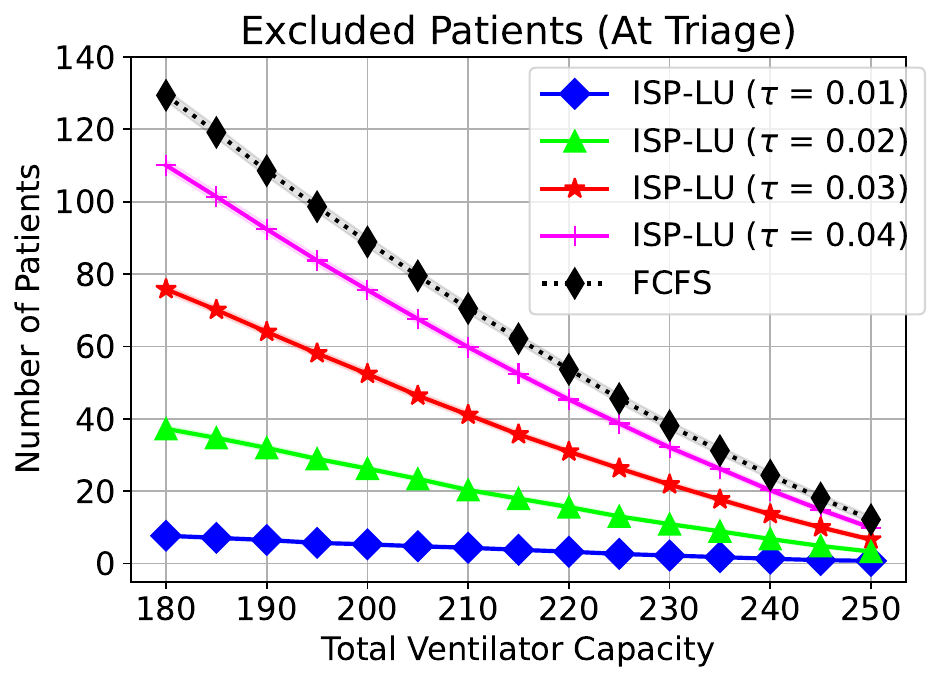}
                  \caption{SOFA+COV.}
         \label{fig:isplu-sofa-cov-nb-exclusion}
  \end{subfigure}
         \caption{Number of patients excluded at triage by the ISP-LU guidelines based on different state space (SOFA, SOFA+AGE and SOFA+COV) for various values of the thresholds $\tau$.}
         \label{fig:isp-lu-nb-exclusion}
\end{figure}

\begin{figure}[htb]
\center                     
\includegraphics[width=0.7\linewidth]{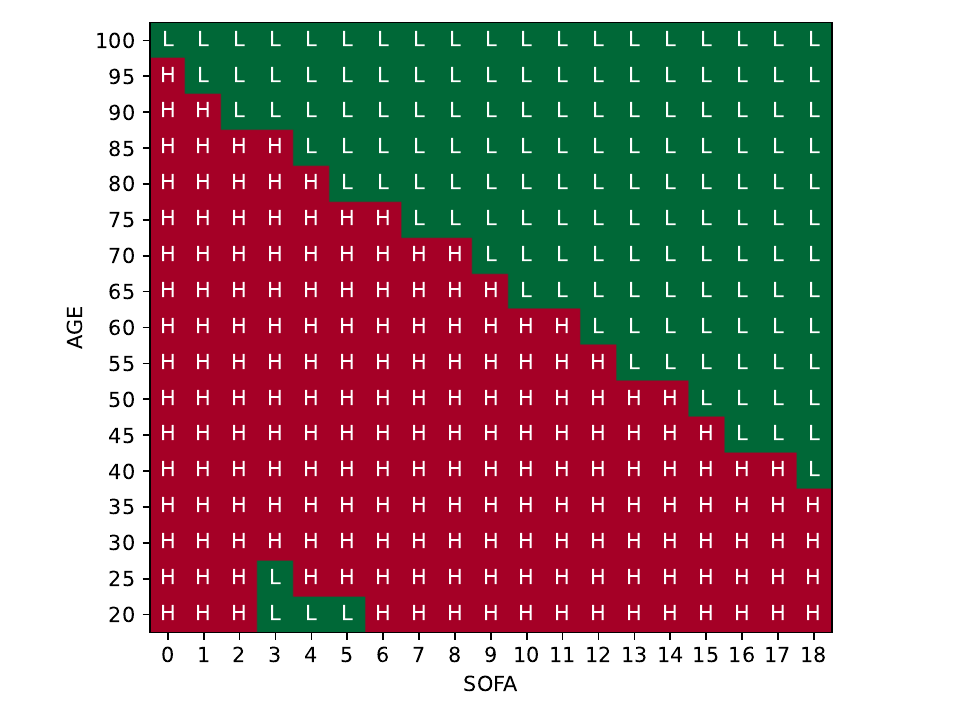}
  \caption{ISP-LU guidelines based on SOFA+AGE and $\tau = 0.02$. }
  \label{fig:ISP LU SOFA AGE heatmap}
\end{figure}

\paragraph{Comparisons with other guidelines from this paper.}
In Figure \ref{fig:comparisons best algorithms with isp-lu}, we compare the best setup for ISP-LU (based SOFA+AGE with $\tau=0.02$) with the best guidelines obtained in this paper, MDP SOFA+COV and the tree policies based on SOFA (Figure \ref{fig:tree-sofa}) returned by Algorithm 1 and Algorithm 2. We also show NYS guidelines and FCFS guidelines as benchmarks. We observe that ISP-LU guidelines perform on par with the best unconstrained policies (MDP SOFA+COV), and obtain comparable performance as the best interpretable policies returned by our algorithms. This emphasizes that tree-based policies are able to compete with other machine learning-based guidelines, the latter being more difficult to interpret (see for instance the representation in Figure \ref{fig:ISP LU SOFA AGE heatmap}).
\begin{figure}[htb]
\center
         \includegraphics[width=0.5\linewidth]{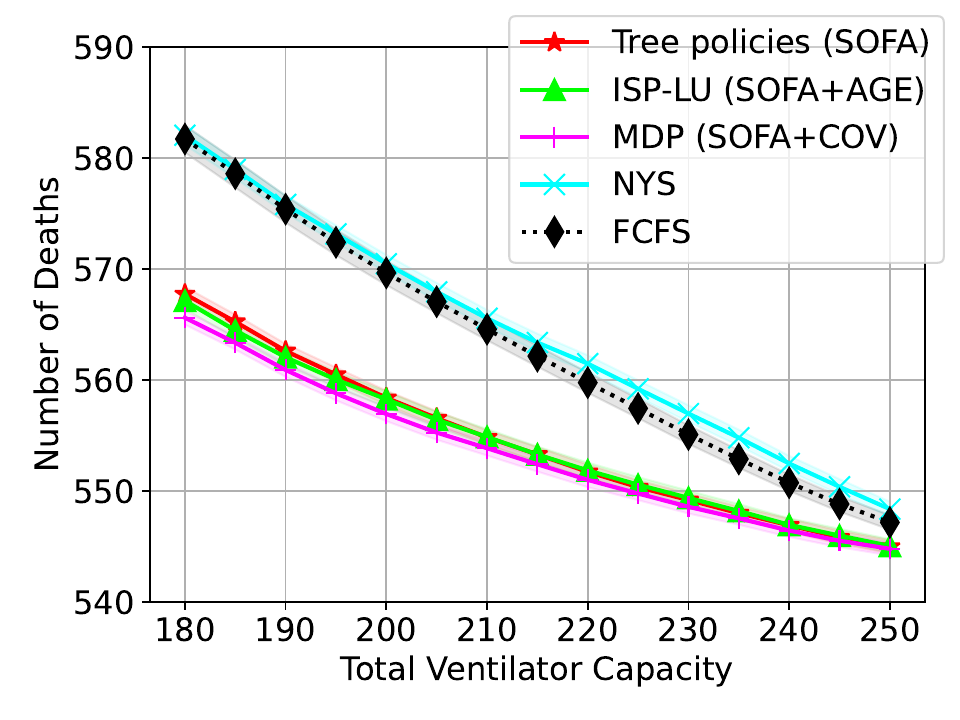}
         \caption{Comparison of the performance of ISP-LU with the other algorithms obtained in this paper.}
         \label{fig:comparisons best algorithms with isp-lu}
\end{figure}
}
\clearpage
\section{Details about the clinical data set}\label{app:details-data-set}
We present the summary statistics for our cohort of patients in Table \ref{tab:summary-statistics}, where we present the proportion of patients with some of the comorbidities present in the dataset, and compare them across the populations of patients that survived or not.
\begin{table}[htb]
\centering
 \begin{tabular}{||l |  c | c | c ||}
 \hline
Variables & All population & Survived & Deceased \\ [0.5ex]
 \hline\hline
 Number (n) & 807 & 264 & 543 \\
Age (year (std)) & 64.0 (13.5) & 60.0 (13.1) & 66.4 (12.9) \\
 Male gender (n (\%)) & 483 (59.9 \%) & 132 (50 \%) & 351 (64.6 \%) \\
  BMI (mean (std)) & 30.8 (7.5) & 30.7 (7.2) & 30.8 (7.7) \\
 \hline
  Diabetes (n (\%))& 319 (40.0 \%) & 78 (30.0 \%) & 241 (44.4 \%)  \\
 Charlson (mean (std)) & 2.9 (2.7) &  2.0 (2.6) & 3.3 (2.7) \\
 Malignancy (n (\%)) & 39 (4.5 \%) & 4 (1.5 \%) & 35 (6.5 \%) \\
  Renal disease (n (\%)) & 341 (42.2 \%) & 66 (25.0 \%) & 275 (50.7 \%) \\
  Dementia (n (\%)) & 92 (11.4 \%) & 23 (8.7 \%) & 69 (12.7 \%) \\
  Congestive Heart Failure (n )\%)) & 149 (18.5 \%) & 26 (9.9 \%) & 123 (22.7 \%) \\
  \hline
 Initial SOFA (mean (IQR)) & 2.0 (0.0-3.0) & 1.5 (0.0-2.0) & 2.2 (0.0-3.0)  \\
 Max SOFA  (mean (IQR)) & 9.7 (8.0-12.0) & 8.5 (7.0-10.0) & 10.2 (8.0-12.0) \\
 LOS (median (IQR)) & 16.8 (8.7-29.2) & 29.1 (18.9-46.5) & 12.5 (6.9-21.0) \\
Survival (n (\%)) & 264 (32.7 \%) & $\cdot $ & $\cdot$ \\
 [1ex]
 \hline
 \end{tabular}
 \caption{Summary statistics for the patients in our data set.}
 \label{tab:summary-statistics}
\end{table}

\clearpage



\end{APPENDICES}

\end{document}